\documentclass[11pt,onecolumn]{article}
\renewcommand{\baselinestretch}{1.0}

\usepackage[mathscr]{eucal}
\usepackage[cmex10]{amsmath}
\usepackage{epsfig,epsf,psfrag}
\usepackage{amssymb,amsmath,amsthm,amsfonts,latexsym}
\usepackage{amsmath,graphicx,bm,xcolor,url,overpic}
\usepackage{fixltx2e}
\usepackage{array}
\usepackage{verbatim}
\usepackage{bm}
\usepackage{algorithmic}
\usepackage{algorithm}
\usepackage{verbatim}
\usepackage{textcomp}
\usepackage{epstopdf}
\usepackage{subfigure}
\usepackage{acronym}
\usepackage{diagbox}
\usepackage{tikz}
\usepackage{setspace}
\usepackage{thmtools}
\usepackage{thm-restate}
\usepackage{hyperref}
\usepackage{cleveref}
\usetikzlibrary{graphs, positioning, quotes, shapes.geometric}

\newcommand{\openone}{\leavevmode\hbox{\small1\normalsize\kern-.33em1}} 

\catcode`~=11 \def\UrlSpecials{\do\~{\kern -.15em\lower .7ex\hbox{~}\kern .04em}} \catcode`~=13 

\allowdisplaybreaks[4]

\newcommand{\nn}{\nonumber}

\newcommand{\calA}{\mathcal{A}}
\newcommand{\calB}{\mathcal{B}}
\newcommand{\calC}{\mathcal{C}}
\newcommand{\calD}{\mathcal{D}}
\newcommand{\calE}{\mathcal{E}}
\newcommand{\calF}{\mathcal{F}}
\newcommand{\calG}{\mathcal{G}}
\newcommand{\calH}{\mathcal{H}}
\newcommand{\calI}{\mathcal{I}}
\newcommand{\calJ}{\mathcal{J}}
\newcommand{\calK}{\mathcal{K}}
\newcommand{\calL}{\mathcal{L}}
\newcommand{\calM}{\mathcal{M}}
\newcommand{\calN}{\mathcal{N}}

\newcommand{\calP}{\mathcal{P}}

\newcommand{\calS}{\mathcal{S}}
\newcommand{\calT}{\mathcal{T}}

\newcommand{\calV}{\mathcal{V}}
\newcommand{\calW}{\mathcal{W}}
\newcommand{\calX}{\mathcal{X}}

\newcommand{\hatcalT}{\hat{\calT}}

\newcommand{\bx}{\mathbf{x}}
\newcommand{\bX}{\mathbf{X}}

\newcommand{\bbE}{\mathbb{E}}

\newcommand{\bbN}{\mathbb{N}}

\newcommand{\bbR}{\mathbb{R}}

\newcommand{\frakT}{\mathfrak{T}}

\newcommand{\scC}{\mathscr{C}}

\DeclareMathAlphabet{\mathbsf}{OT1}{cmss}{bx}{n}
\DeclareMathAlphabet{\mathssf}{OT1}{cmss}{m}{sl}

\DeclareSymbolFont{bsfletters}{OT1}{cmss}{bx}{n}  
\DeclareSymbolFont{ssfletters}{OT1}{cmss}{m}{n}
\DeclareMathSymbol{\bsfGamma}{0}{bsfletters}{'000}
\DeclareMathSymbol{\ssfGamma}{0}{ssfletters}{'000}
\DeclareMathSymbol{\bsfDelta}{0}{bsfletters}{'001}
\DeclareMathSymbol{\ssfDelta}{0}{ssfletters}{'001}
\DeclareMathSymbol{\bsfTheta}{0}{bsfletters}{'002}
\DeclareMathSymbol{\ssfTheta}{0}{ssfletters}{'002}
\DeclareMathSymbol{\bsfLambda}{0}{bsfletters}{'003}
\DeclareMathSymbol{\ssfLambda}{0}{ssfletters}{'003}
\DeclareMathSymbol{\bsfXi}{0}{bsfletters}{'004}
\DeclareMathSymbol{\ssfXi}{0}{ssfletters}{'004}
\DeclareMathSymbol{\bsfPi}{0}{bsfletters}{'005}
\DeclareMathSymbol{\ssfPi}{0}{ssfletters}{'005}
\DeclareMathSymbol{\bsfSigma}{0}{bsfletters}{'006}
\DeclareMathSymbol{\ssfSigma}{0}{ssfletters}{'006}
\DeclareMathSymbol{\bsfUpsilon}{0}{bsfletters}{'007}
\DeclareMathSymbol{\ssfUpsilon}{0}{ssfletters}{'007}
\DeclareMathSymbol{\bsfPhi}{0}{bsfletters}{'010}
\DeclareMathSymbol{\ssfPhi}{0}{ssfletters}{'010}
\DeclareMathSymbol{\bsfPsi}{0}{bsfletters}{'011}
\DeclareMathSymbol{\ssfPsi}{0}{ssfletters}{'011}
\DeclareMathSymbol{\bsfOmega}{0}{bsfletters}{'012}
\DeclareMathSymbol{\ssfOmega}{0}{ssfletters}{'012}

\newcommand{\hath}{\hat{h}}

\newcommand{\hatI}{\hat{I}}

\newcommand{\tiln}{\tilde{n}}

\newcommand{\tilp}{\tilde{p}}

\newcommand{\ttheta}{\tilde{\theta}}

\newcommand{\halpha}{\hat{\alpha}}

\newcommand{\hrho}{\hat{\rho}}

\DeclareMathOperator*{\argmax}{arg\,max}

\newtheorem{theorem}{Theorem} 
\newtheorem{lemma}{Lemma}

\newtheorem{proposition}{Proposition}

\newtheorem{definition}{Definition}

\newtheorem{remark}{Remark}

\newtheorem{assumption}{Assumption}
\usepackage{cite}
\usepackage{amsmath,amssymb,amsfonts}
\usepackage{algorithmic}
\usepackage{graphicx}
\usepackage{textcomp}
\usepackage{xcolor,hyperref}
\usepackage{wrapfig}
\usetikzlibrary{chains,scopes,backgrounds,shapes,fit,shadows,calc,arrows.meta,decorations.pathreplacing,calligraphy}
\def\BibTeX{{\rm B\kern-.05em{\sc i\kern-.025em b}\kern-.08em
    T\kern-.1667em\lower.7ex\hbox{E}\kern-.125emX}}
\usepackage{fullpage}
\acrodef{cl}[CL]{\underline{C}how--\underline{L}iu}
\acrodef{scl}[SCL]{\underline{S}implified \underline{C}how--\underline{L}iu}
\acrodef{kl}[KL]{\underline{K}ullback–\underline{L}eibler}
\acrodef{mst}[MWST]{\underline{M}aximum \underline{W}eight \underline{S}panning \underline{T}ree}
\acrodef{hmm}[HMM]{\underline{H}idden \underline{M}arkov \underline{M}odel}
\acrodef{alathe}[{\sc Active-LATHE}]{\underline{Active} \underline{L}earning \underline{A}lgorithm for \underline{T}rees with \underline{H}omogeneous \underline{E}dges}
\acrodef{ampl}[{\sc AMPL}]{\underline{A}daptive \underline{M}arginalization based \underline{P}arallel \underline{L}asso}
\newcommand{\Passive}{{\rm passive}}
\newcommand{\Active}{{\rm active}}
\begin{document}

\title{Active-LATHE: An Active Learning Algorithm for Boosting the Error Exponent for Learning Homogeneous Ising Trees} 
\date{}



\author{Fengzhuo Zhang\footnotemark[1], Anshoo Tandon\footnotemark[1] \footnotemark[2] $\,$ and Vincent Y.~F.~Tan$^\ddagger$ \\ $^*$Department of Electrical and Computer Engineering, National University of Singapore \\
$^\ddagger$Department of Mathematics, National University of Singapore \\
Emails: \url{fzzhang@u.nus.edu}, \url{anshoo.tandon@gmail.com}, \url{vtan@nus.edu.sg}}


\maketitle

\renewcommand{\thefootnote}{\fnsymbol{footnote}}
\footnotetext[1]{These authors contributed equally to this work.}
\footnotetext[2]{Anshoo Tandon contributed to this paper when he was with the National University of Singapore. He is now with V-Labs, Bengaluru.}

\begin{abstract}
The Chow-Liu algorithm (IEEE Trans.~Inform.~Theory, 1968) has been a mainstay for the learning of tree-structured graphical models from i.i.d.\ sampled data vectors. Its theoretical properties have been well-studied and are well-understood. In this paper, we focus on the class of trees that are arguably even more fundamental, namely {\em homogeneous} trees in which each pair of nodes that forms an edge has the same correlation $\rho$. We ask whether we are able to further reduce the error probability of learning the structure of the homogeneous tree model when {\em active learning} or {\em active sampling of nodes or variables} is allowed. Our figure of merit is the {\em error exponent}, which quantifies the exponential rate of decay of the error probability with   an increasing number of data samples. At first sight, an improvement in the error exponent seems impossible, as all the edges are statistically identical. We design and analyze  an algorithm \ac{alathe}, which surprisingly boosts the error exponent by at least 40\% when $\rho$ is at least $0.8$. For all other values of $\rho$, we also observe commensurate, but more modest, improvements in the error exponent. Our analysis hinges on judiciously exploiting the minute but detectable statistical variation of the samples to allocate more data to parts of the graph in which we are less confident of being correct. 
\end{abstract}
\section{Introduction}
Graphical models provide a succinct representation of the conditional independence relationships among random variables. Each node in the graph represents a random variable, and the independence relationships are captured by the edges in the graph. Tree-structured graphical models are a special class of graphical models whose underlying graphs are trees. Due to their simplicity in inference and learning tasks, they are widely employed in domains such as computational biology \cite{saitou1987neighbor}, natural language processing \cite{parikh2014spectral}, and signal processing \cite{montanari2012graphical}.

The problem of learning the structure  of graphical models is an important and fundamental task in model selection. With independent samples of random variables in the graph, the goal is to estimate the edges in the graph. The \ac{cl} algorithm proposed in \cite{chow1968approximating} is a classical tree structure learning algorithm, which finds the \ac{mst} of the graph whose weights are empirical estimates of mutual information between random variables. The \ac{cl} algorithm essentially finds the closest tree-structured graphical model to the empirical distribution in the sense of \ac{kl} divergence, and so the \ac{cl} tree is the maximum likelihood estimate of the underlying tree structure. Using   large deviation techniques \cite{zeitouni1998large}, Tan, Anandkumar, Tong, and Willsky\cite{tan2011large} analyzed the exponential decay rate  of the error probability, which is known as the {\em error exponent}. For the homogeneous Ising tree models, Tandon, Tan, and Zhu\cite{tandon2020exact} showed that the exponent of the event that the \ac{cl} algorithm estimates the wrong structure tends to $0$ when the correlations between adjacent nodes approaches $1$, which implies that the error probability decreases slowly with an increase in the  number of independent samples. Fortunately, this performance degradation can be greatly mitigated by active learning.

{\em Active learning} is an umbrella term for   algorithms that are allowed to actively select  the data from which they learn the properties of the underlying distribution \cite{settles2009active}. The freedom to choose  data to learn from allows the learner to home in on parts of the model that are deemed to be ``most uncertain''. This    reduces the overall error probability or estimation error. Since all   conventional structure learning algorithms, including the \ac{cl} algorithm, rely on   samples that are collected {\em before} the  learning procedure commences, they are all \emph{passive} learning algorithms. Although Vats, Nowak, and Baraniuk~\cite{vats2014active} and Dasarathy, Singh, Balcan, and Park~\cite{dasarathy2016active} proposed active learning algorithms for general graphical models, their results cannot answer the question that whether the error exponent of the tree structure learning, and in particular {\em homogeneous} trees, can be improved compared to that of the \ac{cl} algorithm. First, they focused on the sample complexities of the algorithms, and the exponents of the probabilities of errors  cannot be directly derived from their results. Second, the effectiveness of their algorithms relies on the {\em precise knowledge} of bounds on the parameters of the graphical models to be estimated, but the \ac{cl} algorithm does not require any information about the parameters of the distribution except that the underlying graph is a tree, and, in our case, a {\em homogeneous} tree. Finally, their algorithms were designed  for the graphical models with general structures, and there is no active structure learning algorithm designed specifically for tree-structured graphical models. Thus, whether active learning techniques can improve the error exponent of arguably the simplest class of graphical models, namely homogeneous trees, constitutes a fascinating and meaningful question to be resolved, both from the statistical and machine learning viewpoints. We remark that {\em prima facie}, active learning cannot improve the error exponent of learning homogeneous trees because each part of the graph is equally easy or difficult to learn, since the correlations on each of the edges are identical. Nevertheless, we show, somewhat  surprisingly, that by judiciously exploiting the minute but detectable statistical variations of the samples at different nodes, we can improve the error exponent for learning homogeneous trees by at least 40\% when the correlations are at least $0.8$! For all other values of the correlation coefficient, we also observe commensurate, but more modest, improvements in the error exponent.  

\renewcommand{\thefootnote}{\arabic{footnote}}
\subsection{Main Contributions}
In this paper, we design an active structure learning algorithm \ac{alathe}\footnote{A \emph{lathe} (noun) is a machine used for shaping an object made of wood or metal by   rotating it about a horizontal axis against a tool. Here,  we are  {\em shaping} or {\em learning}  the structure of a graphical model, analogous to  a lathe, but doing this {\em actively}, hence the origin of the name of the algorithm \ac{alathe}.} for homogeneous Ising trees and prove that its error exponent   algorithm is no less than $c_{\rho}\geq 1$ times the error exponent of the \ac{cl} algorithm on the homogenous Ising trees, where $c_{\rho}$ is a non-decreasing function of the correlation parameter $\rho$ defined in Table~\ref{table:crho}. Informally, denoting $\calT$, $\hatcalT_{\Active}(n)$ and $\hatcalT_{\Passive}(n)$ as true tree, and  the trees learned by \ac{alathe} and the (minimum probability of error) \ac{cl} algorithms (using $n$ vector-valued samples) respectively,  we have
\begin{equation}
    \varliminf_{n\rightarrow \infty}-\frac{1}{n}\log \Pr \big(\hatcalT_{\Active}(n)\neq \calT\big)\ge c_\rho\,\bigg( \varliminf_{n\rightarrow \infty}-\frac{1}{n}\log \Pr \big(\hatcalT_{\Passive}(n)\neq \calT\big)\bigg).
\end{equation}
\ac{alathe} consists of two main phases---the global learning phase and the local refinement phase. The global learning phase learns a coarse structure and estimates the parameter $\rho$ of the underlying model, which serves as a basis for the 
subsequent local phase. The local refinement phase detects the unconfident parts in the coarsely learned structure and uses all the remaining samples to provide more accurate and reliable refinements to the structure of the unconfident parts of the graph.  Empirical simulations for a range of tree structures and correlation parameters corroborate our theoretical findings. 

\def \sca{1}
\def \vsp{12*\sca}
\def \hei{7*\sca}
\def \wid{12*\sca}
\definecolor{lightyel}{RGB}{250,231,143}
\definecolor{stdorg}{RGB}{234,155,70}
\definecolor{deppurp}{RGB}{128,0,128}
\begin{figure}[t]
	\centering
	\begin{tikzpicture}
		\tikzstyle{every node}=[font=\normalsize]
		\node [draw,fill=lightyel,minimum width=\wid mm,minimum height=\hei mm](a1)at(0,0){};
		\node [draw,fill=lightyel,minimum width=\wid mm,minimum height=2*\hei mm](a2)at(0,1.5*\hei mm){};
		\node [draw,fill=lightyel,minimum width=\wid mm,minimum height=1.5*\hei mm](a3)at(0,3.25*\hei mm){};
		\node [draw,fill=lightyel,minimum width=\wid mm,minimum height=1*\hei mm](a4)at(0,4.5*\hei mm){};
		\node [draw,fill=stdorg,minimum width=\wid mm,minimum height=0.8*\hei mm](a5)at(0,5.3*\hei mm){};
		
		\node [draw,fill=lightyel,minimum width=\wid mm,minimum height=2*\hei mm](b1)at(\wid mm,0.5*\hei mm){};
		\node [draw,fill=lightyel,minimum width=\wid mm,minimum height=1.5*\hei mm](b2)at(\wid mm,2.25*\hei mm){};
		\node [draw,fill=lightyel,minimum width=\wid mm,minimum height=1.2*\hei mm](b3)at(\wid mm,3.6*\hei mm){};
		\node [draw,fill=lightyel,minimum width=\wid mm,minimum height=1.3*\hei mm](b4)at(\wid mm,4.85*\hei mm){};
		\node [draw,fill=stdorg,minimum width=\wid mm,minimum height=0.9*\hei mm](b5)at(\wid mm,5.95*\hei mm){};
		
		\node [draw,fill=lightyel,minimum width=\wid mm,minimum height=0.8*\hei mm](c1)at(2*\wid mm,-0.1*\hei mm){};
		\node [draw,fill=lightyel,minimum width=\wid mm,minimum height=0.5*\hei mm](c2)at(2*\wid mm,0.55*\hei mm){};
		\node [draw,fill=lightyel,minimum width=\wid mm,minimum height=1.1*\hei mm](c3)at(2*\wid mm,1.35*\hei mm){};
		\node [draw,fill=lightyel,minimum width=\wid mm,minimum height=0.6*\hei mm](c4)at(2*\wid mm,2.2*\hei mm){};
		\node [draw,fill=stdorg,minimum width=\wid mm,minimum height=1.1*\hei mm](c5)at(2*\wid mm,3.05*\hei mm){};
		
		\node [draw,fill=lightyel,minimum width=\wid mm,minimum height=0.5*\hei mm](d1)at(3*\wid mm,-0.25*\hei mm){};
		\node [draw,fill=lightyel,minimum width=\wid mm,minimum height=1*\hei mm](d2)at(3*\wid mm,0.5*\hei mm){};
		\node [draw,fill=lightyel,minimum width=\wid mm,minimum height=0.7*\hei mm](d3)at(3*\wid mm,1.35*\hei mm){};
		\node [draw,fill=lightyel,minimum width=\wid mm,minimum height=0.3*\hei mm](d4)at(3*\wid mm,1.85*\hei mm){};
		\node [draw,fill=stdorg,minimum width=\wid mm,minimum height=1.2*\hei mm](d5)at(3*\wid mm,2.6*\hei mm){};
		
		\node [draw,fill=lightyel,minimum width=\wid mm,minimum height=\hei mm](e1)at( 4*\wid mm+\vsp mm,0){};
		\node [draw,fill=lightyel,minimum width=\wid mm,minimum height=2*\hei mm](e2)at(4*\wid mm+\vsp mm,1.5*\hei mm){};
		\node [draw,fill=lightyel,minimum width=\wid mm,minimum height=1.5*\hei mm](e3)at(4*\wid mm+\vsp mm,3.25*\hei mm){};
		\node [draw,fill=lightyel,minimum width=\wid mm,minimum height=1*\hei mm](e4)at(4*\wid mm+\vsp mm,4.5*\hei mm){};
		\node [draw,fill=stdorg,minimum width=\wid mm,minimum height=0.8*\hei mm](e5)at(6*\wid mm+\vsp mm,4*\hei mm){};
		
		\node [draw,fill=lightyel,minimum width=\wid mm,minimum height=2*\hei mm](f1)at(5*\wid mm+\vsp mm,0.5*\hei mm){};
		\node [draw,fill=lightyel,minimum width=\wid mm,minimum height=1.5*\hei mm](f2)at(5*\wid mm+\vsp mm,2.25*\hei mm){};
		\node [draw,fill=lightyel,minimum width=\wid mm,minimum height=1.2*\hei mm](f3)at(5*\wid mm+\vsp mm,3.6*\hei mm){};
		\node [draw,fill=lightyel,minimum width=\wid mm,minimum height=1.3*\hei mm](f4)at(5*\wid mm+\vsp mm,4.85*\hei mm){};
		\node [draw,fill=stdorg,minimum width=\wid mm,minimum height=0.9*\hei mm](f5)at(7*\wid mm+\vsp mm,3.65*\hei mm){};
		
		\node [draw,fill=lightyel,minimum width=\wid mm,minimum height=0.8*\hei mm](g1)at(6*\wid mm+\vsp mm,-0.1*\hei mm){};
		\node [draw,fill=lightyel,minimum width=\wid mm,minimum height=0.5*\hei mm](g2)at(6*\wid mm+\vsp mm,0.55*\hei mm){};
		\node [draw,fill=lightyel,minimum width=\wid mm,minimum height=1.1*\hei mm](g3)at(6*\wid mm+\vsp mm,1.35*\hei mm){};
		\node [draw,fill=lightyel,minimum width=\wid mm,minimum height=0.6*\hei mm](g4)at(6*\wid mm+\vsp mm,2.2*\hei mm){};
		\node [draw,fill=stdorg,minimum width=\wid mm,minimum height=1.1*\hei mm](g5)at(6*\wid mm+\vsp mm,3.05*\hei mm){};
		
		\node [draw,fill=lightyel,minimum width=\wid mm,minimum height=0.5*\hei mm](h1)at(7*\wid mm+\vsp mm,-0.25*\hei mm){};
		\node [draw,fill=lightyel,minimum width=\wid mm,minimum height=1*\hei mm](h2)at(7*\wid mm+\vsp mm,0.5*\hei mm){};
		\node [draw,fill=lightyel,minimum width=\wid mm,minimum height=0.7*\hei mm](h3)at(7*\wid mm+\vsp mm,1.35*\hei mm){};
		\node [draw,fill=lightyel,minimum width=\wid mm,minimum height=0.3*\hei mm](h4)at(7*\wid mm+\vsp mm,1.85*\hei mm){};
		\node [draw,fill=stdorg,minimum width=\wid mm,minimum height=1.2*\hei mm](h5)at(7*\wid mm+\vsp mm,2.6*\hei mm){};
		
		\node [draw=none,fill=none,font=\fontsize{10}{6}\selectfont,minimum width=\wid mm,minimum height=1.2*\hei mm](i1)at(1.5*\wid mm,7*\hei mm){\bf{\emph{Passive strategy}}};
		\node [draw=none,fill=none,font=\fontsize{10}{6}\selectfont,minimum width=\wid mm,minimum height=1.2*\hei mm](i2)at(5.5*\wid mm+\vsp mm,7*\hei mm){\bf{\emph{Active strategy}}};
		\node [draw=none,fill=none,font=\fontsize{8}{6}\selectfont,minimum width=\wid mm,minimum height=1.2*\hei mm](i3)at(7*\wid mm+\vsp mm,5.2*\hei mm){\color{red}\bf{expected height}};
		\node [draw=none,fill=none,font=\fontsize{8}{6}\selectfont,minimum width=\wid mm,minimum height=1.2*\hei mm](i4)at(3.2*\wid mm,3.6*\hei mm){\color{deppurp}\bf{active gain}};
		
		\node [draw=none,fill=none,font=\fontsize{8}{6}\selectfont,minimum width=\wid mm,minimum height=1.2*\hei mm](j1)at(0,-0.75*\hei mm){\bf{$1$}};
		\node [draw=none,fill=none,font=\fontsize{8}{6}\selectfont,minimum width=\wid mm,minimum height=1.2*\hei mm](j2)at(1*\wid mm,-0.75*\hei mm){\bf{$2$}};
		\node [draw=none,fill=none,font=\fontsize{8}{6}\selectfont,minimum width=\wid mm,minimum height=1.2*\hei mm](j3)at(2*\wid mm,-0.75*\hei mm){\bf{$3$}};
		\node [draw=none,fill=none,font=\fontsize{8}{6}\selectfont,minimum width=\wid mm,minimum height=1.2*\hei mm](j4)at(3*\wid mm,-0.75*\hei mm){\bf{$4$}};
		\node [draw=none,fill=none,font=\fontsize{8}{6}\selectfont,minimum width=\wid mm,minimum height=1.2*\hei mm](j5)at(4*\wid mm+\vsp mm,-0.75*\hei mm){\bf{$1$}};
		\node [draw=none,fill=none,font=\fontsize{8}{6}\selectfont,minimum width=\wid mm,minimum height=1.2*\hei mm](j6)at(5*\wid mm+\vsp mm,-0.75*\hei mm){\bf{$2$}};
		\node [draw=none,fill=none,font=\fontsize{8}{6}\selectfont,minimum width=\wid mm,minimum height=1.2*\hei mm](j7)at(6*\wid mm+\vsp mm,-0.75*\hei mm){\bf{$3$}};
		\node [draw=none,fill=none,font=\fontsize{8}{6}\selectfont,minimum width=\wid mm,minimum height=1.2*\hei mm](j8)at(7*\wid mm+\vsp mm,-0.75*\hei mm){\bf{$4$}};
		
		\node [draw=none,fill=none,font=\fontsize{10}{6}\selectfont,minimum width=\wid mm,minimum height=1.2*\hei mm](k1)at(-1*\wid mm,-0.5*\hei mm){$h_{0}$};
		\node [draw=none,fill=none,font=\fontsize{10}{6}\selectfont,minimum width=\wid mm,minimum height=1.2*\hei mm](k2)at(-1*\wid mm,3*\hei mm){$b_{1}$};
		\node [draw=none,fill=none,font=\fontsize{10}{6}\selectfont,minimum width=\wid mm,minimum height=1.2*\hei mm](k2)at(0*\wid mm,6.1*\hei mm){$w$};
		\node [draw=none,fill=none,font=\fontsize{10}{6}\selectfont,minimum width=\wid mm,minimum height=1.2*\hei mm](k2)at(0*\wid mm,1.5*\hei mm){$H_{i}$};
		\node [draw=none,fill=none,font=\fontsize{10}{6}\selectfont,minimum width=\wid mm,minimum height=1.2*\hei mm](k2)at(8.0*\wid mm+\vsp mm,2.4*\hei mm){$n\hath$};
		\node [draw=none,fill=none,font=\fontsize{10}{6}\selectfont,minimum width=\wid mm,minimum height=1.2*\hei mm](k2)at(4.5*\wid mm+\vsp mm,6.2*\hei mm){$\calS$};
		\node [draw=none,fill=none,font=\fontsize{10}{6}\selectfont,minimum width=\wid mm,minimum height=1.2*\hei mm](k2)at(6.5*\wid mm+\vsp mm,6.2*\hei mm){$\calW$};
		
		\draw[red,very thick,dashed](3.3*\wid mm+\vsp mm,4.8125*\hei mm)--(7.9*\wid mm+\vsp mm,4.8125*\hei mm);
		\draw[deppurp,very thick,dashed](3.5*\wid mm,4.1*\hei mm)--(7*\wid mm+\vsp mm,4.1*\hei mm);
		\draw[deppurp,very thick,dashed](2.3*\wid mm,3.2*\hei mm)--(3.5*\wid mm+0.5*\vsp mm,3.2*\hei mm);
		\draw[deppurp,very thick,<->](3.5*\wid mm+0.4*\vsp mm,3.2*\hei mm)--(3.5*\wid mm+0.4*\vsp mm,4.1*\hei mm);
		\draw[thick,->](-0.7*\wid mm,-0.5*\hei mm)--(-0.7*\wid mm,5.7*\hei mm);
		\draw[thick,<->](-0.5*\wid mm,5.9*\hei mm)--(0.5*\wid mm,5.9*\hei mm);
		\draw[thick,->](0*\wid mm,1*\hei mm)--(0*\wid mm,0.5*\hei mm);
		\draw[thick,->](0*\wid mm,2.0*\hei mm)--(0*\wid mm,2.5*\hei mm);
		\draw[thick,->](7.7*\wid mm+\vsp mm,-0.5*\hei mm)--(7.7*\wid mm+\vsp mm,4.8125*\hei mm);
		
        \draw[decorate,decoration={brace,raise=8pt,amplitude=0.2cm}] (3.5*\wid mm+\vsp mm,5.2*\hei mm)--(5.5*\wid mm+\vsp mm,5.2*\hei mm);
        \draw[decorate,decoration={brace,raise=8pt,amplitude=0.2cm}] (5.5*\wid mm+\vsp mm,5.2*\hei mm)--(7.5*\wid mm+\vsp mm,5.2*\hei mm);
        
	\end{tikzpicture}
	\caption{An example comparing the passive strategy and the active strategy  based on the analogy in Section~\ref{sec:work}.   The passive strategy allocates all the bricks uniformly among all the $p=4$ breakages. In contrast, the active strategy first allocates $\alpha=80\%$ of all the bricks---namely, the yellow bricks---uniformly over all the breakages. Since the heights of $1^{\mathrm{st}}$ and $2^{\mathrm{nd}}$ breakages, the strong ones, are larger than the expected height, all the remaining bricks---namely, the orange ones---are then allocated to the $3^{\mathrm{rd}}$ and $4^{\mathrm{th}}$ breakages.} 
	\label{fig:demo_fig}
\end{figure}
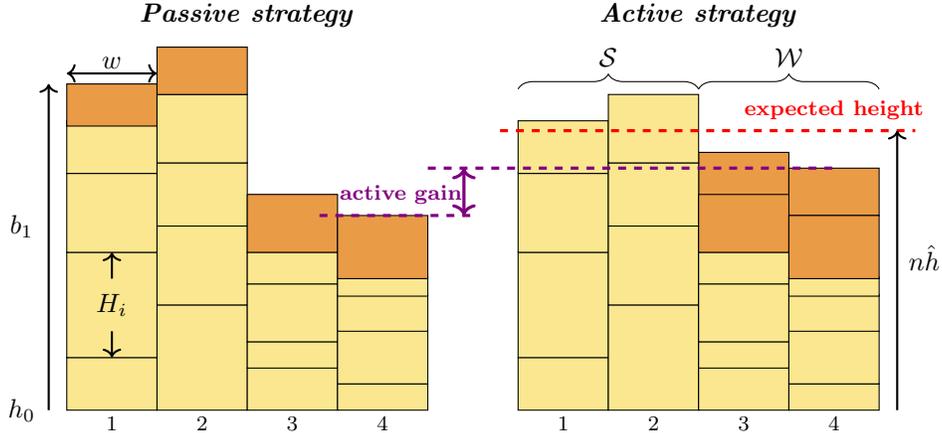
\subsection{Why does \ac{alathe} work?} \label{sec:work}
We now use a toy example to explain why the proposed \ac{alathe} algorithm can reduce the error probability compared to  passive learning in the homogeneous case (in which every edge has the same correlation). See Fig.~\ref{fig:demo_fig}. Suppose that we are tasked   to repair a dam with $p$ breakages using $np$ bricks to defend a city from an impending flood. Each breakage has the same width $w$ and the same initial height $h_{0}$. The width of each 
brick is exactly $w$, and the heights of all the bricks correspond to random variables $H_{i}$ for $i\in[np]=\{1,2,\ldots, np\}$, which are i.i.d.\ generated from some unknown distribution $Q$. We denote the height of the $i^{\rm th}$ breakage as $b_{i}$, then we have 
$b_{i}=b_{i}(I_{i})=h_{0}+\sum_{i\in I_{i}} H_{i}$ for $i\in[p]$, where $I_{i}\subseteq [np]$ is the set of indices of the bricks that are piled up at $i^{\rm th}$ breakage. Our objective here is to maximize the minimum height of the  breakages,  i.e., 
maximize $\min_{i\in[p]} b_{i}$, since the flood will rush in through the shortest, and hence weakest, breakage. The passive repair strategy involves acquiring all these $np$ bricks at once, and sets $I_{i} \triangleq [ip]\backslash[(i-1)p]=\{(i-1)p+1,(i-1)p+2,\ldots,ip\}$ for $i\in[p]$. 
In contrast, our active repair strategy first sets $\alpha \in (0,1)$, then acquires $\alpha np$ bricks, allocating them uniformly to all the breakages, i.e., $I_{i}\triangleq  [\alpha ip]\backslash[\alpha(i-1)p]$ for all $i\in [p]$. Since the heights of the bricks $H_i$ are i.i.d., the expectations of the heights $\bbE[H_i]$ 
can be estimated as $\hath=\sum_{i\in[\alpha np]} H_{i}/(\alpha np)$. We then find the set of ``strong'' breakages $\calS \triangleq \{i\in[p] \mid  b_{i}([\alpha ip]\backslash[\alpha(i-1)p])\geq n\hath\}$. We are then relatively \emph{confident} that 
the heights of the breakages indexed by $\calS$ exceed the expected height $n\hath$. We then assign all the remaining bricks to the ``weak'' breakages $\calW=[p]\backslash\calS$ uniformly at random. This active repair strategy leverages the information that the heights of  bricks have the same distribution, which is known as the \emph{homogeneous} setting, and focuses on the breakages that are below the expected height. We apply this idea in an iterative fashion to the graphical model selection problem at hand where the bricks correspond to samples and the breakages correspond to edges. However, in our problem, there is an additional consideration that the structure is unknown and has to be estimated. In this flooding problem, the structure is ``linear'' and known entirely.

\subsection{Related Works}
The seminal work of Chow and Liu\cite{chow1968approximating} initiated the study of   structure learning of the graphical models. The authors showed that the \ac{mst} of the weighted  graph with weights being the  mutual information quantities is the maximum likelihood estimate of the underlying tree structure. Chow and Wagner  \cite{chow1973consistency} then proved the consistency of the \ac{cl} algorithm. Using large deviation techniques \cite{zeitouni1998large}, Tan, Anandkumar, Tong, and Willsky\cite{tan2011large} studied the exponential rate of the decay of the error probability of the \ac{cl} algorithm. Building on this result, Tandon, Tan, and Zhu\cite{tandon2020exact} derived a closed-form expression of the exponent of the error probability of the \ac{cl} algorithm on homogeneous Ising trees and provided (exact asymptotic/strong large deviations-based) refinements to the prefactors in the exponent. Bhattacharyya {\em et al.}~\cite{bha21} considered near-optimal distribution learning using the \ac{cl} algorithm. Variants of the \ac{cl} algorithm have also attracted extensive attention. The Chow--Liu initialization designed by Choi, Tan, Anandkumar, and Willsky \cite{choi2011learning} was adopted to initialize the latent tree structure learning algorithm and this was recently exploited by Zhang and Tan~\cite{zhang21} to handle arbitrary corruptions in learning Gaussian latent trees.  Bresler and Karzand~\cite{bresler2020learning} learned tree models using \ac{cl} to make predictions. Finally,  Boix-Adsera, Bresler, and Koehler\cite{boix2021chow} considered the tree structure learning problem for performing downstream inference  and prediction tasks optimally.

In the past decade, active learning has been extensively studied\cite{hanneke2009theoretical,naghshvar2013active}. For Bayesian networks,  the structure learning problem has been well studied in \cite{tong2001active, he2008active, li2009active, squires2020active, ben2021active} with different assumptions and performance metrics. In contrast, the active structure learning problem of undirected graphical models has not been thoroughly studied. Vats, Nowak, and Baraniuk\cite{vats2014active} first designed an active structure learning algorithm based on junction tree representations for discrete alphabets. The \ac{ampl} algorithm was then designed to learn the structure of Gaussian graphical models in an active fashion~\cite{dasarathy2016active}. Impossibility results of active structure learning of Ising models and Gaussian models were derived by Scarlett and Cevher~\cite{scarlett2017lower}. These results  show  that the lower bounds of the sample complexity of the active structure learning have the same dependence as that of   passive structure learning in the number of nodes $p$ and the maximal degree $d$ when $p,d\rightarrow\infty$. However, the constants in \cite{scarlett2017lower}, corresponding to our error exponent here,  were not provided explicitly. 

\subsection{Paper Outline}
The rest of this paper is structured as follows. In Section~\ref{sec:prelim}, we state the notation used, mathematical preliminaries, and the problem statement. In Section~\ref{sec:structure}, we introduce the passive structure learning algorithm  and outline the main ideas in the active structure learning algorithm. The main result, Theorem~\ref{thm:mainthem},  is also stated formally in Section~\ref{sec:structure}.  In Section~\ref{sec:sim}, we present simulation results to corroborate Theorem~\ref{thm:mainthem} and show that active learning significantly outperforms the the passive Chow--Liu algorithm even for {\em homogeneous} Ising trees. The proof of Theorem~\ref{thm:mainthem} is outlined in Section~\ref{sec:pfview}, and the details of certain lemmas are relegated to the appendices in Section~\ref{sec:app}. The paper is concluded in Section~\ref{sec:conclusion} where we also suggest directions for future work. 

\section{Preliminaries and problem statement}\label{sec:prelim}
We denote the set $\{1,\ldots,n\}$ as $[n]$. We use san-serif letters $x$, boldface letters $\bx$, uppercase letters~$X$, and bold uppercase letters $\bX$ 
to denote scalars, vectors, random 
variables,  and random vectors respectively. 
The notation $X^{(m)}$ represents the $m^{\mathrm{th}}$ sample of $X$. The unique path connecting nodes $i$ and $j$ in the tree $\calT$ is denoted as ${\rm Path}_{\calT}(i,j)$. The {\em KL-divergence}   between two distributions $P$ and $Q$ with support $\calX$ is 
$D(P\|Q)\triangleq\sum_{x\in\calX}P(x)\log\frac{P(x)}{Q(x)}$. 
The {\em binary KL-divergence}  is denoted as  $D(a\|b)\triangleq D(\mathrm{Bern}(a) \| \mathrm{Bern}(b)) = a\log (a/b)+(1-a)\log( (1-a)/(1-b))$ for any $a,b\in [0,1]$. 

An undirected graphical model is a multivariate probability distribution that factorizes to an undirected graph\cite{lauritzen1996graphical}. More precisely, a $p$-dimensional random vector $\bX=(X_{1},X_{2},\ldots,X_{p})$ is \emph{Markov} on an undirected graph $\calG=(\calV,\calE)$, where $\calV=[p]\triangleq\{1,\ldots,p\}$ is the vertex set and $\calE\subseteq\tbinom{\calV}{2}$ is the edge set, if $\Pr(X_{i} \mid X_{\calV\backslash \{i\}})=\Pr(X_{i} \mid X_{\calN(i)})$ for all $i\in[p]$, where $\calN(i)\triangleq\{j\in\calV \mid \{i,j\}\in\calE\}$ is the neighborhood of node $i$. The degree of  node $i$ is $d (i)\triangleq |\calN(i)|$, and the maximal degree of the graph is $d=\max_{i\in\calV} d (i)$. We focus on   tree-structured graphical models, which are   multivariate distributions that are Markov on  acyclic and connected  graphs.
\subsection{System Model}
We consider the binary vector $\bX=(X_{1},X_2,\ldots, X_p)\in\{+1,-1\}^p$, which is also called an \emph{Ising} model. The random vector $\bX$ is Markov on a tree $\calG=(\calV,\calE)$ and, in this paper, we focus on models in which each random variable has zero mean and the interactions between the nodes that are connected by an edge are   identical. This is also known as a {\em homogeneous Ising model   with zero external field}\cite{daskalakis2019testing,tandon2021sga}. The probability mass function can be expressed as
\begin{align}
	P_\lambda(\bx)=\frac{1}{Z(\lambda)}\exp\bigg(\lambda\sum_{ \{i,j\}\in\calE} x_{i}x_{j}\bigg), \label{eqn:ising}
\end{align}
where $\lambda\in\bbR$ is the parameter that describes the relationship  of adjacent nodes and $Z(\lambda)$ is the partition function. For the model in \eqref{eqn:ising}, we have that for all $\{i,j\}\in\calE$, 
\begin{align}
	&\Pr(X_{i}=+1,X_{j}=+1)=\Pr(X_{i}=-1,X_{j}=-1)=\frac{e^{\lambda}}{2(e^{\lambda}+e^{-\lambda})}\nonumber\\*
	&\Pr(X_{i}=+1,X_{j}=-1)=\Pr(X_{i}=-1,X_{j}=+1)=\frac{e^{-\lambda}}{2(e^{\lambda}+e^{-\lambda})}\triangleq\frac{\theta}{2}.
\end{align}
The correlation of adjacent nodes  (that form an edge)
is
\begin{equation}
\rho_{ij}\triangleq\bbE_{P_\lambda}[X_{i}X_{j}]=\frac{e^{\lambda}-e^{-\lambda}}{(e^{\lambda}+e^{-\lambda})}=1-2\theta\triangleq\rho    
\end{equation}
 for all $\{i,j\}\in\calE$. In this work, we consider the structure learning problem of homogeneous Ising tree models with zero external field parameterized by $\rho\in(0,1)$, i.e., $\theta\in(0,\frac{1}{2})$. We further assume that the underlying graph satisfies the 
following mild assumption.
\begin{assumption}[Small degree relative to graph size]\label{assump:smalldeg}
	The number of vertices $p$ and the maximal degree $d$ satisfy $p\geq 82d$.
\end{assumption}
Intuitively, the reason why this assumption is needed is because in \ac{alathe}, we first estimate the unknown correlation $\rho$ (see~\eqref{eq:rhohat}), similarly to estimating the height of the bricks  $H_i$ in  our  flood-and-breakage-repair example in Section~\ref{sec:work}. The number of nodes $p$, and hence, the number of edges $p-1$, should be sufficiently large so that the estimate of $\rho$  is statistically reliable.

\subsection{Problem Statement}
In this paper, we study the structure learning problem of the homogeneous Ising tree $\calT$ given the independent samples drawn from a distribution $P$ Markov on $\calT$, i.e, we are interested in recovering the edge set $\calE(\calT)$. Different from the 
previous works \cite{chow1968approximating,tan2011large}, we focus on the active learning setting. In general, as stated in Algorithm~\ref{algo:generalal} in Section~\ref{sec:appgenalgo}, the active learning algorithm is designed to be an iterative algorithm. In each stage, it acquires $n_{i}$ independent samples of a subset 
of vertices $S_{i}\subseteq [p]$ and learns the structure of the graph with all the samples that have been acquired before. The selection of the vertices set $S_{i}$ to sample in each stage depends on all the 
previous samples $\bX_{S_{j}}$ for $j=1,\ldots,i-1$. The passive learning algorithm can be regarded as a one-stage active learning algorithm for which $S_{1}=[p]$ and $n_{1}=n$.

In this paper, we focus on the \emph{asymptotic} behavior  of algorithms with   increasing number of vector samples. A vector sample of a graphical model $\calG$ is a vector that contains scalar samples from  all the vertices in $\calG$, while a 
sub-vector sample is a vector containing samples corresponding to a strict subset of the vertices in $\calG$. 
\begin{definition}[Exponent]
	The exponent of a sequence of events $\calA=\{\calA(n)\}_{n\in\bbN}$ (where $n$ denotes the number of vector samples) is defined as
	\begin{align}
		K(\calA)\triangleq \varliminf_{n\rightarrow \infty}-\frac{1}{n}\log\Pr\big(\calA(n)\big).
	\end{align}
\end{definition}
The exponent $K(\calA)$ characterizes the asymptotic behavior of the exponential decay of the sequence of  probabilities of the events $\{\calA (n)\}_{n\in\bbN}$. Roughly speaking, 
$\Pr(\calA(n))\approx \exp (-n K(\calA) )$, 
so the larger the exponent $K(\calA)$ is, the faster the   probability decays to $0$.
For the structure estimate $\hatcalT(n)$ of the homogeneous Ising tree $\calT$ with $n$ vector samples, our work focuses on the behavior of the exponent of the sequence of events $\{\hatcalT(n)\neq \calT\}_{n\in\bbN}$. However, the exponent of $\{\hatcalT(n)\neq \calT\}_{n\in\bbN}$ is intuitively difficult to improve with active learning techniques, because every edge has the same correlation $\rho$, and so each part of the tree is equally good or bad. This   symmetry seems to preclude the possibility of improving the error exponent by focusing on any   particular subset of nodes $S_{i} \subseteq[p]$. We show, somewhat surprisingly,
that by judiciously exploiting the minute but detectable statistical variations of the samples at
different nodes, our active learning algorithm can improve the error exponent for learning homogeneous trees by at least $40\%$
when the correlations are at least $0.8$. For all other values of the correlation coefficient, we also
observe commensurate, but more modest, improvements in the error exponent.

\section{Structure Learning Algorithms and Main Result} \label{sec:structure}
In this section, we describe both the passive (\ac{cl}-based) and active tree structure learning algorithms  (\ac{alathe}) and state the main result of our paper.
\subsection{Passive Structure Learning Algorithm}\label{sec:pasivlearn}
The \ac{cl} algorithm is a classical tree structure learning algorithm\cite{chow1968approximating}. It finds the \ac{mst} of the graph whose weights are the empirical estimates of the mutual information between vertices, i.e.,
\begin{align}
	\hatcalT_{\rm{CL}}(n)=\argmax_{\calT\in\frakT^{p}} \sum_{\{i,j\}\in\calE(\calT)} \hatI(X_{i};X_{j}),
\end{align}
where $\frakT^{p}$ is the set of all the trees with $p$ vertices, and $\hatI(X_{i};X_{j})$ is the empirical estimate of the mutual information between $X_{i}$ and $X_{j}$ computed using $n$ independent samples. The \ac{cl} algorithm takes $n$ vector samples (i.e., $np$ 
scalar samples) as the inputs and acquires the samples of all the vertices at once. Using the terminology in this paper, the \ac{cl} algorithm is a   \emph{passive} learning algorithm. To avoid the analyzing the empirical estimate 
of the mutual information which depends nonlinearly on the empirical distribution, we note that the mutual information between any two vertices $I(X_{i};X_{j})$ is an increasing function of the correlation $0<\rho_{ij}<1$ for Ising trees with zero external field. Thus, instead of taking the mutual 
information as the weights of edges, the \ac{scl} algorithm takes the empirical {\em correlation} estimates  $\hrho_{ij}(m) \triangleq \sum_{\ell=1}^{m}X_{i}^{(\ell)}X_{j}^{(\ell)}/ m$ where $m$ is the number of used samples, as the weights of the edges, i.e.,
\begin{align}\label{eq:cl}
	\hatcalT_{\Passive}(n)\triangleq\hatcalT_{\rm{SCL}}(n)=\argmax_{\calT\in\frakT^{p}} \sum_{e\in\calE(\calT)} \hrho_{e}(n),
\end{align}
where the subscript in $\hatcalT_{\rm{SCL}}$ stands for the \ac{scl} algorithm. The relationship between $\hatcalT_{\rm{SCL}}(n)$ and $\hatcalT_{\rm{CL}}(n)$ was  stated in  \cite[Lemma~A.2]{bresler2020learning}. We restate it here  for ease of reference.
\begin{proposition} \label{prop:SCL}
	For the homogeneous Ising tree-structured graphical model with zero external field  parameterized by $0<\rho<1$, $\hatcalT_{\rm{SCL}}(n)=\hatcalT_{\rm{CL}}(n)$ for all $n\in\bbN$. 
\end{proposition}
Since the homogeneous Ising tree $\calT$ is parameterized by (a single parameter) $\rho$, the exponent of the error events $\{\hatcalT_{\Passive}(n)\neq \calT\}_{n\in \bbN}$, i.e., the passive learning algorithm in Eqn.~\eqref{eq:cl} learns the wrong structure, can be expressed as
\begin{align}\label{eq:passiveee}
	K _{\Passive}(\calT,\rho)\triangleq \varliminf_{n\rightarrow \infty}-\frac{1}{n}\log \Pr \big(\hatcalT_{\Passive}(n)\neq \calT\big).
\end{align}
For homogeneous Ising trees, the passive error exponent $K _{\Passive}(\calT,\rho)$ can be expressed as an  explicit  function of $\rho$.
\begin{proposition}[Tandon, Tan, and Zhu \cite{tandon2020exact}]
	For the homogeneous Ising tree with zero external field $\calT$ parameterized by $0<\rho<1$, the error exponent of the \ac{scl} algorithm can be expressed as
	\begin{align}
		K _{\Passive}(\calT,\rho)=-\log\bigg(1-\frac{1-\rho}{2}\big(1-\sqrt{1-\rho^{2}}\big)\bigg).
	\end{align} 
\end{proposition}
\begin{remark}
	The error exponent for the passive learning algorithm on homogeneous Ising models does not depend on the tree structure $\calT$, but only depends on the edge 
	correlation \cite{tandon2020exact}. This is because the number of nodes is finite (and not growing) in our setting. In the following, we will denote $K _{\Passive}(\calT,\rho)$ as $K _{\Passive}(\rho)$ to suppress its dependence on the underlying structure~$\calT$.
\end{remark}

\begin{figure}[t] 
	\centering
	\subfigure[Values of $K _{\Passive}(\rho)$ for different $\rho$.]{
	\begin{minipage}[t]{0.45\linewidth}
	\centering
	\includegraphics[width=3in]{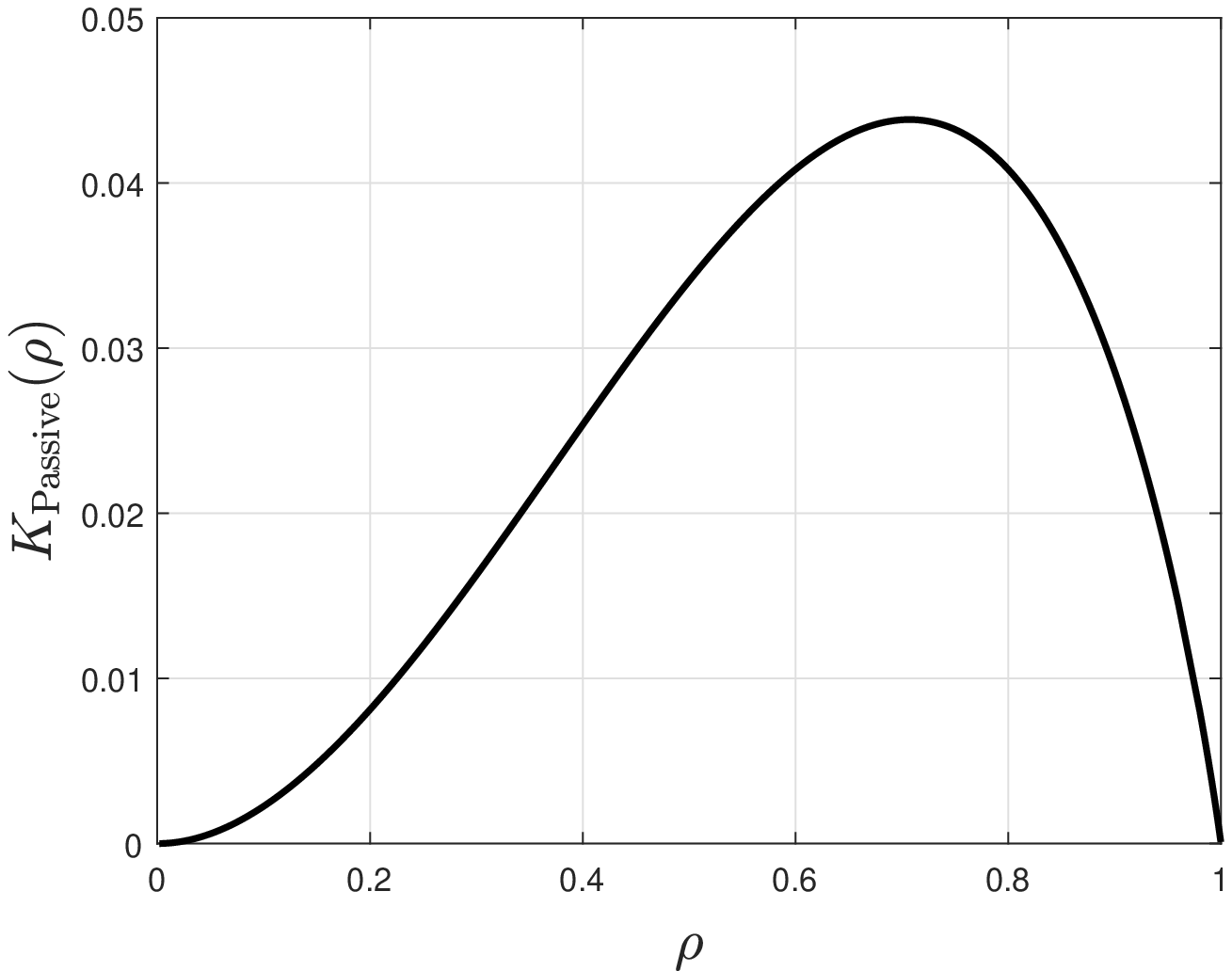}
	\end{minipage}%
	\label{fig:passiveee}
	}%
	\subfigure[Ratios of $K_{t}(\calT,\rho)$ and $K_{2}(\calT,\rho)$ for $t=3,4,5$.]{
	\begin{minipage}[t]{0.45\linewidth}
	\centering
	\includegraphics[width=3in]{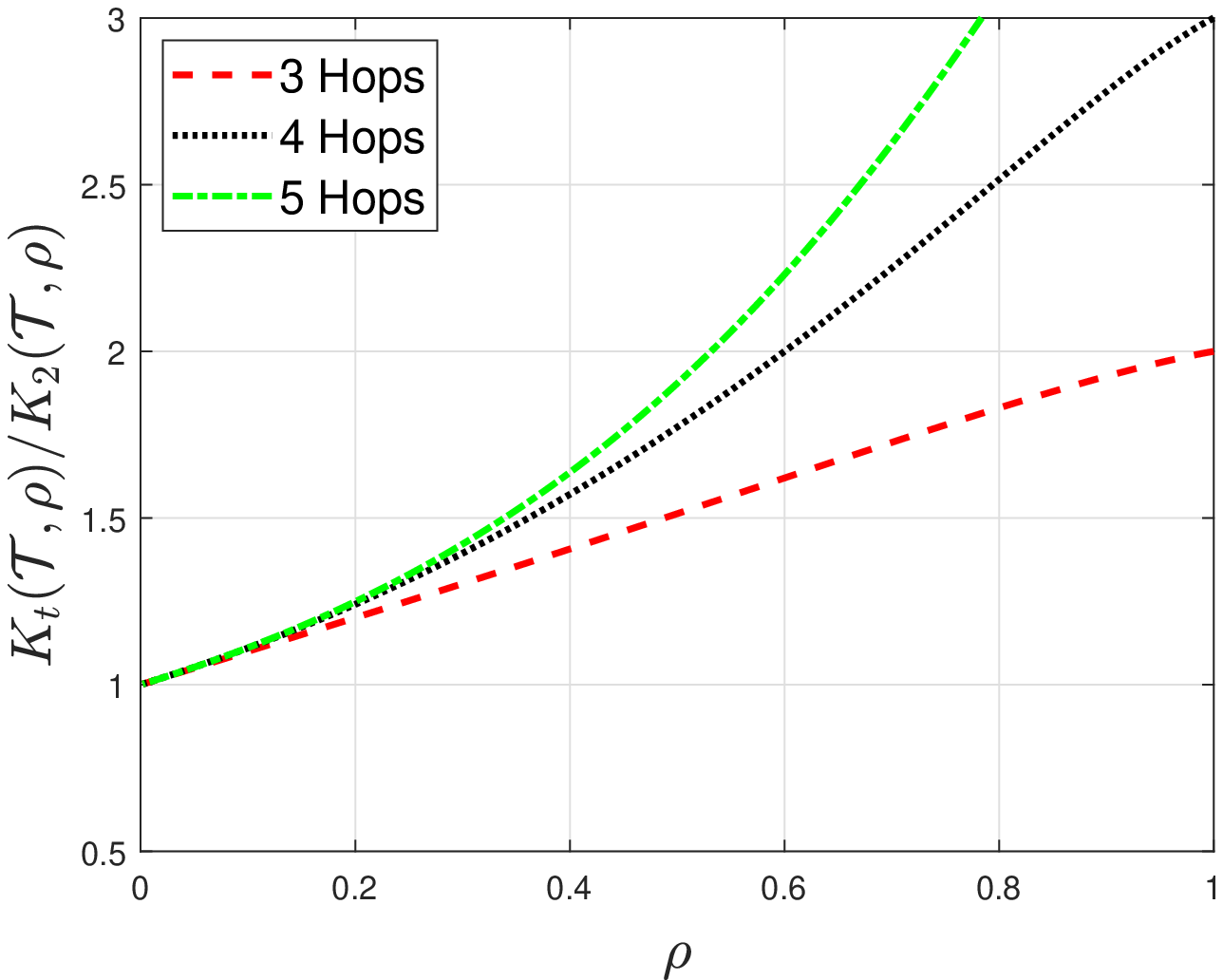}
	\end{minipage}%
	\label{fig:thopexp}
	}%
	\centering
	\caption{The values of $K _{\Passive}(\rho)$ and the ratios between $K_{t}(\calT,\rho)$ and $K_{2}(\calT,\rho)$.}
\end{figure}
	
Fig.~\ref{fig:passiveee} plots the values of $K _{\Passive}(\rho)$ for different values of $\rho$. When $\rho$ 
is close to $1$, the  exponent $K_{\Passive}(\rho)$ is small, which  implies that the error probability of the passive learning algorithm decreases very slowly with an increasing number of i.i.d.\ vector samples. In the following, we present an active learning algorithm \ac{alathe} that mitigates this problem.   \ac{alathe} significantly improves
the error exponent for large values of $\rho$. In particular, for $\rho\in[0.8,1)$, the error exponent of  \ac{alathe} will be shown to be at least $1.4$ times that of $K _{\Passive}(\rho)$.

\subsection{Active Structure Learning Algorithm}
 \ac{alathe}, whose pseudo-code is shown in Algorithm~\ref{algo:alathe} (the full implementation is provided at   \textcolor{blue}{\href{https://github.com/Zhang-Fengzhuo/Active-LATHE}{this    GitHub link}}), has two phases---the global learning  and  local refinement phases. In the following subsections, we describe these phases in detail.
\subsubsection{Global learning phase}
The global learning phase is designed to be iterative up to a stopping time. As indicated in Algorithm~\ref{algo:alathe}, in the $i^{{\rm th}}$ iteration of the global learning phase, we first implement the \ac{scl} algorithm 
with $\halpha_{i}n$ vector samples, where we set $\halpha_{1}=0.8$ for the first iteration; this is the minimum   of $\alpha$ shown in Table~\ref{table:alpha}, i.e., 
\begin{align}\label{eq:clvariant}
	\hatcalT_{\rm{SCL}}(\halpha_{i}n)=\argmax_{\calT\in\frakT^{p}} \sum_{e\in\calE(\calT)} \hrho_{e}(\halpha_{i}n).
\end{align}
With the \ac{cl} structure estimate, we then estimate the correlation $\rho$ of adjacent vertices as
\begin{align}\label{eq:rhohat}
	\hrho(\halpha_{i}n)=\frac{1}{p-1}\sum_{e\in\calE(\hatcalT_{\rm{SCL}}(\halpha_{i}n) )}\hrho_{e}(\halpha_{i}n) .
\end{align}
The estimate $\hrho(\halpha_{i}n)$ is used to find the corresponding $\halpha_{i+1} \triangleq \halpha\big(\hrho(\halpha_{i}n)\big)$ according to Table~\ref{table:alpha}. If $\halpha_{i+1}\leq \halpha_{i}$, we will stop and move on to the local refinement phase. Otherwise, we go to the next 
iteration. Note that   the $\halpha_{i}$'s are  random variables that are  functions of all the acquired samples. We denote the $\halpha_{i-1}$ in the last iteration as $\halpha$. As Table~\ref{table:alpha} suggests, this iterative design of the global learning phase aims to use more vector samples to learn the global structure of the model when $\rho$ is small. 
The intuition is that when $\rho$ is close to 0, all the vertices are close to being independent, and all the $t$-hop errors defined in Section~\ref{sec:pfpre} are close to one another (as shown in Fig.~\ref{fig:thopexp}). Hence, we need to collect more vector samples to resolve these $t$-hop errors.

\subsubsection{Local refinement phase}
In the local refinement phase,  we first detect the confident part of $\hatcalT_{\rm{SCL}}(\halpha_{i}n)$ using the correlation decay property \cite{nikolakakis2021predictive}.  
For a three-node subtree $i-j-k$ in a homogeneous tree with $0<\rho<1$, we have that $\rho_{ik}=\rho_{ij}\rho_{jk}$, which follows from the correlation decay property. Since $\rho_{ij}=\rho_{jk}=\rho$ and $0<\rho<1$, we have that
\begin{align}
	\rho_{ik}\leq \rho_{ij}\cdot\frac{11+9\rho}{20}  \quad\mbox{and}\quad \rho_{ik}\leq \rho_{jk}\cdot\frac{ 11+9\rho}{20}.
\end{align}
The somewhat arcane constants $11$, $9$, and $20$ have been lightly optimized to maximize the active error exponent. In general, for any three nodes $i$, $j$ and $k$ of the graph, we define the \emph{$(i,j,k)$-confident event} 
based on correlation estimates with $m$ vector samples as
\begin{align}\label{eq:confevent}
	\calC_{ijk}(m)\triangleq\Big\{\hrho_{ik}(m)\leq \hrho_{ij}(m)\cdot\frac{11+9\hrho(m)}{20}\Big\}\bigcap\Big\{\hrho_{ik}(m)\leq \hrho_{jk}(m)\cdot\frac{11+9\hrho(m)}{20}\Big\}.
\end{align}
When the number of samples $m$ is clear from the context, we abbreviate   $\hrho_{ij}(m)$ and $\calC_{ijk}(m)$ as $\hrho_{ij}$ and $\calC_{ijk}$ respectively.

\begin{algorithm}[t]
	\caption{Active Learning Algorithm for Trees with Homogeneous Edges}
	\textbf{Input:}  The number of nodes $p$, the number of vector samples $n$\\
	\textbf{Output:} Estimated tree structure $\hatcalT_{\Active}(n)$ \\
	\textbf{Procedure:}
	\begin{algorithmic}[1]\label{algo:alathe}
		\STATE \# Global Learning Phase\\
		\STATE Initialize $i=1$, $\halpha_{0}=0$, $\halpha_{1}=0.8$.\\
		\WHILE{$\halpha_{i}>\halpha_{i-1}$}
			\STATE Acquire $\lfloor\halpha_{i}n\rfloor-\lfloor\halpha_{i-1}n\rfloor$ vector samples of all the nodes.\\
			\STATE Learn the \ac{cl} tree $\hatcalT_{\rm{SCL}}(\lfloor\halpha_{i}n\rfloor)$ as in Eqn.~\eqref{eq:clvariant}.\\
			\STATE Calculate $\hrho(\lfloor\halpha_{i}n\rfloor)$ as in Eqn.~\eqref{eq:rhohat}.\\
			\STATE Set $\halpha_{i+1}$ according to $\hrho(\lfloor\halpha_{i}n\rfloor)$ and Table~\ref{table:alpha}.\\
			\STATE $i=i+1$.\\
		\ENDWHILE
		\STATE Set $\halpha=\halpha_{i-1}$. \\
		\STATE \# Local Refinement Phase\\
		\STATE Set $\scC_{\mathrm{V}}(\lfloor\halpha_{i}n\rfloor)^{\complement}=\emptyset$ and $\scC_{\mathrm{E}}(\lfloor\halpha_{i}n\rfloor)^{\complement}=\emptyset$.\\
		\FOR {All node triples $(i,j,k)$ where $\{i,j\},\{j,k\}\in \calE\big(\hatcalT_{\rm{SCL}}(\lfloor\halpha_{i}n\rfloor)\big)$ }
			\IF {$\calC_{ijk}(\lfloor\halpha_{i}n\rfloor)$  (defined in Eqn.~\eqref{eq:confevent}) does not hold}
				\STATE Add $\{i,j\},\{j,k\}$ to $\scC_{\mathrm{E}}(\lfloor\halpha_{i}n\rfloor)^{\complement}$ and add nodes $i,j,k$ to $\scC_{\mathrm{V}}(\lfloor\halpha_{i}n\rfloor)^{\complement}$.
			\ENDIF
		\ENDFOR
		\STATE Acquire $\lfloor(1-\halpha)np/|\scC_{\mathrm{V}}(\lfloor\halpha_{i}n\rfloor)^{\complement}| \rfloor$  sub-vector samples corresponding only to the nodes in $\scC_{\mathrm{V}}(\lfloor\halpha_{i}n\rfloor)^{\complement}$.\\
		\STATE Implement the \ac{scl} algorithm to relearn the structure of connected components in the subgraph formed by $\scC_{\mathrm{V}}(\lfloor\halpha_{i}n\rfloor)^{\complement}$ and $\scC_{\mathrm{E}}(\lfloor\halpha_{i}n\rfloor)^{\complement}$ using all the sub-vector samples corresponding to each connected component.
		\STATE Replace every connected component with the relearnt structure.
		\STATE Return the derived $\hatcalT_{\Active}(n)$.
	\end{algorithmic}
\end{algorithm}
\begin{remark}
	The event $\calC_{ijk}$ implies that $\{i,k\}\notin\calE(\hatcalT_{\rm{SCL}})$. In fact, if $\calC_{ijk}$ holds, we are ``overconfident'' about the fact  that the subgraph $i-j-k$ among nodes $i$, $j$ and $k$ is correct, since the \ac{scl} algorithm only requires $\hrho_{ik}< \hrho_{ij}$ and $\hrho_{ik}< \hrho_{jk}$ to recover the correct underlying subtree $i-j-k$.
\end{remark}
Intuitively, if the confident event $\calC_{ijk}$ holds, $\hrho_{ik}$ is sufficiently smaller than $\hrho_{ij}$ and $\hrho_{jk}$. This condition suffices  for the \ac{scl} algorithm to find the correct subgraph among these three nodes. To describe the subroutines of the local refinement steps clearly, we define the confident edges, the unconfident edges and the unconfident nodes as follows.
\begin{table}[H]
	\normalsize
	\centering
	\begin{tabular}{|c|c|c|c|c|c|c|c|}
	\hline
	$\hrho$ & $(0,0.02)$ & $[0.02,0.07)$ & $[0.07,0.16)$ & $[0.16,0.34)$ & $[0.34,0.53)$ & $[0.53,0.76)$ & $[0.76,1)$\\
	\hline
	$\halpha(\hrho)$ & 1 & 0.995 & 0.985 & 0.95 & 0.9 & 0.85 & 0.8 \\
	\hline
	\end{tabular}
	\caption{Correlation coefficients $\hrho$ and their corresponding $\halpha(\hrho)$.}
	\label{table:alpha}
\end{table}
\begin{definition}[Confident edges, Unconfident edges and Unconfident nodes]\label{def:confide}
	For the \ac{cl} tree estimated with $m$ vector samples $\hatcalT_{\rm{SCL}}(m)$, the confident and unconfident edge sets are defined as
	\begin{align}
		\scC_{\mathrm{E}}(m)&\triangleq \Big\{\{i,j\}\in\calE\big(\hatcalT_{\rm{SCL}}(m)\big) ~\Big|~\forall\, \{j,k\},\{l,i\}\in\calE\big(\hatcalT_{\rm{SCL}}(m)\big),~\calC_{ijk}(m) \cap\calC_{lij}(m) \text{ hold}\Big\}
		\end{align}
		and $\scC_{\mathrm{E}}(m)^{\complement} \triangleq \calE\big(\hatcalT_{\rm{SCL}}(m)\big)\backslash \scC_{\mathrm{E}}(m)$,
	respectively. The edges in $\scC_{\mathrm{E}}(m)$ (resp.\ $\scC_{\mathrm{E}}(m)^{\complement}$) are called confident (resp.\ unconfident) edges. The set of unconfident nodes  is defined as
	\begin{align}
		\scC_{\mathrm{V}}(m)^{\complement}\triangleq \big\{i\in\calV \ \big| \ \exists\,  j\in\calV \text{ s.t. }\{i,j\}\in \scC_{\mathrm{E}}(m)^{\complement} \big\}.\nonumber
	\end{align}
\end{definition}
\begin{remark}
    The set of unconfident nodes can be equivalently defined as the set of all the terminal nodes of all the unconfident edges.
\end{remark}
We denote the number of unconfident nodes as $\tilp(m)\triangleq |\scC_{\mathrm{V}}(m)^{\complement}|$. Note that $\tilp(m)$  is a function of the number of the vector samples used to learn the tree  $m$. As shown in Algorithm~\ref{algo:alathe}, the local refinement phase first checks the occurrences of each of the events $\calC_{ijk}$ for all the triples $(i,j,k)$ such that 
$\{i,j\}, \{j,k\}\in \calE(\hatcalT_{\rm{SCL}}(\lfloor\halpha_{i}n\rfloor))$. Since the confident edges are correct with higher probability compared to the unconfident edges, we then only assign samples to the unconfident nodes $\scC_{\mathrm{V}}(\lfloor\halpha_{i}n\rfloor)^{\complement}$ and relearn the 
subgraph of these unconfident nodes. The number of scalar samples assigned to each unconfident node is about
\begin{align}
\bigg[\halpha+(1-\halpha)\frac{p}{\tilp(\lfloor\halpha_{i}n\rfloor)}\bigg]n\geq n,    
\end{align}
 so the probability of the event that the 
subgraph of the unconfident nodes is wrong is reduced compared to the passive learning algorithm. Furthermore, we will show (see Lemma~\ref{lem:packnum}) that 
\begin{align}
\bigg[\halpha+(1-\halpha)\frac{p}{\tilp(\lfloor\halpha_{i}n\rfloor)}\bigg]n\geq c_{\rho}n,   \label{eqn:unconf_samples}  
\end{align}
 where $c_{\rho}$ is the  piecewise constant  function  defined in Table~\ref{table:crho}. Eqn.~\eqref{eqn:unconf_samples} shows that the number of scalar samples assigned to the unconfident nodes is strictly  larger than that in the passive learning scenario.

\subsection{The Main Result} \label{sec:main_res}
With the homogeneous Ising tree $\calT$ parameterized by $\rho$, the exponent of the (sequence of) error events $\{\hatcalT_{\Active}(n)\neq \calT\}_{n\in\bbN}$ (this event corresponds to  \ac{alathe}  learning the wrong structure), can be expressed as
\begin{align}\label{eq:activeee}
	K _{\Active}(\calT,\rho)\triangleq \varliminf_{n\rightarrow \infty}-\frac{1}{n}\log \Pr\big(\hatcalT_{\Active}(n)\neq \calT\big).
\end{align}
Our main theorem is stated as follows.

\begin{table}[t]
	\normalsize
	\centering
	\begin{tabular}{|c|c|c|c|c|c|c|c|}
	\hline
	$\rho$ & $(0,0.03)$ & $[0.03,0.1)$ & $[0.1,0.2)$ & $[0.2,0.4)$ & $[0.4,0.6)$ & $[0.6,0.8)$ & $[0.8,1)$\\
	\hline
	$c_{\rho}$ & 1 & 1.01 & 1.03 & 1.08 & 1.19& 1.29 & 1.4 \\
	\hline
	\end{tabular}
	\caption{Correlation coefficients $\rho$ and their corresponding $c_{\rho}$.}
	\label{table:crho}
\end{table}

\begin{theorem}\label{thm:mainthem}
	For the homogeneous Ising tree with zero external field $\calT$ parameterized by $0<\rho<1$ that satisfies Assumption \ref{assump:smalldeg}, the error exponent of the \ac{alathe} algorithm is at least $c_{\rho}$ times the 
	error exponent of the passive learning algorithm $K _{\Passive}(\rho)$, i.e.,
	\begin{align}
		K _{\Active}(\calT,\rho)\geq c_{\rho}K _{\Passive}(\calT,\rho) \quad \forall\, \calT\in\frakT^{p},
	\end{align}
	where $c_{\rho}$ is the piecewise constant function   defined in Table~\ref{table:crho}.
\end{theorem} 
Theorem \ref{thm:mainthem} shows that the error exponent of the \ac{alathe} algorithm $K _{\Active}(\calT,\rho)$ is no less than the error exponent of the passive learning algorithm $K _{\Passive}(\rho)$. In fact, it is strictly larger than $K _{\Active}(\calT,\rho)$ by a multiplicative factor of $c_{\rho}$. In particular, when 
$\rho\in[0.8,1)$, in which $K _{\Passive}(\rho)$ suffers from a rapid decline as $\rho\to 1^-$ (see Fig.~\ref{fig:passiveee}), the \ac{alathe} algorithm improves the error exponent by at least 40\% relative to  its passive counterpart, which implies that the probability of error decays significantly faster if active learning is allowed. 

\section{Simulation Results} \label{sec:sim}

In this section, we present   simulation results to demonstrate the efficacy of the \ac{alathe} algorithm compared to its passive counterpart. Samples are generated from three different tree  structures: (1) A chain graph with $p=300$; (2) A \ac{hmm}-structured graph with $p=300$ (where all the nodes are observed); and (3) A binary tree with $p = 255$ (seven levels). The correlations between adjacent nodes $\rho$ are set to $0.9$, $0.7$, and $0.5$. Theorem 3 in \cite{tandon2020exact} provides an exact asymptotic expression for the best possible error probability, and it serves as a baseline of the simulation; this is indicated as ``Passive \ac{scl}: Theory'' in  Figs.~\ref{fig:chain},  \ref{fig:hmm}, and~\ref{fig:binary_tree}. Note that no other passive algorithm can perform better than that given by Theorem 3 in \cite{tandon2020exact} (since this is based on the maximum likelihood or minimum error probability principle), so it also serves as a {\em bona fide} impossibility result for tree structure learning using passive strategies.

\begin{figure}[t] 
	\centering
	\subfigure[$p=300$, $\rho=0.9$.]{
	\begin{minipage}[t]{0.32\linewidth}
	\centering
	\includegraphics[width=2.2in]{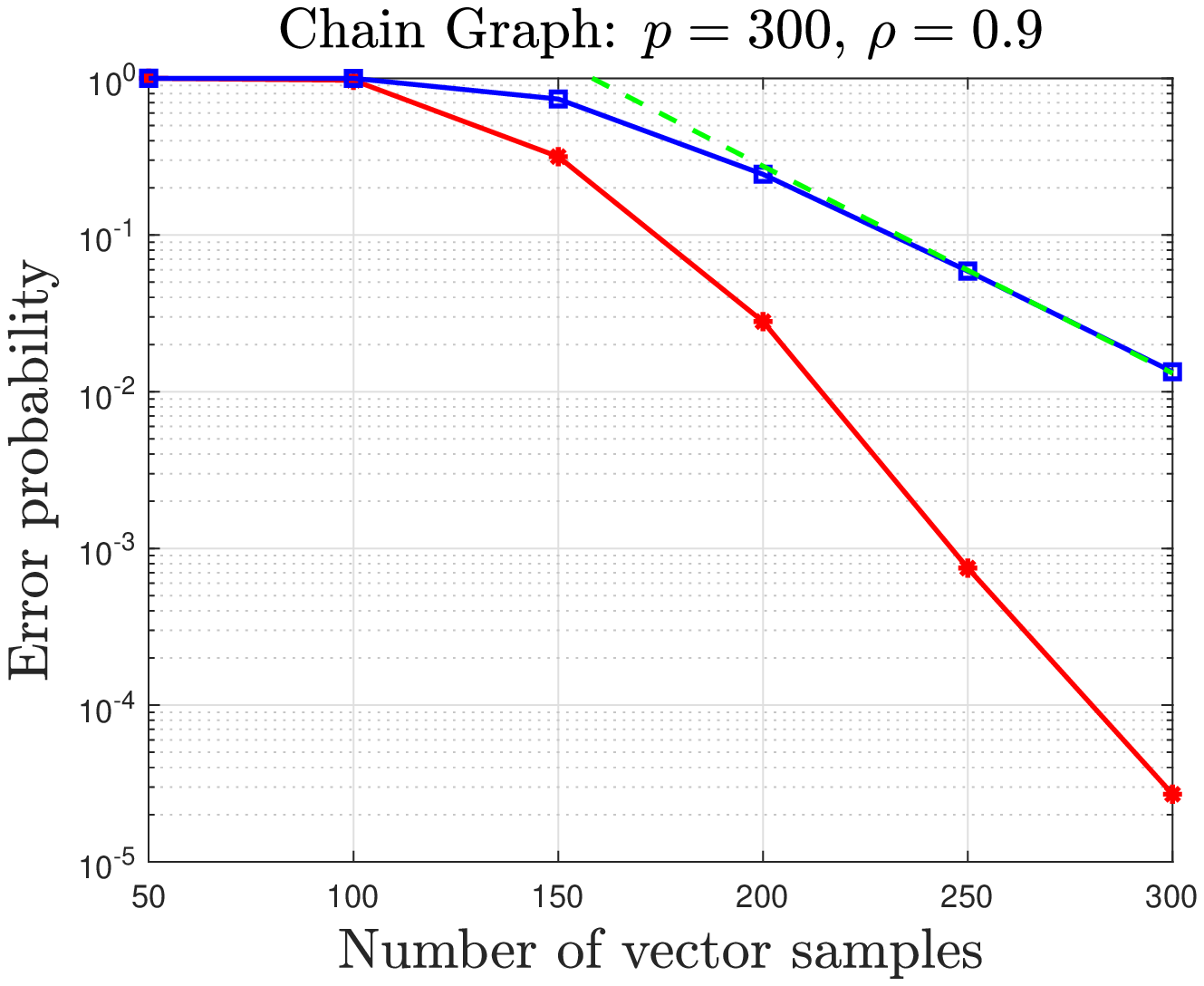}
	\end{minipage}%
	\label{fig:chain_9}
	}%
	\subfigure[$p=300$, $\rho=0.7$.]{
	\begin{minipage}[t]{0.32\linewidth}
	\centering
	\includegraphics[width=2.2in]{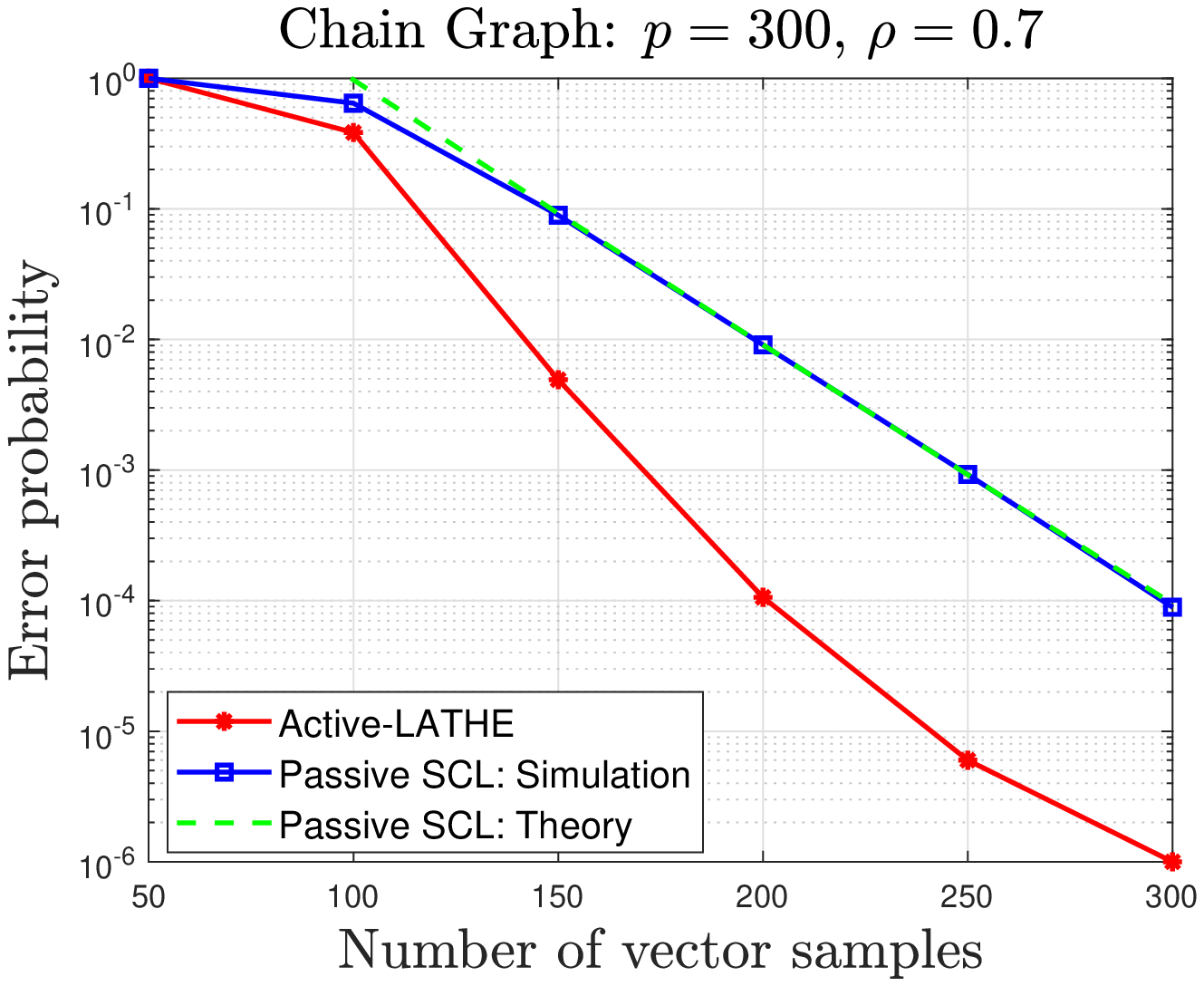}
	\end{minipage}%
	\label{fig:chain_7}
	}%
	\subfigure[$p=300$, $\rho=0.5$.]{
	\begin{minipage}[t]{0.32\linewidth}
	\centering
	\includegraphics[width=2.2in]{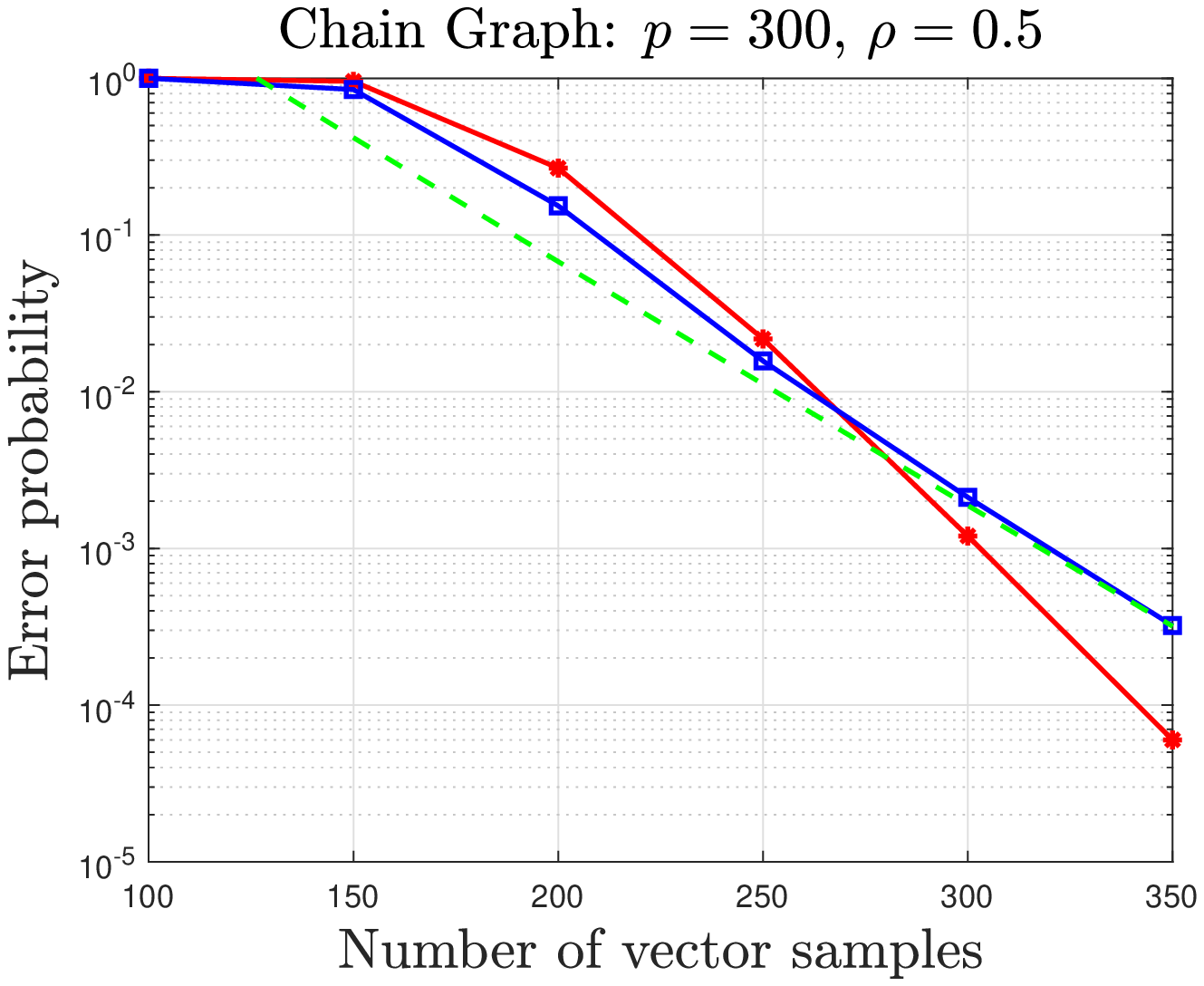}
	\end{minipage}%
	\label{fig:chain_5}
	}%
	\caption{Simulated error probabilities of structure learning algorithms on  chains}
	\label{fig:chain}
\end{figure}
\begin{figure}[t] 
	\centering
	\subfigure[$p=300$, $\rho=0.9$.]{
	\begin{minipage}[t]{0.32\linewidth}
	\centering
	\includegraphics[width=2.2in]{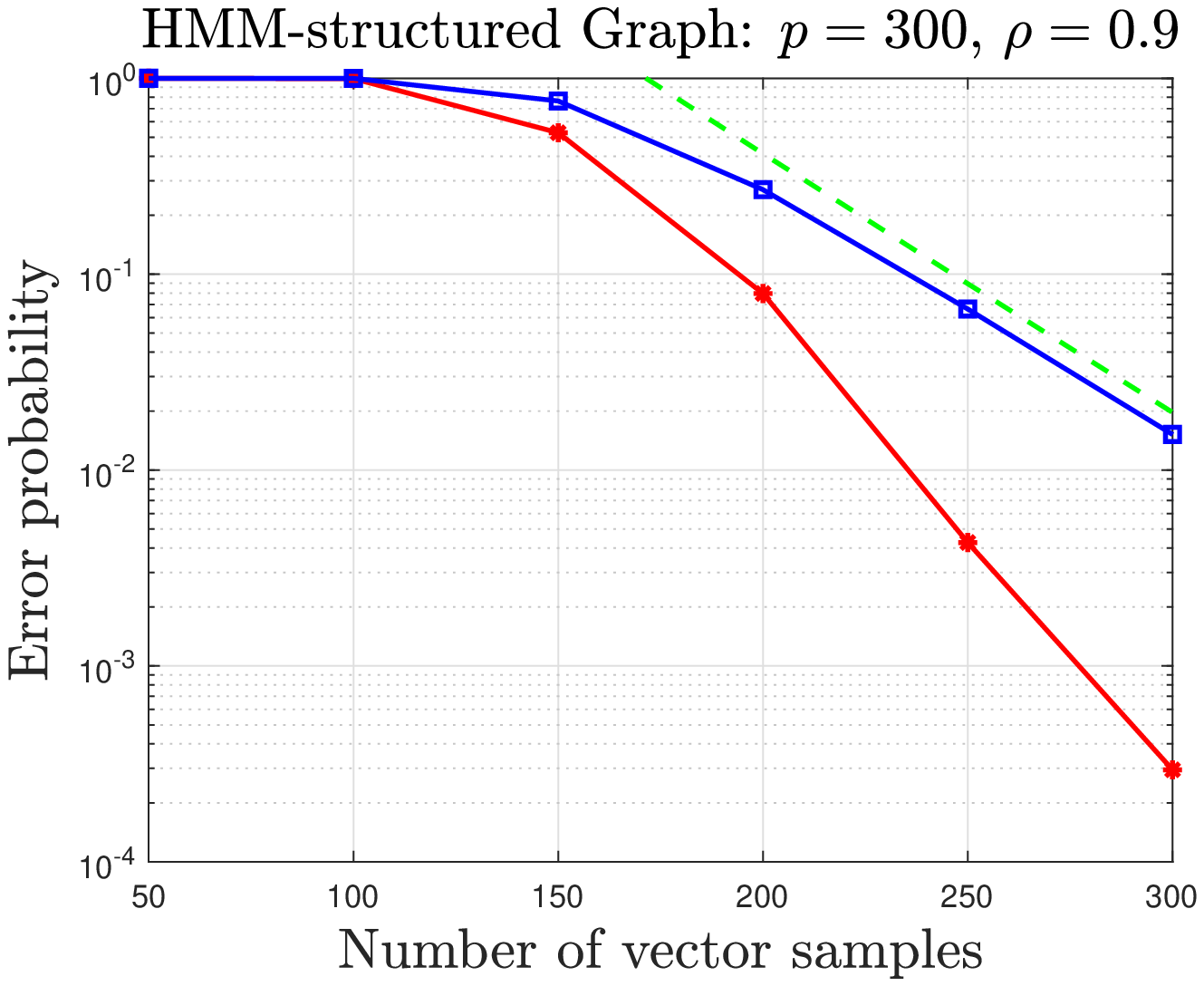}
	\end{minipage}%
	\label{fig:hmm_9}
	}%
	\centering
	\subfigure[$p=300$, $\rho=0.7$.]{
	\begin{minipage}[t]{0.32\linewidth}
	\centering
	\includegraphics[width=2.2in]{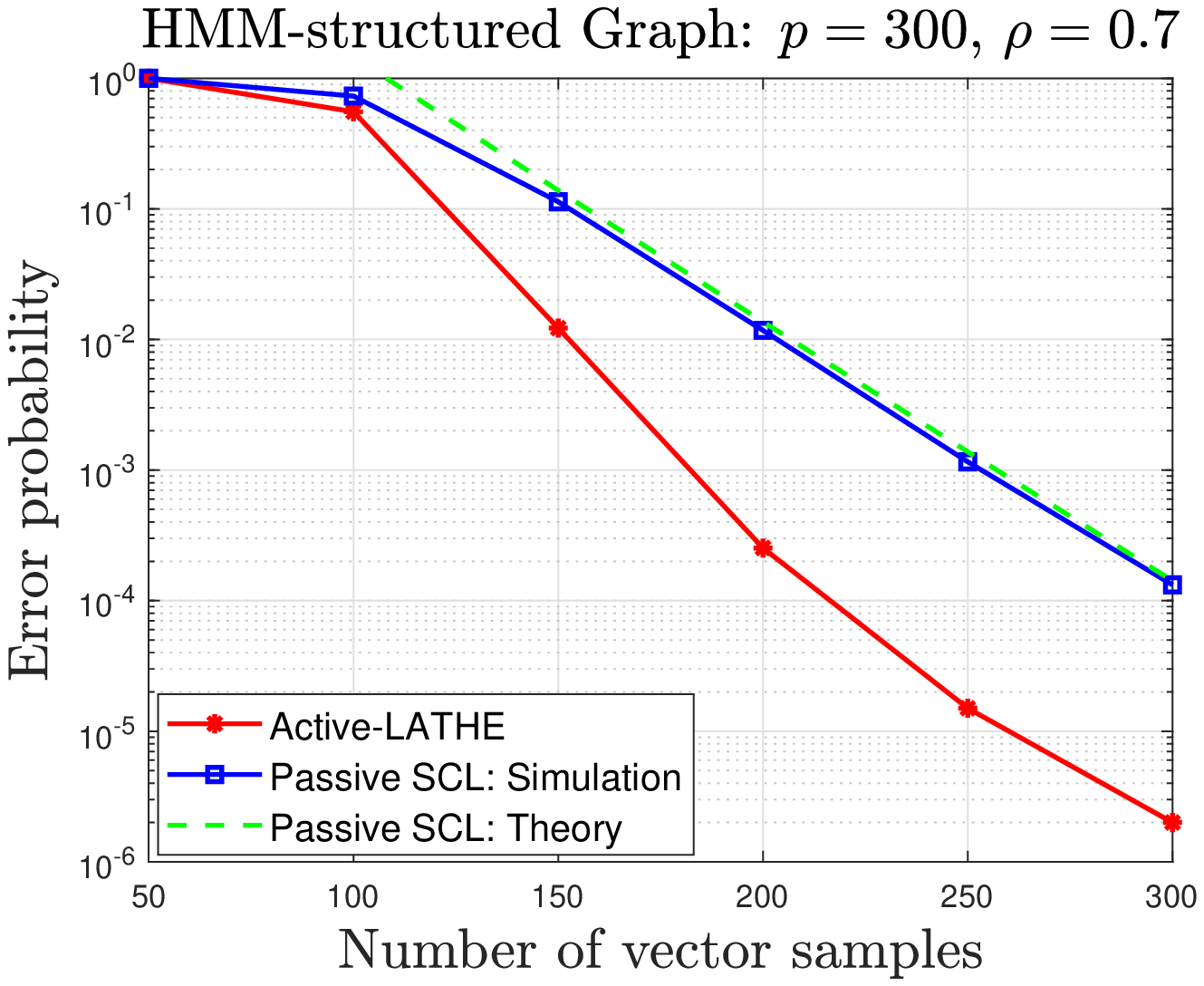}
	\end{minipage}%
	\label{fig:hmm_7}
	}%
	\centering
	\subfigure[$p=300$, $\rho=0.5$.]{
	\begin{minipage}[t]{0.32\linewidth}
	\centering
	\includegraphics[width=2.2in]{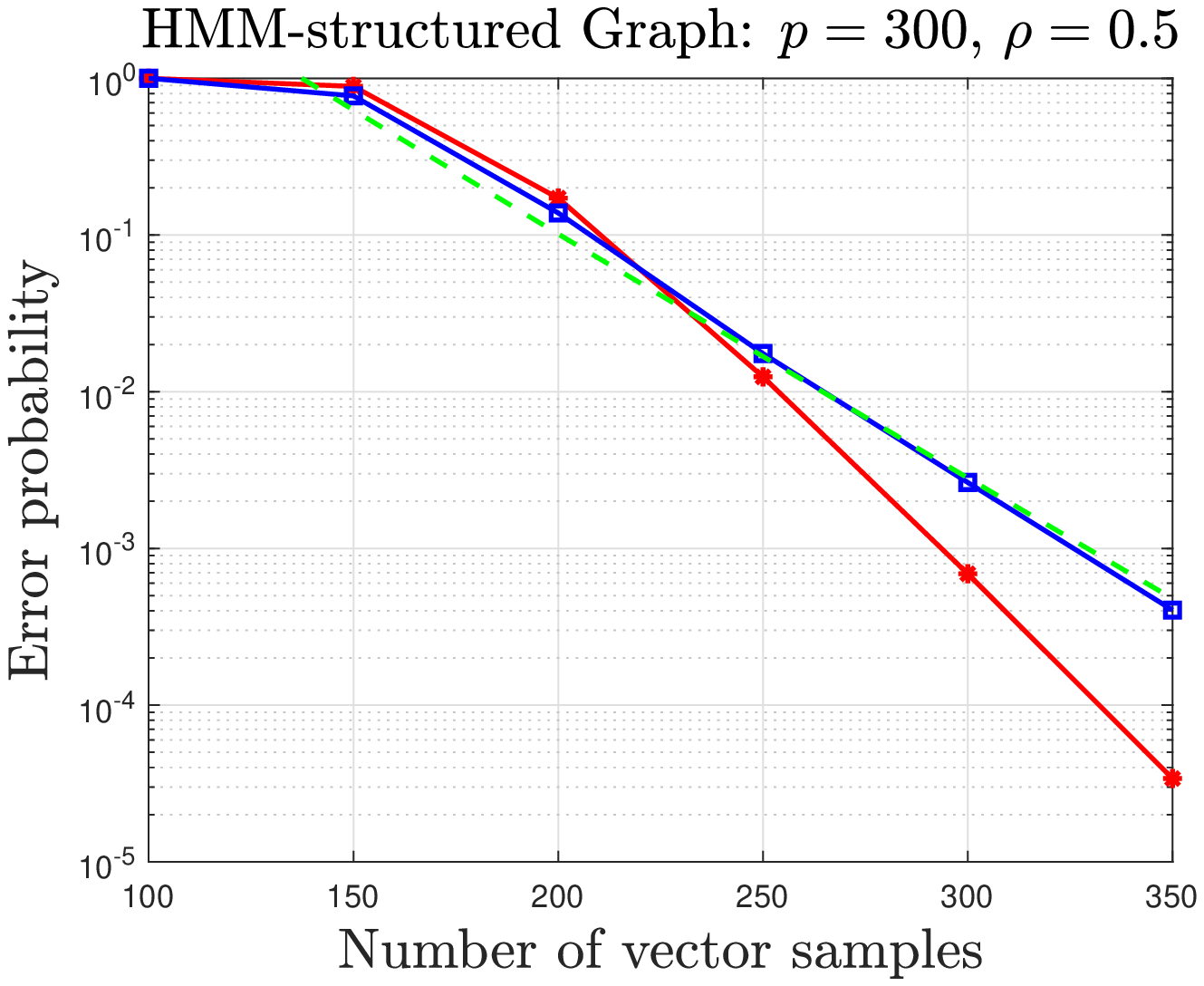}
	\end{minipage}%
	\label{fig:hmm_5}
	}%
	\centering
	\caption{Simulated error probabilities of structure learning algorithms on HMMs }
	\label{fig:hmm}
\end{figure}
\begin{figure}[t]
	\centering
	\subfigure[$p=255$, $\rho=0.9$.]{
	\begin{minipage}[t]{0.32\linewidth}
	\centering
	\includegraphics[width=2.2in]{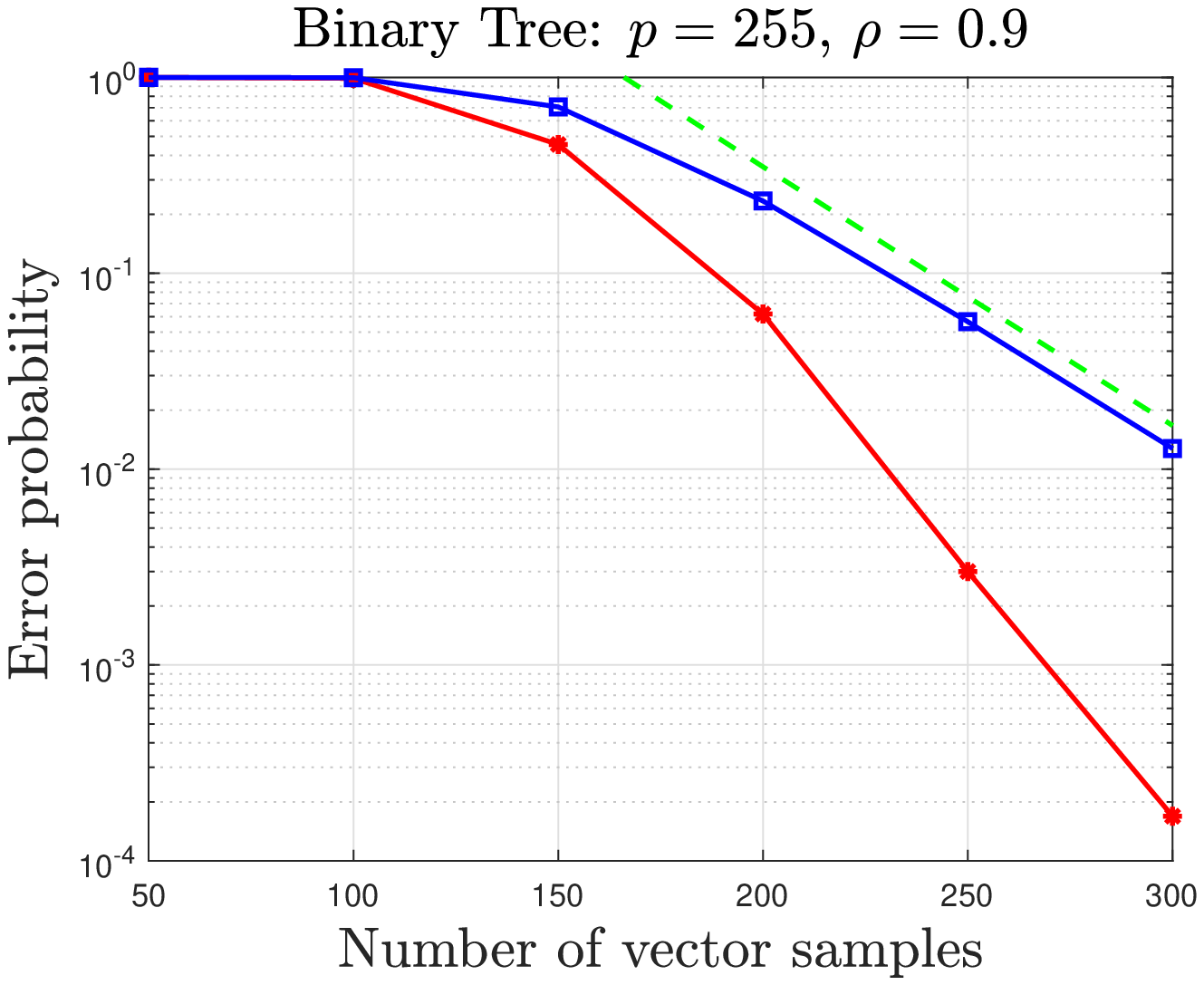}
	\end{minipage}%
	\label{fig:bt_9}
	}%
	\centering
	\subfigure[$p=255$, $\rho=0.7$.]{
	\begin{minipage}[t]{0.32\linewidth}
	\centering
	\includegraphics[width=2.2in]{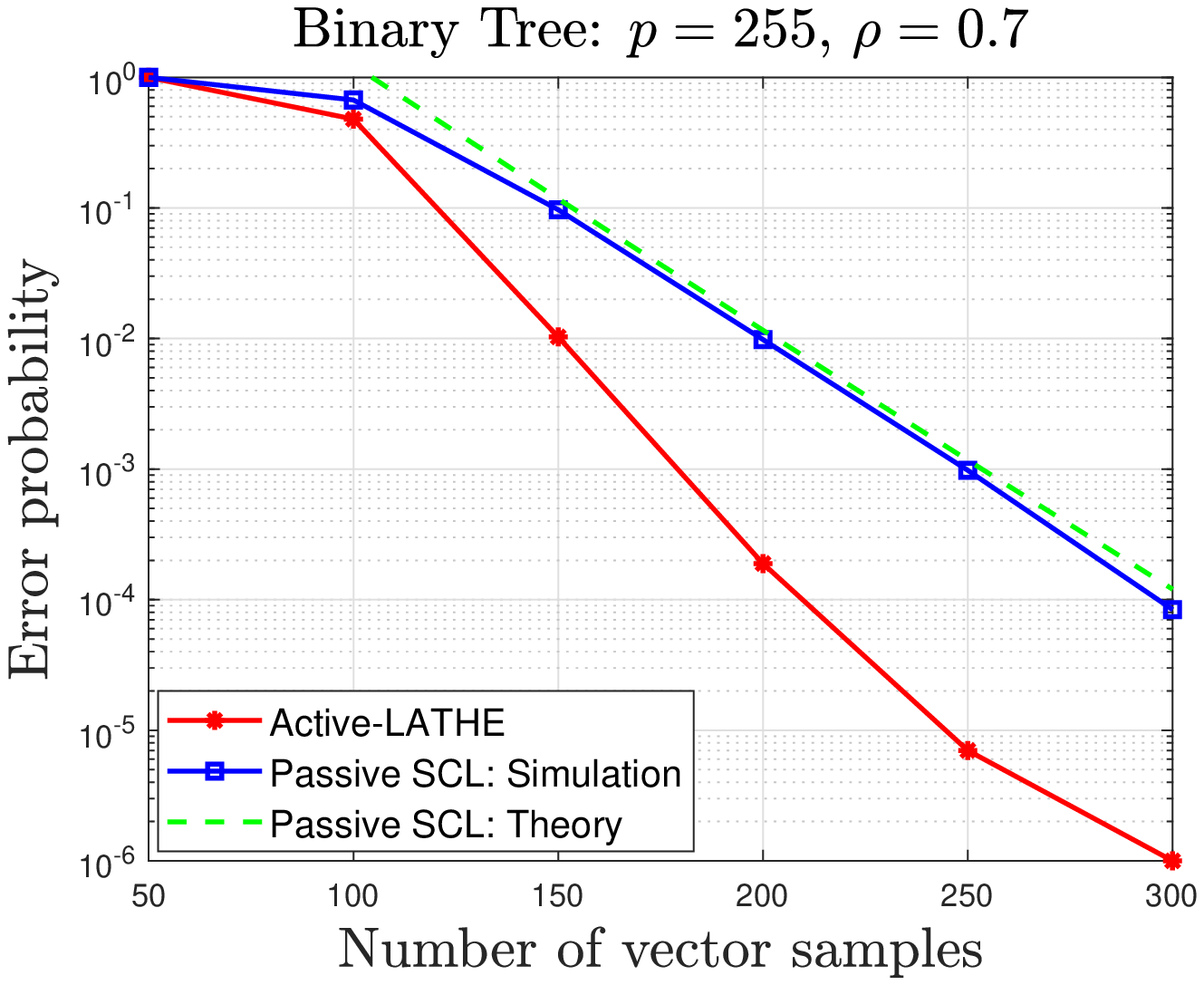}
	\end{minipage}%
	\label{fig:bt_7}
	}%
	\centering
	\subfigure[$p=255$, $\rho=0.5$.]{
	\begin{minipage}[t]{0.32\linewidth}
	\centering
	\includegraphics[width=2.2in]{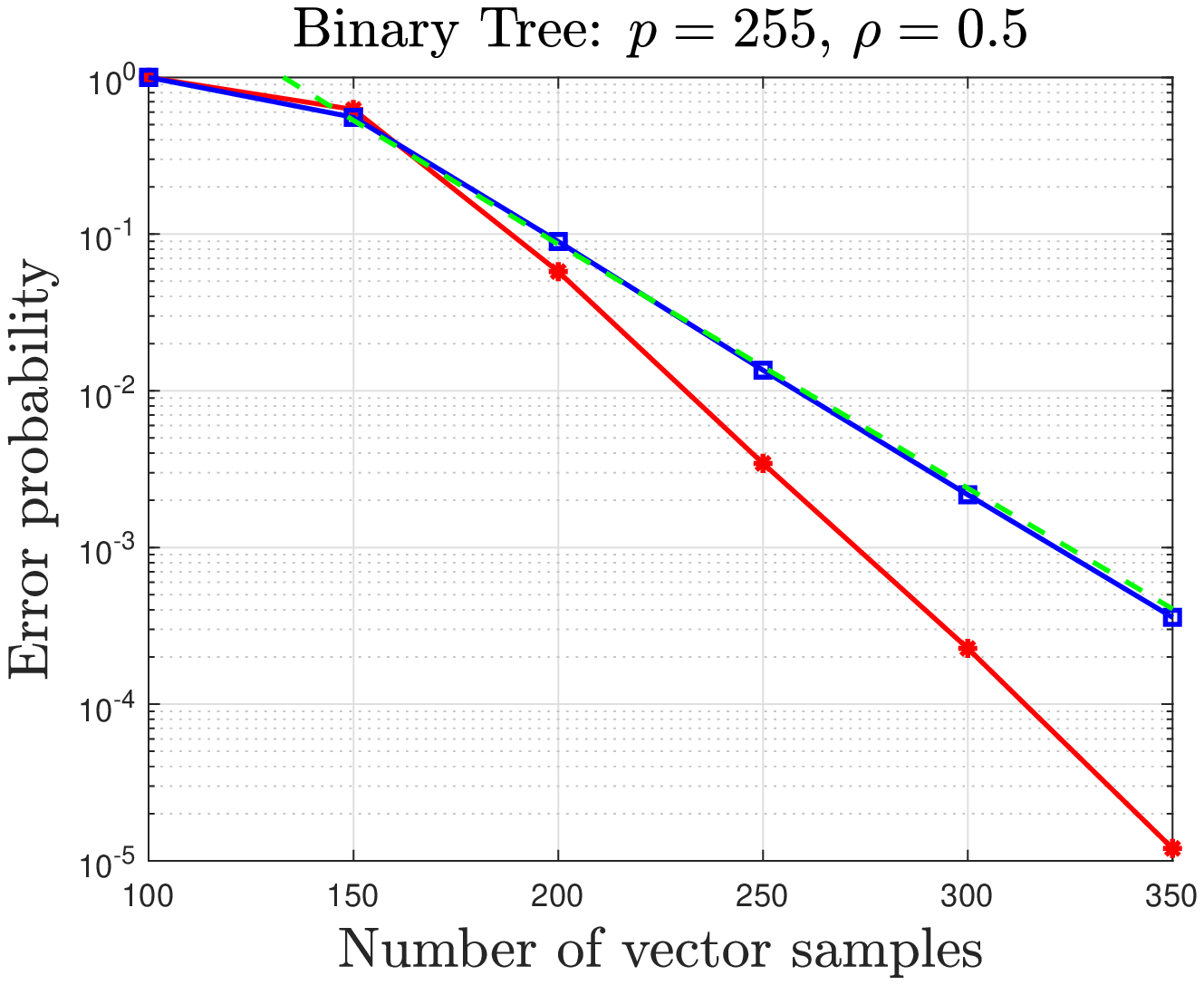}
	\end{minipage}%
	\label{fig:bt_54}
	}%
	\centering
	\caption{Simulated error probabilities of structure learning algorithms on   binary trees}
	\label{fig:binary_tree}
\end{figure}

Figs.~\ref{fig:chain} to \ref{fig:binary_tree} corroborate the behavior of the (passive) exponent in Fig.~\ref{fig:passiveee} and demonstrate the deficiencies of the passive learning algorithm, especially when $\rho $ is close to $1$. Indeed, in Fig.~\ref{fig:chain}, the error probabilities with $250$ vector samples on the chain graphs with different $\rho$'s show that the Ising tree with $\rho=0.9$ is the most difficult to learn (among $\rho \in \{0.5, 0.7, 0.9\}$), which coincides with the intuition gleaned by the exponent shown in Fig.~\ref{fig:passiveee}. The match between the experimental results and the exact asymptotic expression (indicated by ``Passive \ac{scl}: Theory'') also substantiates the slower decrease of the error probability when $\rho$ is close to $1$ (compared to smaller values of $\rho$). Similar phenomena can be observed in Figs.~\ref{fig:hmm} and~\ref{fig:binary_tree} for other tree structures.

Figs.~\ref{fig:chain} to \ref{fig:binary_tree} clearly demonstrate the effectiveness of the \ac{alathe} algorithm whose error probabilities are much smaller than those of the passive Chow--Liu learning algorithm, especially when $n$ is sufficiently large. 
The {\em slopes} of the error curves correspond to the error exponents.  
The improvements to  the slopes   of the \ac{alathe} algorithm  (the slopes are more negative for \ac{alathe} compared to the passive algorithm) are especially pronounced when $\rho$ is large. 
The large reductions of the error probability (the error exponent increases by at least $40\%$) shown in Figs.~\ref{fig:chain_9} to~\ref{fig:bt_9} greatly mitigate the inherent deficiencies of the passive learning for large $\rho$ as discussed in Section~\ref{sec:pasivlearn}. The code to reproduce the figures is available at  \textcolor{blue}{\href{https://github.com/Zhang-Fengzhuo/Active-LATHE}{this  GitHub link}}.

\section{Proof Overview}\label{sec:pfview}
In this section, we provide the overview of the proof of Theorem~\ref{thm:mainthem}. We first introduce the definitions and key results that will be used in the proof. Then we provide the outline of the proof of Theorem~\ref{thm:mainthem}. The proofs of the supporting lemmas are provided in Section~\ref{sec:app}.

\subsection{Preliminaries}\label{sec:pfpre}
We begin with definitions of the errors that occur in the learning of the \ac{mst} in the global learning phase which uses $\halpha n$ vector samples.  
\begin{definition}[$t$-hop error] \label{def:thop}
	Let $i_{k_{1}}-i_{k_{2}}-\cdots -i_{k_{t+1}}$ be a   path $\calP_{t}$ of length $t$. The corresponding $t$-hop error is defined as the event $E_{t}(\calP_{t},m)=\bigcup_{j=1}^{t}\{\hrho_{k_{1}k_{t+1}}(m)\geq \hrho_{k_{j}k_{j+1}}(m)\}$. A graph 
	$\calG$ has a $t$-hop error if one of its subgraphs is a path that has a $t$-hop error. This error event   is denoted as $E_{t}(\calG,m)=\cup_{\calP_{t} \subseteq \calG}E_{t}(\calP_{t},m).$
\end{definition}
\begin{remark}
	The \ac{cl} estimate of $\calT$ is the tree $\hatcalT_{{\rm SCL}}(m)$, which is the \ac{mst} of the graph with edge weights $\hrho_{e}(m)$. The event $E_{t}(\calT,m)$ implies that a path $\calP_{t}$ in $\calT$ is wrongly estimated, and 
	there is an edge between the terminal nodes of $\calP_{t}$ in $\hatcalT_{{\rm SCL}}(m)$, which are the terminal nodes of a $t$-long path in the nominal tree. Thus, $\hatcalT_{{\rm SCL}}(m)$ is said to have a $t$-hop error if $E_{t}(\calT,m)$ occurs, and 
	we will use the notation $E_{t}(\hatcalT_{{\rm SCL}},\calT,m)$ to denote the same event as $E_{t}(\calT,m)$ when we want to emphasize that the learned \ac{cl} tree $\hatcalT_{{\rm SCL}}$ is wrong.
\end{remark}
\begin{definition}[$t$-hop error exponent]
For a homogeneous tree Ising with zero external field $\calT$ parameterized by $0<\rho<1$, the exponent corresponding to the $t$-hop error event is defined as
\begin{align}\label{eq:thopee}
	K _{t}(\calT,\rho)\triangleq \varliminf_{n\rightarrow \infty}-\frac{1}{n}\log\Pr \big(E_{t}(\calT,n)\big),
\end{align}
where $n$ is the number of vector samples used for estimating correlations between nodes $\hrho_{ij}$.
\end{definition} 
In the following, we also use $K _{t}(\hatcalT_{{\rm SCL}},\calT,\rho)$ to denote the exponent $K _{t}(\calT,\rho)$ when we wish to emphasize that the \ac{cl} tree $\hatcalT_{{\rm SCL}}$ is wrong (i.e., not equal to the true tree $\calT$). 
The global learning phase of Algorithm~\ref{algo:alathe} uses $\halpha n$ vector samples to learn the \ac{cl} tree, where $\halpha$ is a random variable. Thus, we need to define the error exponent of 
$t$-hop error with a random number of samples.
\begin{definition}[$t$-hop error exponent with a random number of samples]
For a homogeneous Ising tree with zero external field $\calT$ parameterized by $0<\rho<1$, the exponent corresponding to the $t$-hop error event with $\alpha$ fraction of samples $E_{t}(\hatcalT_{{\rm SCL}},\calT,\alpha n)$ is defined as
\begin{align}
	K_{t}(\calT,\rho,P_{\alpha})\triangleq\varliminf_{n\rightarrow\infty} -\frac{1}{n}\log\bbE_{\alpha}\big[\Pr\big(E_{t}(\hatcalT_{{\rm SCL}},\calT,\alpha n)\,\big|\,\alpha \big)\big],
\end{align}
where $P_{\alpha}$ is the probability distribution of random variable $\alpha$.
\end{definition}
We also use $K _{t}(\hatcalT_{{\rm SCL}},\calT,\rho,P_{\alpha})$ to denote $K _{t}(\calT,\rho,P_{\alpha})$ when we want to emphasize that the \ac{cl} tree $\hatcalT_{{\rm SCL}}$ is incorrect (i.e., does not equal to the true tree $\calT$). When we invoke the definition of $K_t(\hatcalT_{{\rm SCL}},\calT,\rho,P_{\alpha})$ (e.g., in Lemma~\ref{lem:thoperror}), $\alpha$ will be set to $\halpha$, the parameter that is produced by the global learning phase (Line 10 of Algorithm~\ref{algo:alathe}).
\begin{remark}
	On the event $\{\alpha=\alpha_{0}\}$ (i.e., $P_\alpha$ is the Dirac measure at $\alpha_0$) where $\alpha_{0}\in(0,1]$, we have
	\begin{align}
		K _{t}(\calT,\rho,P_{\alpha})=\alpha_{0} \varliminf_{n\rightarrow\infty} -\frac{1}{n}\log\Pr\big(E_{t}(\hatcalT_{{\rm SCL}},\calT,n)\big)=\alpha_{0} K _{t}(\hatcalT_{{\rm SCL}},\calT,\rho).
	\end{align}
\end{remark}
To describe the distance between edges, we adopt the definition of the edge geodesic pre-distance in \cite{neykov2019property} 
to measure the distance between edges.
\begin{definition}[Edge geodesic pre-distance]
	Let $\calG$ be a graph and $\{e,e^{\prime}\}$ be two pairs of nodes that do not necessarily  belong to $\calG$. The {\em edge geodesic pre-distance} is given by
	\begin{align}
		d_{\calG}(e,e^{\prime})\triangleq \min_{u\in e,v\in e^{\prime}} |{\rm Path}_{\calG}(u,v)|.
	\end{align}
	If there is no path between  every $u\in e$ and $v\in e^{\prime}$, $d_{\calG}(e,e^{\prime})=\infty$.
\end{definition}
Next, we define the corresponding packing and the packing number in graphs.
\begin{definition}[$r$-packing and $r$-packing number]
	Given a graph $\calG=(\calV,\calE)$ and a collection of edges $\calC$ with vertices in $\calV$, an $r$-packing of $\calC$ is any subset of edges $S\subseteq\calC$ such that 
	for all $e$, $e^{\prime}\in S$ satisfy $d_{\calG}(e,e^{\prime})\geq r$. The $r$-packing number is defined as
	\begin{align}
		N(\calC,\calG,r)\triangleq\max \{|S| \ | \ S\subseteq\calC , S \text{ is an }r\text{-packing}\}.
	\end{align}
\end{definition}
In other words,  $N(\calC,\calG,r)$ is the maximum cardinality over all  $r$-packings of $\calC$ in $\calG$.
\subsection{Proof of Theorem \ref{thm:mainthem}}
Fig.~\ref{fig:prooffig} shows the outline of the proof of Theorem \ref{thm:mainthem}. Since we consider the asymptotic regime in which  $n\rightarrow\infty$, we  simplify the  subsequent notations and, in particular, denote  the number of samples used in the global learning phase as  $\halpha n$ (instead of the more verbose $\lfloor\halpha n\rfloor$). 

\def \hos{3}
\def \ves{1.5}
\def \vess{1.7}
\def \scal{0.9}
\definecolor{lightora}{RGB}{250,217,120}
\definecolor{lightprop}{RGB}{147,170,216}
\definecolor{lightblue}{RGB}{224,235,246}
\definecolor{lightgreen}{RGB}{213,222,203}
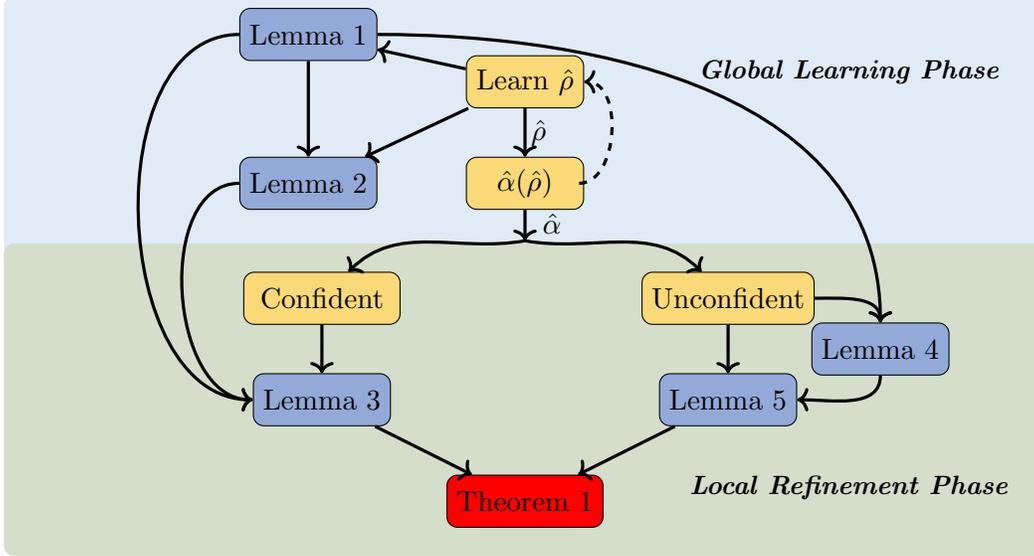
\begin{figure}[t]
	\centering
	\begin{tikzpicture}
		\tikzstyle{every node}=[font=\normalsize]
		\node [draw=none,rounded corners,fill=lightblue,minimum width=\scal*40em,minimum height=\scal*10.5em](m)at(0,-\scal*0.75){ };
		\node [draw=none,rounded corners,fill=lightgreen,minimum width=\scal*40em,minimum height=\scal*12em](n)at(0,-\scal*4.7){ };
		\node [draw,rounded corners,fill=lightora,minimum width=\scal*4.5em,minimum height=\scal*2em](a)at(0,0){Learn $\hrho$};
		\node [draw,rounded corners,fill=lightora,minimum width=\scal*4.5em,minimum height=\scal*2em](b)at(0,-\scal*\ves){$\halpha(\hrho)$};
		\node [draw,rounded corners,fill=lightora,minimum width=\scal*6em,minimum height=\scal*2em](c)at(-\scal*\hos,-\scal*\ves-\scal*\vess){Confident};
		\node [draw,rounded corners,fill=lightora,minimum width=\scal*6em,minimum height=\scal*2em](d)at(\scal*\hos,-\scal*\ves-\scal*\vess){Unconfident};
		\node [draw,rounded corners,fill=lightprop,minimum width=\scal*5em,minimum height=\scal*2em](e)at(-\scal*\hos,-2*\scal*\ves-\scal*\vess){Lemma \ref{lem:confide}};
		\node [draw,rounded corners,fill=lightprop,minimum width=\scal*5em,minimum height=\scal*2em](f)at(\scal*\hos,-2*\scal*\ves-\scal*\vess){Lemma \ref{lem:unconfide}};
		\node [draw,rounded corners,fill=red,minimum width=\scal*5em,minimum height=\scal*2em](g)at(0,-3*\scal*\ves-\scal*\vess){Theorem \ref{thm:mainthem}};
		\node [draw,rounded corners,fill=lightprop,minimum width=\scal*5em,minimum height=\scal*2em](h)at(-\scal*\hos-\scal*0.2,\scal*0.7){Lemma \ref{lem:thoperror}};
		\node [draw,rounded corners,fill=lightprop,minimum width=\scal*5em,minimum height=\scal*2em](i)at(-\scal*\hos-\scal*0.2,-\scal*\ves){Lemma \ref{lem:rhobound}};
		\node [draw,rounded corners,fill=lightprop,minimum width=\scal*5em,minimum height=\scal*2em](j)at(1.75*\scal*\hos,-1.5*\scal*\ves-\scal*\vess){Lemma \ref{lem:packnum}};
		\node (k)at(\scal*0.2,-\scal*0.5*\ves){$\hrho$};
		\node (l)at(\scal*0.4,-\scal*\ves-\scal*0.35*\vess){$\halpha$};
		\node [draw=none,rounded corners,fill=none,font=\fontsize{10}{6}\selectfont, minimum width=\scal*5em,minimum height=\scal*2em](o)at(1.6*\scal*\hos,0.1*\scal*\ves){\textbf{\emph{Global Learning Phase}}};
		\node [draw=none,rounded corners,fill=none,font=\fontsize{10}{6}\selectfont, minimum width=\scal*5em,minimum height=\scal*2em](o)at(1.6*\scal*\hos,-4*\scal*\ves){\textbf{\emph{Local Refinement Phase}}};

		\draw[very thick] [->] (a) to [out=270,in=90] (b);
		\draw[very thick,dashed] [->] (\scal*0.8,-\scal*\ves) to [out=0,in=0] (a);
		\draw[very thick] [->] (b) to [out=270,in=90] (0,-\scal*\ves-0.5*\scal*\vess);
		\draw[very thick] [->] (0,-\scal*\ves-0.5*\scal*\vess) to [out=190,in=45](c);
		\draw[very thick] [->] (0,-\scal*\ves-0.5*\scal*\vess) to [out=350,in=135](d);
		\draw[very thick] [->] (c) to (e);
		\draw[very thick] [->] (d) to (f);
		\draw[very thick] [->] (e) to (g);
		\draw[very thick] [->] (f) to (g);
		\draw[very thick] [->] (a) to (h);
		\draw[very thick] [->] (a) to (i);
		\draw[very thick] [->] (h) to (i);
		\draw[very thick] [->] (h) to [out=180,in=180] (e);
		\draw[very thick] [->] (i) to [out=180,in=180] (e);
		\draw[very thick] [->] (d) to [out=0,in=90] (j);
		\draw[very thick] [->] (j) to [out=270,in=0](f);
		\draw[very thick] [->] (h) to [out=0,in=90](j);

	\end{tikzpicture}
	\caption{A flowchart of the outline of the proof of Theorem \ref{thm:mainthem}.} 
	\label{fig:prooffig} 
\end{figure}
The error event $\{\hatcalT_{\Active}(n)\neq \calT\}$ can be decomposed as the union of two events as follows,
\begin{align}
	\{\hatcalT_{\Active}(n)\neq \calT\}=E_{{\rm Conf}}(n)\cup E_{{\rm Unconf}}(n),
\end{align}
where
\begin{align}
	E_{{\rm Conf}}(n)&\triangleq\Big\{\big(\calE\big(\hatcalT_{\rm{SCL}}(\halpha n)\big)\backslash \calE(\calT)\big)\cap \scC_{\mathrm{E}}(\halpha n)\neq \emptyset\Big\},\quad\text{and}\\*
	E_{{\rm Unconf}}(n)&\triangleq\Big\{\hatcalT_{\Active}\big(\scC_{\mathrm{V}}(\halpha n)^{\complement}\big)\neq \calT\big(\scC_{\mathrm{V}}(\halpha n)^{\complement}\big)\Big\}.
\end{align}
For ease of notation, when the number of vector samples $n$ is clear from the context, we will omit the dependencies of   $E_{{\rm Conf}}$ and $E_{{\rm Unconf}}$ on $n$. The event $E_{{\rm Conf}}$ signifies that there is an edge $e\in \hatcalT_{\rm{SCL}}(\halpha n)\big)\backslash \calE(\calT)$ that is identified as confident; this edge then will not be 
rectified in the local refinement phase, since it is not in any connected component of $\scC_{\mathrm{V}}(\halpha n)^{\complement}$ in $\hatcalT_{{\rm SCL}}(\halpha n)$. The event $E_{{\rm Unconf}}$ signifies that the edges in the connected components of $\scC_{\mathrm{V}}(\halpha n)^{\complement}$ 
in $\hatcalT_{{\rm SCL}}(\halpha n)$ are not correctly learned in the local refinement phase. Then we have that
\begin{align}
	K (\hatcalT_{\Active}\neq \calT)\geq \min \{K (E_{{\rm Conf}}),K (E_{{\rm Unconf}})\}.
\end{align}
To prove the desired inequality, we need to lower bound the terms on the right-hand side separately. 

We first state Lemma \ref{lem:thoperror}, which serves as a basis for all the following lemmas.
\begin{restatable}{lemma}{lemthoperror}
\label{lem:thoperror}
For a homogeneous Ising tree with zero external field parameterized by $0<\rho<1$ that satisfies Assumption \ref{assump:smalldeg}, $\hatcalT_{\rm{SCL}}(\halpha n)$ learned in Algorithm~\ref{algo:alathe} satisfies
\begin{align}
	K _{t}(\hatcalT_{{\rm SCL}},\calT,\rho,P_{\halpha})\geq c_{\rho}K _{\Passive}(\rho) \text{ for all }t\geq 3,
\end{align}
\end{restatable}
In other words, the exponent of the event that $\hatcalT_{\rm{SCL}}(\halpha n) \text{ has a }t\text{-hop error}$ is at least as large as $c_{\rho}K _{\Passive}(\rho)$ for all $t\geq 3$. Lemma \ref{lem:thoperror} thus implies 
that we only need to prove the following lemmas assuming that 
there are  $2$-hop errors (and not $t$-hop ones for $t\ge3$). 
Using this fact, we can then derive a high-probability bound of the estimate $\hrho(\halpha n)$.

\begin{restatable}{lemma}{lemrhobound}
\label{lem:rhobound}
For a homogeneous Ising tree with zero external field parameterized by $0<\rho<1$ that satisfies Assumption \ref{assump:smalldeg}, the estimate $\hrho(\halpha n)$ in Algorithm~\ref{algo:alathe} satisfies 
\begin{align}
	\varliminf_{n\rightarrow\infty}-\frac{1}{n}\log\Pr\Big(\rho+\frac{2(1-\rho)}{9}<\hrho(\halpha n)\Big)&\geq c_{\rho}K _{\Passive}(\rho)\qquad\mbox{and}\\ \varliminf_{n\rightarrow\infty}-\frac{1}{n}\log\Pr\Big(\rho-\frac{1-\rho}{6}>\hrho(\halpha n)\Big)&\geq c_{\rho}K _{\Passive}(\rho).
\end{align}
\end{restatable}
The purpose of Lemma \ref{lem:rhobound} is to provide   bounds on $\hrho(\halpha n)$ to facilitate the  lower  and upper bounding the probability of the event that some edge is identified as confident.
\begin{restatable}{lemma}{lemconfide}
\label{lem:confide}
For a homogeneous Ising tree with zero external field parameterized by $0<\rho<1$ that satisfies Assumption \ref{assump:smalldeg}, $\hatcalT_{\rm{SCL}}(\halpha n)$ learned in  \ac{alathe} in Algorithm~\ref{algo:alathe} satisfies
\begin{align} 
	\varliminf_{n\rightarrow\infty}-\frac{1}{n}\log\Pr\Big(\big(\calE\big(\hatcalT_{\rm{SCL}}(\halpha n)\big)\backslash \calE(\calT)\big)\cap \scC_{\mathrm{E}}(\halpha n)\neq \emptyset\Big)\geq c_{\rho}K _{\Passive}(\rho),\label{ieq:confide}
\end{align}
where $\scC_{\mathrm{E}}(\halpha n)$ is the set of confident edges (see Definition \ref{def:confide}). 
\end{restatable}
The proof of Lemma \ref{lem:confide} decomposes $E_{{\rm Conf}}$ into two disjoint events
\begin{align}
	E_{{\rm Conf}}(n)&=\big(E_{{\rm Conf}}(n)\cap E_{\rho}(n)\big)\cup\big(E_{{\rm Conf}}(n)\cap E_{\rho}(n)^{\complement}\big),
\end{align}
where 
\begin{align}
	E_{\rho}(n)&\triangleq\Big\{\rho-\frac{1-\rho}{6}\leq\hrho(\halpha n)\leq\rho+\frac{2(1-\rho)}{9}\Big\},
\end{align}
and it applies  Lemma \ref{lem:rhobound} as follows
\begin{align}
    K (E_{{\rm Conf}})\geq\min \big\{K (E_{{\rm Conf}}\cap E_{\rho}),K (E_{{\rm Conf}}\cap E_{\rho}^{\complement})\big\}\geq \min \big\{K (E_{{\rm Conf}}\cap E_{\rho}),K (E_{\rho}^{\complement})\big\}.
\end{align}
Then Lemma \ref{lem:confide} shows that $K (E_{{\rm Conf}}\cap E_{\rho})\geq c_{\rho}K _{\Passive}(\rho)$, and Lemma \ref{lem:rhobound} proves that $K (E_{\rho}^{\complement})\geq c_{\rho}K _{\Passive}(\rho)$. Thus, we have $K (E_{{\rm Conf}})\geq c_{\rho}K _{\Passive}(\rho)$.

To state the next lemma, we recall the notion of an induced subgraph. For a graph $\calG=(\calV,\calE)$ and a set of nodes $\calV^{\prime}\subseteq \calV$, the  {\em subgraph induced} by $\calV^{\prime}$ on $\calG$ is   denoted as $\calG(\calV^{\prime}) \triangleq (\calV^{\prime},\calE^{\prime})$, where 
$\calE^{\prime}\triangleq\{\{i,j\}\ | \ \{i,j\}\in \calE \text{ and }i,j\in\calV^{\prime}\}$. 

\begin{restatable}{lemma}{lempacknum}
\label{lem:packnum}
For a homogeneous Ising tree with zero external field  parameterized by the correlation coefficient $0<\rho<1$ and satisfies Assumption \ref{assump:smalldeg}, the packing number $N\big(\scC_{\mathrm{E}}(\halpha n)^{\complement},\calT\big(\scC_{\mathrm{V}}(\halpha n)^{\complement}\big),2\big)$ of the unconfident edges set $\scC_{\mathrm{E}}(\halpha n)^{\complement}$ 
learned in Algorithm~\ref{algo:alathe} satisfies
\begin{align}
	K\Big(\big\{N(\scC_{\mathrm{E}}(\halpha n)^{\complement},\calT(\scC_{\mathrm{V}}(\halpha n)^{\complement}),2)\geq 13\big\}_{n\in\bbN}\Big)>c_{\rho}K _{\Passive}(\rho).
\end{align}
\end{restatable}
The purpose of Lemma~\ref{lem:packnum} is to provide a lower bound for the number of scalar samples of the set of unconfident nodes. The number $13$ is chosen to ensure that the number of unconfident nodes $\tilp(\halpha n)<26d$ with high probability, and the number of scalar samples of each unconfident node $(1-\halpha)\frac{p}{\tilp(\halpha n)}$ is then guaranteed to be sufficiently large. Together with Lemmas~\ref{lem:tilpbound} and~\ref{lem:numsample}, Lemma~\ref{lem:packnum} shows that we can safely assume   that the number of scalar samples of each unconfident node is lower bounded by $[\halpha+3(1-\halpha)]n=(3-2\halpha)n$. Thus, we have a sufficient number of scalar samples to boost the correctness of the unconfident parts of the graph in the local refinement stage of \ac{alathe}. 

\begin{restatable}{lemma}{lemunconfide}
\label{lem:unconfide}
For a homogeneous Ising tree with zero external field parameterized by $0<\rho<1$ that satisfies Assumption \ref{assump:smalldeg}, the error exponent of the event that the subgraph of unconfident nodes in $\hatcalT_{\Active}(n)$ in Algorithm~\ref{algo:alathe} is wrong is lower bounded as
\begin{align}
	K \Big(\big\{\hatcalT_{\Active}\big(\scC_{\mathrm{V}}(\halpha n)^{\complement}\big)\neq \calT\big(\scC_{\mathrm{V}}(\halpha n)^{\complement}\big)\big\}_{n\in\bbN}\Big)\geq c_{\rho}K _{\Passive}(\rho).
\end{align}
\end{restatable}
The proof of Lemma \ref{lem:unconfide} decomposes $E_{{\rm Unconf}}$ as
\begin{align}
	E_{{\rm Unconf}}(n)&=\big(E_{{\rm Unconf}}\cap E_{{\rm SS}}(n)\big)\cup \big(E_{{\rm Unconf}}\cap E_{{\rm SS}}(n)^{\complement}\big)\text{,}\\
	\text{where }E_{{\rm SS}}(n)&\triangleq\Big\{N(\scC_{\mathrm{E}}(\halpha n)^{\complement},\calT\big(\scC_{\mathrm{V}}(\halpha n)^{\complement}),2\big)\leq 12\Big\}.
\end{align}
We note that the event $E_{{\rm SS}}$, whose subscript stands for \underline{S}ufficient \underline{S}amples, implies that there are at least $[3-2\halpha]n$ samples for each unconfident node. Then we have
\begin{align}
	K (E_{{\rm Unconf}})&\geq \min \big\{K (E_{{\rm Unconf}}\cap E_{{\rm SS}}),K (E_{{\rm Unconf}}\cap E_{{\rm SS}}^{\complement})\big\}\\
	&\geq \min\big\{K (E_{{\rm Unconf}}\cap E_{{\rm SS}}),K (E_{{\rm SS}}^{\complement})\big\}.
\end{align}
Lemma \ref{lem:unconfide} shows that $K (E_{{\rm Unconf}}\cap E_{{\rm SS}})\geq c_{\rho}K _{\Passive}(\rho)$, and Lemma \ref{lem:packnum} shows that $K (E_{{\rm SS}}^{\complement})\geq c_{\rho}K _{\Passive}(\rho)$. Thus, we 
have $K (E_{{\rm Unconf}})\geq c_{\rho}K_{\Passive}(\rho)$, as desired.

\section{Concluding Remarks and Future Work} \label{sec:conclusion}
In this paper, we have proposed an algorithm \ac{alathe} that improves the error exponent---and hence {\em significantly} reduces the error probability---of learning the structures of the simplest class of graphical models, namely, homogeneous Ising tree models. By exploiting the fact that the number of nodes $p$ is sufficiently large relative to the maximum degree of the graph (Assumption~\ref{assump:smalldeg}), we have been able to  exploit the statistical variability of the samples across the various parts of the graph to allocate a fraction  of the sample budget to those parts  that we are more unconfident of  (compared to other parts). This boosts the overall probability of success of learning the graph and we have also  corroborated this via  numerical experiments. 

The principle that we have outlined using our flood-and-breakage-repair example in Section~\ref{sec:work} and that is used to analyze \ac{alathe} seems to be  more generally applicable to other statistical learning problems (in which each constituent component ``has the same quality'') beyond graphical model selection. In future, we hope to exploit similar ideas to boost the performances of other statistical tasks such as high-dimensional property testing or ranking of a large number of objects when active learning is permitted. In addition, establishing an upper bound (impossibility result) on the error exponent for active learning of homogeneous trees as a function of $\rho$ is also a fruitful but challenging endeavor. We leave this as an open problem. 

\section{Appendices}\label{sec:app}

\subsection{General active learning algorithm in Section~\ref{sec:prelim}}\label{sec:appgenalgo}
The general algorithm for active structure learning is shown in Algorithm~\ref{alg:general}.
\begin{algorithm}[t]
\label{alg:general}
	\caption{General Active Structure Learning Algorithm}
	\textbf{Input:}  The number of nodes $p$, the number of vector samples $n$\\
	\textbf{Output:} Estimated tree structure $\hatcalT(n)$ \\
	\textbf{Procedure:}
	\begin{algorithmic}[1]\label{algo:generalal}
		\STATE Initialize $i=0$, $T_{0}=0$.
		\WHILE{$T_{i}<np$}
			\STATE $i=i+1$.\\
			\STATE Choose $S_{i}\subseteq [p]$ s.t. $|S_{i}|\geq 1$.\\
			\STATE Obtain $n_{i}$ independent samples $\bX_{S_{i}}$ having distribution $P_{S_{i}}$\\
			\STATE Learn partial statistics of the tree model using $\{\bX_{S_{1}},\bX_{S_{2}},\cdots,\bX_{S_{i}}\}$.\\
			\STATE Update $T_{i}=\sum_{j=1}^{i}n_{j}|S_{j}|$.
		\ENDWHILE
		\STATE Return the tree structure $\hatcalT(n)$ learned from $\{\bX_{S_{1}},\bX_{S_{2}},\cdots,\bX_{S_{i}}\}$.
	\end{algorithmic}
\end{algorithm}

\subsection{Proofs of results in Section~\ref{sec:pfview}}
In this section, we provide the proofs of the lemmas stated in the proof overview of Theorem \ref{thm:mainthem} in Section~\ref{sec:pfview}. We start with some additional definitions and results, then prove Lemmas~\ref{lem:thoperror} to~\ref{lem:unconfide}. 
\subsubsection{Additional Definition and Results}
  Define the set of possible values of $\halpha$ as  $\Psi\triangleq\{0.8,0.85,0.9,0.95,0.985,0.995,1\}$ (see Table~\ref{table:alpha}). To describe the relationship between two error events, we define the notion of independent $t$-hop errors. See Fig.~\ref{fig:indeperr} for an illustration of this definition.
\begin{definition}[Independent $t$-hop errors]
	Given a graph $\calG$, the two error events $E_{t_{1}}(\calP_{t_{1}},m)$ and $E_{t}(\calP_{t_{2}},m)$ (defined in Definition~\ref{def:thop}) are said to be {\em independent}  if $\calE (\calP_{t_{1}} )\cap\calE (\calP_{t_{2}} )= \emptyset$. 
\end{definition}
As shown in Fig.~\ref{fig:indeperr}, the node sets of the paths where two error events are independent can have some common nodes. We denote the event that there is an independent $t_{1}$-hop error and $t_{2}$-hop error in the graph $\calT$ as $\mathrm{IE}_{t_{1},t_{2}}(\calT,m)$. To emphasize that $\hatcalT_{{\rm SCL}}$ is not equal to the true tree $\calT$, we also write $\mathrm{IE}_{t_{1},t_{2}}(\calT,m)$ as $\mathrm{IE}_{t_{1},t_{2}}(\hatcalT_{{\rm SCL}},\calT,m)$.
The Ising tree models with zero-external field have the property that the products of terminal nodes of edges are independent\cite{nikolakakis2021predictive}.

\begin{figure}[t]
	\centering
	\includegraphics[width=0.665\textwidth]{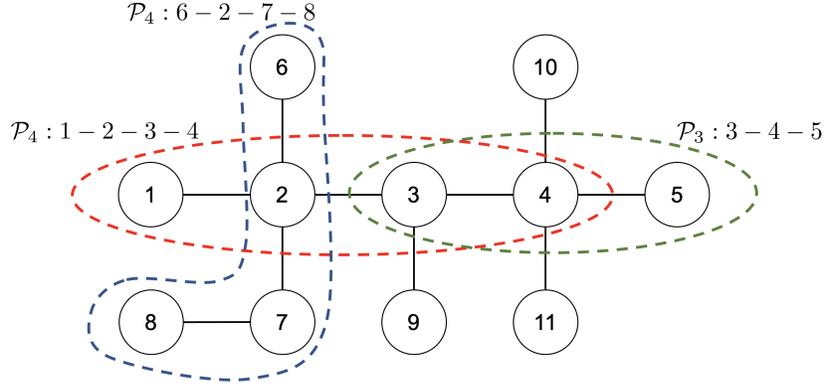}
	\caption{\small The 4-hop errors involving path $6-2-7-8$ and $1-2-3-4$ are independent, while the 4-hop error involving path $1-2-3-4$ and the 3-hop error involving $3-4-5$ are not independent.}
	\label{fig:indeperr}
\end{figure}
\begin{lemma}[Nikolakakis, Kalogerias, and Sarwate  \cite{nikolakakis2021predictive}]\label{lem:indeppair}
	If $\bX=(X_1, X_2, \ldots, X_p) \in \{+1,-1\}^{p}$ is a  binary random vector    drawn from a forest-structured distribution with uniform marginals for each entry $X_{i}$ in $\bX$, then the elements  $X_iX_j$
	of the collection of $|\calE|$ random variables $\{X_{i}X_{j}~|~\{i,j\}\in\calE\}$ are independent.
\end{lemma}
The independence of the products of terminal nodes of edges and the definition of $t$-hop errors imply that two so-called ``independent'' $t$-hop error events are indeed statistically independent. 

\begin{lemma}\label{lem:mono}
	In the homogeneous Ising tree model with zero external field parameterized by $0<\rho<1$, the error exponent for a $t$-hop error, $K _{t}(\calT,\rho)$, is strictly increasing in $t\geq 2$.
\end{lemma}
\begin{proof}[Proof of Lemma \ref{lem:mono}]
	Since the number of length-$t$ paths in $\calT$ is finite, if $\calT$ has at least one $\calP_{t}$ as its subgraph, Eqn.~\eqref{eq:thopee} can be expressed as
	\begin{align}
		K _{t}(\calT,\rho)=\varliminf_{n\rightarrow \infty}-\frac{1}{n}\log\Pr (E_{t}(\calP_{t},n)),
	\end{align}
	where $\calP_{t}$ is a path with length $t$ that is also parameterized by the correlation of adjacent nodes $\rho$. Thus, we only need to consider the path graph. Lemma 1 in \cite{tandon2020exact} shows that for a three-node path $1-2-3$ which is parameterized by $P_{1}(0)=\frac{1}{2}$, $P_{2|1}(0|1)=P_{2|1}(1|0)=\theta$ 
	and $P_{3|2}(0|1)=P_{3|2}(1|0)\triangleq\ttheta_{1}$, we have
	\begin{align}
		K_{2}(\calP_{2},1-2\theta)=\varliminf_{n\rightarrow \infty}-\frac{1}{n}\log\Pr\big(\hrho_{13}(n)\geq\hrho_{12}(n)\big) 
		=-\log\Big(1-\ttheta_{1}\big(1-\sqrt{4\theta(1-\theta)}\big)\Big),
	\end{align}
	where $\theta=\frac{1-\rho}{2}$ in the homogeneous Ising trees. To calculate the error exponent of $t$-hop error in the path with $\theta\in (0,0.5)$, we note that
	\begin{align}
		K _{t}(\calP_{t},1-2\theta)&=\varliminf_{n\rightarrow \infty}-\frac{1}{n}\log\Pr\bigg(\bigcup_{j=1}^{t}\{\hrho_{k_{1}k_{t+1}}\geq \hrho_{k_{j}k_{j+1}}\}\bigg)\\*
		&=\varliminf_{n\rightarrow \infty}-\frac{1}{n}\log\Pr\big(\hrho_{k_{1}k_{t+1}}\geq \hrho_{k_{1}k_{2}}\big),
	\end{align}
	which results from the symmetry of the underlying homogeneous tree model.

	Taking $\ttheta_{t-1}\triangleq P_{t+1|2}(0|1)=P_{t+1|2}(1|0)$, we have that
	\begin{align}
		\varliminf_{n\rightarrow \infty}-\frac{1}{n}\log\Pr\big(\hrho_{k_{1}k_{t+1}}\geq \hrho_{k_{1}k_{2}}\big)=-\log\Big(1-\ttheta_{t-1}\big(1-\sqrt{4\theta(1-\theta)}\big)\Big).
	\end{align}
	To prove that $\varliminf_{n\rightarrow \infty}-\frac{1}{n}\log \Pr\big(\hrho_{k_{1}k_{t+1}}\geq \hrho_{k_{1}k_{2}}\big)$ strictly increases with   $t$, we only need to show that 
	$\ttheta_{t-1}$ is strictly increasing in $t$. In fact, we prove that $\ttheta_{t}\in(0,\frac{1}{2})$ and $\ttheta_{t}>\ttheta_{t-1}$ for $t\geq 2$ by induction. When 
	$t=2$, $\ttheta_{2}=2\theta(1-\theta)\in (0,\frac{1}{2})$ and $\ttheta_{2}=2\theta(1-\theta)>\theta=\ttheta_{1}$. If $\ttheta_{t}>\ttheta_{t-1}$ and $\ttheta_{t}\in(0,\frac{1}{2})$ holds 
	for $t=2,\ldots,k-1$, then for $t=k$ we have
	\begin{align}
		\ttheta_{k}&=\ttheta_{k-1}(1-\theta)+(1-\ttheta_{k-1})\theta=(1-2\theta)\ttheta_{k-1}+\theta<(1-2\theta)\frac{1}{2}+\theta=\frac{1}{2}\\*
		\ttheta_{k}-\ttheta_{k-1}&=\theta(1-2\ttheta_{k-1})>0,
	\end{align}
	as desired.
\end{proof}

\subsubsection{Proof of Lemma \ref{lem:thoperror}}
We start with a basic lemma. 
\begin{lemma}\label{lem:twoseq}
    For two sequences of events $\calA=\big\{\calA(n)\big\}_{n\in\bbN}$ and $\calB=\big\{\calB(n)\big\}_{n\in\bbN}$, the exponent of $\calA$ satisfies
    $
        K ( \calA)\geq \min\{K ( \calA\cap \calB),K (\calB^{\complement})\}.
    $
\end{lemma}
\begin{proof}[Proof of Lemma \ref{lem:twoseq}]
    We have
    $
	    \Pr\big(\calA(n)\big)\le \Pr\big( \calA(n)\cap \calB(n)\big)+\Pr\big(\calB(n)^{\complement}\big)
	$  for all $n\in\bbN$. The result then follows by considering  the normalized logarithm of the probabilities and taking limits. 
\end{proof}
\begin{proof}[Proof of Lemma \ref{lem:thoperror}]
	We adopt a computer-aided approach to prove the lemma. Lemma \ref{lem:mono} shows that the error exponent is strictly increasing in $t$, so we only need to prove the lemma for $t=3$. 
	
	Let $\calA(n)\triangleq E_{3}(\hatcalT_{{\rm SCL}},\calT,\halpha n)$ be the event that $\hatcalT_{\rm{SCL}}(\halpha n)\text{ has a }3\text{-hop error}$. In addition,   let $\calB(n)\triangleq\{\halpha \in \calM\}$ be the event that 
	$\halpha$ takes the value in a set $\calM$, where $\calM\subseteq \Psi$ is an appropriately chosen subset of values of $\halpha$. Lemma \ref{lem:twoseq} implies that we only need to separately lower bound $K(\calA\cap \calB)$ and $K (\calB^{\complement})$. We have
	\begin{align}
		K ( \calA\cap \calB)&=\varliminf_{n\rightarrow\infty}-\frac{1}{n}\log \bigg(\sum_{\beta\in \calM}\Pr\big(E_{3}(\hatcalT_{{\rm SCL}},\calT,\halpha n),\halpha=\beta\big)\bigg)\\
		&\geq \varliminf_{n\rightarrow\infty}-\frac{1}{n}\log \bigg(\sum_{\beta\in \calM}\Pr\big(E_{3}(\hatcalT_{{\rm SCL}},\calT,\beta n)\big)\bigg)\\
		&=\min_{\beta\in \calM}\varliminf_{n\rightarrow\infty}-\frac{1}{n}\log\Pr\big(E_{3}(\hatcalT_{{\rm SCL}},\calT,\beta n)\big)\\
		&=\min_{\beta\in \calM}\,\beta\cdot \varliminf_{n\rightarrow\infty}-\frac{1}{n}\log\Pr\big(E_{3}(\hatcalT_{{\rm SCL}},\calT,n)\big)\\
		&=\min_{\beta\in \calM}\,\beta \cdot K _{3}(\hatcalT_{{\rm SCL}},\calT,\rho)\\
		&=\min_{\beta\in \calM}\,\beta \bigg(-\log\Big(1-2\theta(1-\theta)\big(1-\sqrt{4\theta(1-\theta)}\big)\Big)\bigg), \label{ieq:lb1}
	\end{align}
	where the last equality follows from \cite[Lemma 1]{tandon2020exact}.

	In the following, we prove that the lower bound for $K ( \calA\cap \calB)$ in inequality \eqref{ieq:lb1} is at least $c_{\rho}K _{\Passive}(\rho)$ using a computer-aided approach.

	For $\rho\in[0.8,1)$, we set $\calM=\Psi$, then $K ( \calA\cap \calB)\geq 0.8K _{3}(\hatcalT_{{\rm SCL}},\calT,\rho)$. It clearly holds that  $\Pr(\halpha\in \calM^{\complement})=0$ from Table~\ref{table:alpha}; hence,  $K (\calB^{\complement})=\infty$. 
	\begin{figure}[t] 
		\centering
		\subfigure[Ratios $0.8K _{3}(\hatcalT_{{\rm SCL}},\calT,\rho)$, $0.8K _{4}(\hatcalT_{{\rm SCL}},\calT,\rho)$ and $K _{\Passive}(\rho)$.]{
		\begin{minipage}[t]{0.45\linewidth}
		\centering
		\includegraphics[width=3in]{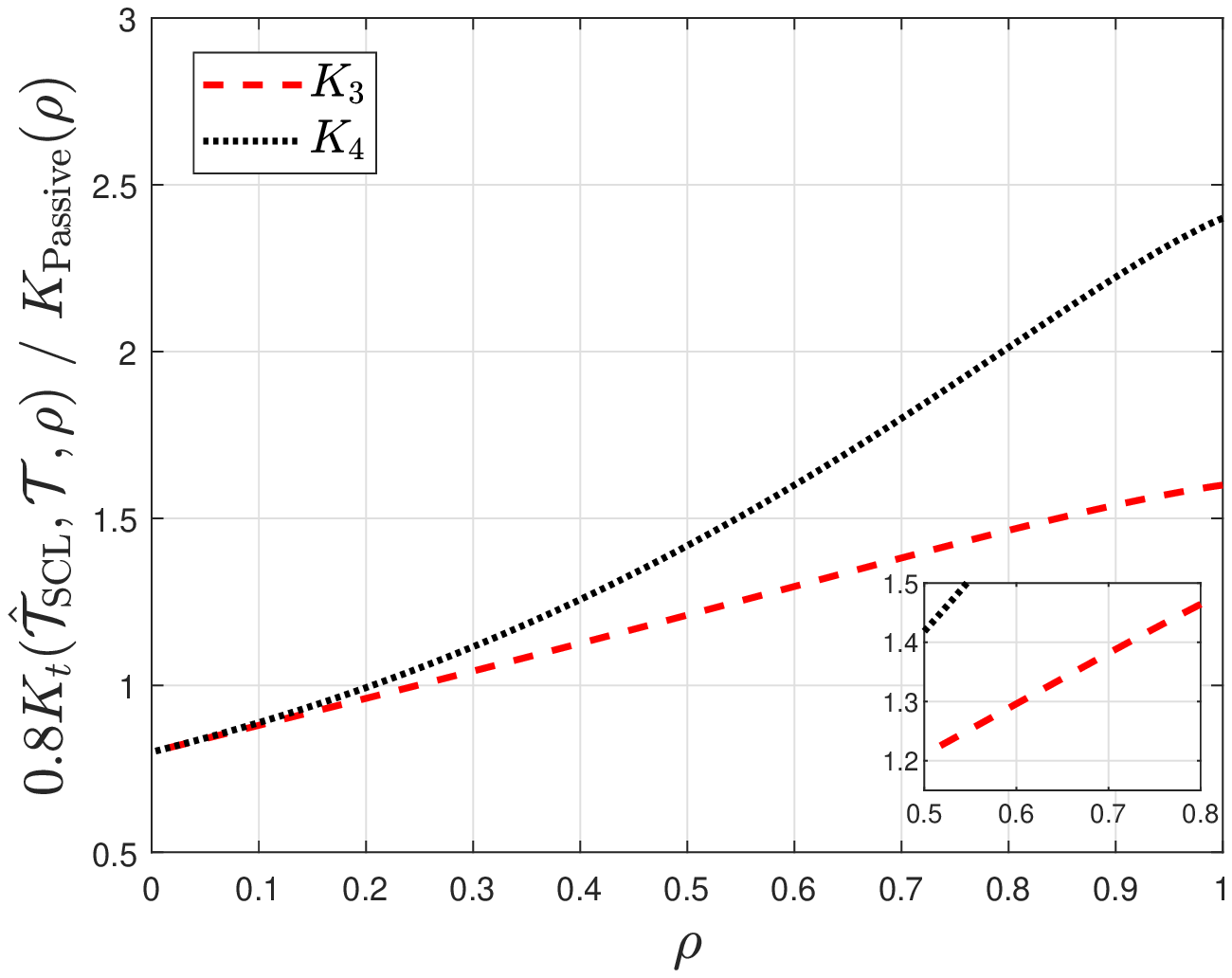}
		\end{minipage}%
		\label{fig:thoperr_8}
		}%
		\hspace{.15in}
		\subfigure[Ratios of $0.85K _{3}(\hatcalT_{{\rm SCL}},\calT,\rho)$, $0.85K _{4}(\hatcalT_{{\rm SCL}},\calT,\rho)$ and $K _{\Passive}(\rho)$.]{
		\begin{minipage}[t]{0.45\linewidth}
		\centering
		\includegraphics[width=3in]{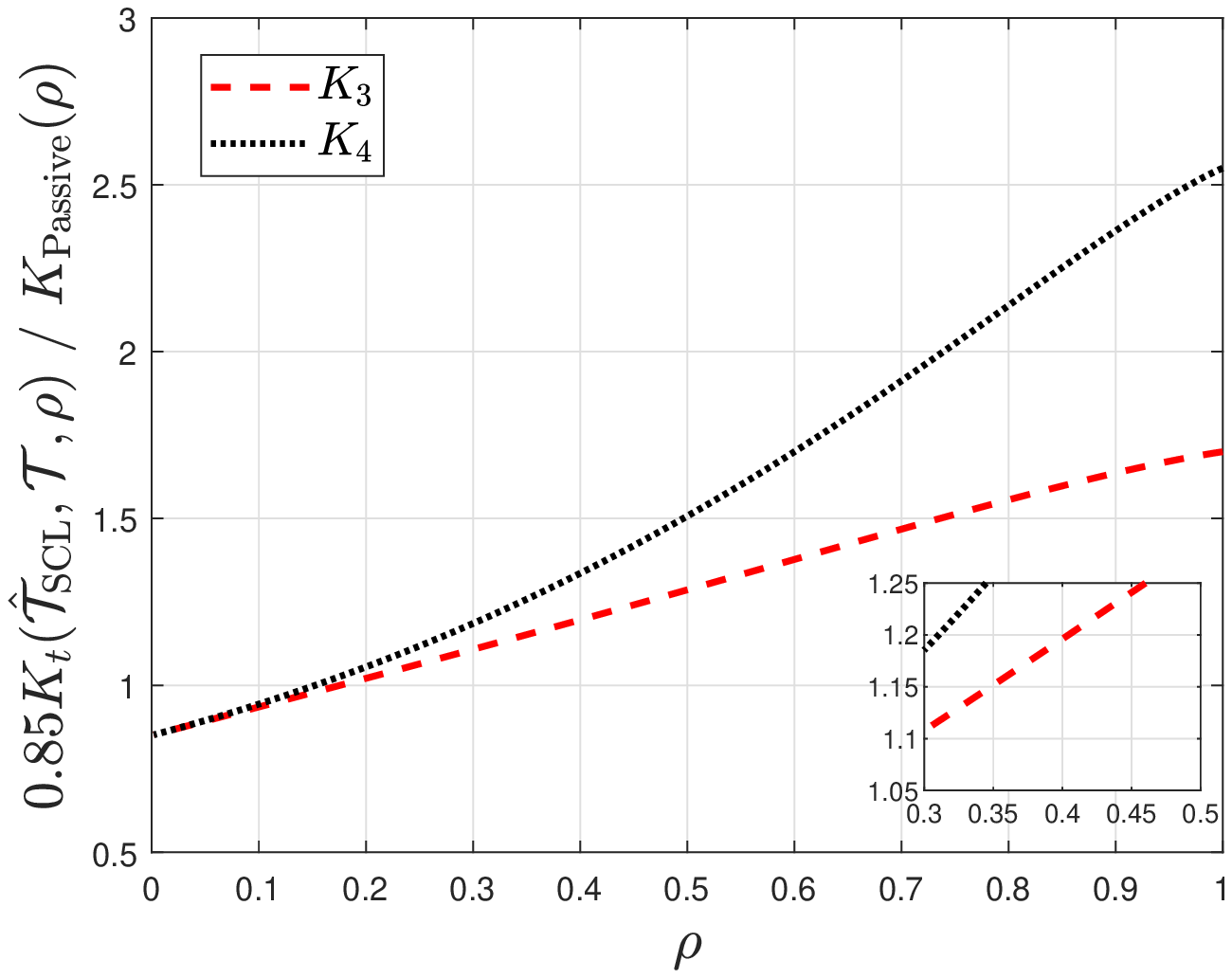}
		\end{minipage}%
		\label{fig:thoperr_85}
		}%
		\centering
		\caption{Ratios of the exponents of $t$-hop errors and that of passive structure learning error}
	\end{figure}

	Fig.~\ref{fig:thoperr_8} displays the ratio of $0.8K _{3}(\hatcalT_{{\rm SCL}},\calT,\rho)$ and $K _{\Passive}(\rho)$, and it shows that this ratio is uniformly larger than $1.4$ for $\rho\in[0.8,1)$.

	For $\rho\in[0.6,0.8)$, we also set $\calM=\Psi$, then $K ( \calA\cap \calB)\geq 0.8K _{3}(\hatcalT_{{\rm SCL}},\calT,\rho)$. It is obvious that $\Pr(\halpha\in \calM^{\complement})=0$ from Table~\ref{table:alpha}, 
	then $K (\calB^{\complement})=\infty$. Fig.~\ref{fig:thoperr_8} displays the ratio between $0.8K _{3}(\hatcalT_{{\rm SCL}},\calT,\rho)$ and $K _{\Passive}(\rho)$ is  uniformly larger than $1.29$ for $\rho\in[0.6,8)$.
	
	For $\rho\in[0.4,0.6)$, we decompose the interval as $[0.4,0.6)=[0.4,0.5)\cup [0.5,0.6)$. When $\rho\in[0.5,0.6)$, we set $\calM=\Psi$, and Fig.~\ref{fig:thoperr_8} 
	shows that the ratio of error exponents is larger than 1.19. When $\rho\in[0.4,0.5)$, we set $\calM=\Psi\backslash\{0.8\}$, then $K ( \calA\cap \calB)\geq 0.85K _{3}(\hatcalT_{{\rm SCL}},\calT,\rho)$. $\halpha\notin \calM$ implies that $\hrho \geq 0.76$. Then we have
	\begin{align}
		K (\halpha\notin \calM)&=K (\hrho(\halpha n)\geq 0.76)\\*
		&=\varliminf_{n\rightarrow \infty}-\frac{1}{n}\log\Pr\Bigg(\frac{\sum_{e\in\calE (\hatcalT_{{\rm SCL}}(\halpha n) )}\hrho_{e}(\halpha n)}{p-1}\geq 0.76\Bigg)\\
		&\geq\varliminf_{n\rightarrow \infty}-\frac{1}{n}\log\Pr\Bigg(\max_{e\in\calE (\hatcalT_{{\rm SCL}}(\halpha n) )}\hrho_{e}(\halpha n)\geq 0.76\Bigg)\\
		&\geq \varliminf_{n\rightarrow \infty}-\frac{1}{n}\log\Pr\big(\hrho_{e^{\prime}}(\halpha n)\geq 0.76\big) \label{ieq:ieq1}\\
		&\geq \varliminf_{n\rightarrow \infty}-\frac{1}{n}\log\Pr\big(\hrho_{e^{\prime}}(0.8n)\geq 0.76\big) \label{eq:eq1}\\
		&= 0.8\varliminf_{n\rightarrow \infty}-\frac{1}{n}\log\Pr\big(\hrho_{e^{\prime}}(n)\geq 0.76\big)\\
		&= 0.8D\Big(\frac{1-\rho}{2}\Big\|0.12\Big),
	\end{align}
	where inequality \eqref{ieq:ieq1} results from the fact that the correlation between non-adjacent nodes is strictly less that $\rho$ and $e^{\prime}$ is an edge such that $\rho_{e^{\prime}}=\rho$, and inequality \eqref{eq:eq1} comes from that the fact that the event $\hrho_{e^{\prime}}(\halpha n)\geq 0.76$ implies that $\halpha=0.8$. When $\rho\in[0.4,0.5)$, $0.8D(\frac{1-\rho}{2}\|0.12)>0.8D(0.25\|0.12)>1.23K _{\Passive}(\rho).$
	Fig.~\ref{fig:thoperr_85} shows that ratio of $0.85K _{3}(\hatcalT_{{\rm SCL}},\calT,\rho)$ and $K _{\Passive}(\rho)$ is larger than $1.19$ when $\rho\in[0.4,0.5)$.

	For $\rho$'s in other intervals, similar analyses can be performed.
\end{proof}

\subsubsection{Proof of Lemma \ref{lem:rhobound}}
\begin{proof}[Proof of Lemma \ref{lem:rhobound}]
	Following Lemma \ref{lem:twoseq}, by denoting the error exponents corresponding to the (sequences of) events $ \calC$ and $\calD$ as $K ( \calC)$ and $K (\calD)$, respectively, we have  the inequality $K ( \calC)\geq \min\big\{K ( \calC\cap \calD),K \big(\calD^{\complement}\big)\big\}$.
	
	We first prove  that
	\begin{align}
	    \varliminf_{n\rightarrow\infty}-\frac{1}{n}\log\Pr\Big(\rho+\frac{2(1-\rho)}{9}<\hrho(\halpha n)\Big)\geq c_{\rho}K _{\Passive}(\rho).
	\end{align}
	To prove this, we set $ \calC(n)\triangleq\big\{\rho-\frac{2(1-\rho)}{9}<\hrho(\halpha n)\big\}$ and $\calD(n)\triangleq\calD_{1}(n)\cap \calD_{2}(n)$ where
	\begin{align}
		\calD_{1}(n)&\triangleq\bigcap_{t_{1},t_{2}\geq 2} \big(\mathrm{IE}_{t_{1},t_{2}}(\hatcalT_{{\rm SCL}},\calT,\halpha n)\big)^{\complement}\qquad\mbox{and}\\
		\calD_{2}(n)&\triangleq\bigcap_{t_{3}\geq 3}\big(E_{t_{3}}(\hatcalT_{{\rm SCL}},\calT,\halpha n)\big)^{\complement}.
	\end{align}
	We remark that $\calD_{1}$ represents the event that $\hatcalT_{\rm{SCL}}(\halpha n)$  does not have a $t_{1}$-hop error that is independent of a $t_{2}$-hop error for $t_{1},t_{2}\geq 2$, and $\calD_{2}$ represents  the event that $\hatcalT_{\rm{SCL}}(\halpha n)$ does not have a $t_3$-hop error for any $t_{3}\geq 3$.

	In the following, we   prove $K ( \calC\cap \calD)\geq c_{\rho}K _{\Passive}(\rho)$ and $K (\calD^{\complement})\geq c_{\rho}K _{\Passive}(\rho)$.

	We begin by proving that $K (\calD^{\complement})\geq c_{\rho}K _{\Passive}(\rho)$. To prove this inequality, we only need to prove that $K (\calD_{i}^{\complement})\geq c_{\rho}K _{\Passive}(\rho)$ for $i=1,2$, because $\Pr(\calD^{\complement})=\Pr(\calD_{1}^{\complement}\cup \calD_{2}^{\complement})\leq \Pr(\calD_{1}^{\complement})+\Pr(\calD_{2}^{\complement}).$
	
	For $\calD_{1}^{\complement}$, Lemmas~\ref{lem:indeppair} and \ref{lem:mono} show that 
	\begin{align}
		K (\calD_{1}^{\complement})&\geq K _{t_{1}}(\hatcalT_{\rm{SCL}},\calT,\rho,0.8)+K _{t_{2}}(\hatcalT_{\rm{SCL}},\calT,\rho,0.8)\\
		&\geq 2K _{2}(\hatcalT_{\rm{SCL}},\calT,\rho,0.8)
		= 2\times 0.8\times K _{2}(\hatcalT_{\rm{SCL}},\calT,\rho)
		=1.6 K _{\Passive}(\rho).
	\end{align}
	Since $1.6\geq c_{\rho}$, we have $K (D_{1}^{\complement})\geq c_{\rho}K _{\Passive}(\rho)$. 	For $\calD_{2}^{\complement}$, Lemma \ref{lem:thoperror} proves the claim.

	We now turn our attention to the event $ \calC\cap \calD$. The event $\calD$ implies that $ \hatcalT_{\rm{SCL}}$ can have only dependent $2$-hop errors. The number
	of dependent edges in $ \hatcalT_{\rm{SCL}}$ is at most $3d-2$, and Lemma \ref{lem:indeppair} shows that the number of independent products $X_{i}X_{j}$, $\{i,j\}\in\calE(\hatcalT_{{\rm SCL}})$ 
	is at least $p-1-(3d-2)\geq 82d-1-(3d-2)=79d+1$. Consider,
	\begin{align}
		\Pr( \calC\cap \calD)&=\Pr\Bigg(\!\frac{\sum_{e\in\calE(\calG)\cap\calE( \hatcalT_{\rm{SCL}}(\halpha n))}\hrho_{e}(\halpha n)\!+\!\sum_{e^{\prime}\in\calE( \hatcalT_{\rm{SCL}}(\halpha n))\backslash\calE(\calG)}\!\hrho_{e^{\prime}}(\halpha n)}{p-1}\!>\!\rho\!+\!\frac{2(1\!-\!\rho)}{9},\calD\Bigg)\\
		&\leq\Pr\Bigg(\frac{\sum_{e\in\calE(\calG)\cap\calE( \hatcalT_{\rm{SCL}})}\hrho_{e}+\sum_{e^{\prime}\in\calE( \hatcalT_{\rm{SCL}})\backslash\calE(\calG)}1}{p-1}>\rho+\frac{2(1-\rho)}{9},\calD\Bigg)\\
		&\leq\Pr\Bigg(\frac{3d-2+\sum_{e\in\calE(\calG)\cap\calE( \hatcalT_{\rm{SCL}})}\hrho_{e}}{p-1}>\rho+\frac{2(1-\rho)}{9},\calD\Bigg)\label{ieq:ieq2}\\
		&\leq\Pr\Bigg(\frac{3d-2+\sum_{e\in\calE(\calG)\cap\calE( \hatcalT_{\rm{SCL}})}\hrho_{e}}{p-1}>\rho+\frac{2(1-\rho)}{9}\Bigg)\\
		&\leq\Pr\Bigg(\frac{3d+\sum_{e\in\calE(\calG)\cap\calE( \hatcalT_{\rm{SCL}})\backslash \{e_{1},e_{2}\}}\hrho_{e}}{p-1}>\rho+\frac{2(1-\rho)}{9}\Bigg)\label{ieq:ieq3}\\
		&=\Pr\Bigg(\frac{\sum_{e\in\calE(\calG)\cap\calE( \hatcalT_{\rm{SCL}})\backslash \{e_{1},e_{2}\}}\hrho_{e}}{p-3d-1}>\frac{(p-1)\big[\rho+\frac{2(1-\rho)}{9}\big]-3d}{p-3d-1}\Bigg)\\
		&=\Pr\Bigg(\frac{\sum_{e\in\calE(\calG)\cap\calE( \hatcalT_{\rm{SCL}})\backslash \{e_{1},e_{2}\}}\hrho_{e}}{p-3d-1}>\rho+\frac{2p-27d-2}{9(p-3d-1)}(1-\rho)\Bigg)\\
		&\leq\Pr\Bigg(\frac{\sum_{e\in\calE(\calG)\cap\calE( \hatcalT_{\rm{SCL}})\backslash \{e_{1},e_{2}\}}\hrho_{e}}{p-3d-1}>\rho+\frac{1-\rho}{6}\Bigg)\label{ieq:ieq5}\\
		&=\sum_{\beta\in\Psi}\Pr\Bigg(\frac{\sum_{e\in\calE(\calG)\cap\calE( \hatcalT_{\rm{SCL}}(\halpha n))\backslash \{e_{1},e_{2}\}}\hrho_{e}(\halpha n)}{p-3d-1}>\rho+\frac{1-\rho}{6},\halpha=\beta\Bigg)\label{eq:ome}\\
		&\leq 7\max_{\beta \in \Psi}\Pr\Bigg(\frac{\sum_{e\in\calE(\calG)\cap\calE( \hatcalT_{\rm{SCL}}(\halpha n))\backslash \{e_{1},e_{2}\}}\hrho_{e}(\halpha n)}{p-3d-1}>\rho+\frac{1-\rho}{6},\halpha=\beta\Bigg)\\
		&\leq 7\Pr\Bigg(\frac{\sum_{e\in\calE(\calG)\cap\calE( \hatcalT_{\rm{SCL}}(0.8 n))\backslash \{e_{1},e_{2}\}}\hrho_{e}(0.8 n)}{p-3d-1}>\rho+\frac{1-\rho}{6}\Bigg),
	\end{align}
	where $\calE(\calG)\cap\calE( \hatcalT_{\rm{SCL}})$ in inequality~\eqref{ieq:ieq2} satisfies $|\calE(\calG)\cap\calE( \hatcalT_{\rm{SCL}})|\geq p-3d+1$; $e_{1}$ and $e_{2} $ in inequality \eqref{ieq:ieq3} are two edges in $\calE(\calG)\cap\calE( \hatcalT_{\rm{SCL}})$; inequality \eqref{ieq:ieq5} results from the simple calculation
	\begin{equation}
	    	\frac{9(p-3d-1)}{2p-27d-2}=\frac{9}{2}\Big(1+\frac{21d}{2p-2-27d}\Big)\leq\frac{9}{2}\Big(1+\frac{21d}{133d-2}\Big)\leq 6.
	\end{equation}

	Using Sanov's theorem~\cite{zeitouni1998large}, we have
	\begin{align}
		K (\calC\cap \calD)&=\varliminf_{n\rightarrow\infty}-\frac{1}{n}\log\Pr\big( \calC(n)\cap \calD(n)\big)\\
		&\geq 0.8(p-3d-1)D\Big(\frac{5}{6}\theta\Big\|\theta\Big)\\
		&\geq 0.8(79d-1)D\Big(\frac{5}{6}\theta\Big\|\theta\Big)\\
		&\geq 125D\Big(\frac{5}{6}\theta\Big\|\theta\Big).
	\end{align}
	\begin{figure}[t]
		\centering 
		\includegraphics[width=0.45\textwidth]{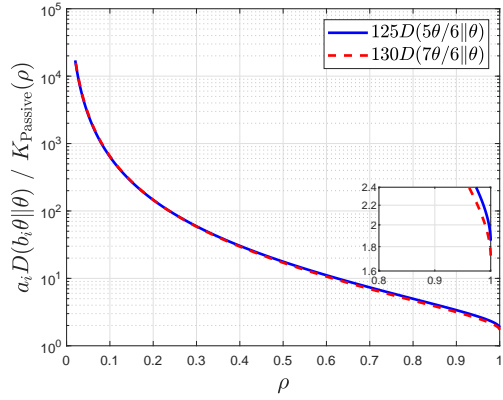} 
		\caption{The ratios of  $125D(\frac{5}{6}\theta\|\theta)$, $130D(\frac{7}{6}\theta\|\theta)$ and passive structure learning error. Note that correlation $\rho$ is related to $\theta$ as $\rho=1-2\theta$.} 
		\label{fig:thetaKL} 
	\end{figure}
	We then prove that 
	\begin{equation}
	\varliminf_{n\rightarrow\infty}-\frac{1}{n}\log\Pr\bigg(\rho-\frac{1-\rho}{6}>\hrho(\halpha n)\bigg)\geq c_{\rho}K _{\Passive}(\rho),
	\end{equation}
	and we set $ \calC^{\prime}(n)=\big\{\rho-\frac{1-\rho}{6}>\hrho(\halpha n)\big\}$ and $\calD^{\prime}(n)=\calD(n)$ as above. Consider,
	\begin{align}
		\Pr( \calC^{\prime}\cap \calD^{\prime})&=\Pr\Bigg(\frac{\sum_{e\in\calE(\calG)\cap\calE( \hatcalT_{\rm{SCL}})}\hrho_{e}+\sum_{e^{\prime}\in\calE( \hatcalT_{\rm{SCL}})\backslash\calE(\calG)}\hrho_{e^{\prime}}}{p-1}<\rho-\frac{1-\rho}{6}\Bigg)\\
		&\leq\Pr\Bigg(\frac{\sum_{e\in\calE(\calG)}\hrho_{e}}{p-1}<\rho-\frac{1-\rho}{6}\Bigg)\label{ieq:ieq4}\\
		&\leq 7\Pr\Bigg(\frac{\sum_{e\in\calE(\calG)}\hrho_{e}(0.8n)}{p-1}<\rho-\frac{1-\rho}{6}\Bigg),
	\end{align}
	where inequality \eqref{ieq:ieq4} comes from that $ \hatcalT_{\rm{SCL}}$ is derived by finding the \ac{mst} with $\hrho_{ij}$ as weights, so $\sum_{e\in\calE(\calG)\backslash\calE( \hatcalT_{\rm{SCL}})}\hrho_{e}\leq\sum_{e^{\prime}\in\calE( \hatcalT_{\rm{SCL}})\backslash\calE(\calG)}\hrho_{e^{\prime}}$. 
	Then we have
	\begin{align}
		K (\calC^{\prime}\cap \calD^{\prime})&=\varliminf_{n\rightarrow\infty}-\frac{1}{n}\log\Pr\big( \calC^{\prime}(n)\cap \calD^{\prime}(n)\big)\\
		&\geq 0.8(p-1)D\Big(\frac{7}{6}\theta\Big\|\theta\Big)\\
		&\geq 130 D\Big(\frac{7}{6}\theta\Big\|\theta\Big),
	\end{align}
	where the last inequality follows as $p\geq 82d\geq 164.$

	Fig.~\ref{fig:thetaKL} shows that the ratios of $125D(\frac{5}{6}\theta\|\theta)$, $130D(\frac{7}{6}\theta\|\theta)$ and passive structure learning error are larger than 1.4 for all  values of $\rho$.
\end{proof}

\subsubsection{Proof of Lemma \ref{lem:confide}}
\begin{proof}[Proof of Lemma \ref{lem:confide}]
	Since the error that appears in the \ac{scl} algorithm belongs to one of $t$-hop errors, the event in the probability in inequality~\eqref{ieq:confide} is equivalent to the event that the edges of any $t$-hop error in $\hatcalT_{\rm{SCL}}$ are confident for all $t\geq 2$.

	We take different methods to handle the events that the edges of any 2-hop error are confident and the edges of any $t$-hop error are confident for $t\geq 3$.

	For a $t$-hop error with $t\geq 3$, the event that the edges of any $t$-hop error are confident is included in the event of the appearance of the $t$-hop error, so the exponent of the event that the edges of 
	any $t$-hop error are confident is lower bounded by the exponent of the $t$-hop error. Then the desired result is proved by Lemma \ref{lem:thoperror}.

	For a 2-hop error, given a three-node subtree $i-j-k$, we want to prove  that
	\begin{align}
    	\varliminf_{n\rightarrow\infty}-\frac{1}{n}\log\Pr\big(\calC_{ikj}(\halpha n)\big)\geq c_{\rho}K _{\Passive}(\rho).
	\end{align}
	Recalling Lemma \ref{lem:twoseq}, we let $\calF(n)\triangleq\calC_{ikj}(\halpha n)$ be the event that the wrongly learned tree $i-k-j$ is confident and $\calG(n)\triangleq\big\{\rho-\frac{1-\rho}{6}\leq\hrho(\halpha n)\leq\rho+\frac{2(1-\rho)}{9}\big\}$. Lemma \ref{lem:rhobound} shows that $K (\calG^{\complement})\geq c_{\rho}K _{\Passive}(\rho)$, so we only need to prove that $\alpha K (\calF\cap \calG)\geq c_{\rho}K _{\Passive}(\rho)$. We have 
	\begin{align}
		&\Pr\big(\calF(n)\cap \calG(n)\big)\nonumber\\
		&=\Pr\bigg(\hrho_{ij}\!\leq\! \hrho_{ik}\frac{11+9\hrho}{20},\hrho_{ij}\!\leq\! \hrho_{jk}\frac{11+9\hrho}{20},\hrho_{ik}\!\geq\! \hrho_{ij},\hrho_{jk}\!\geq\! \hrho_{ij},\rho-\frac{1-\rho}{6}\!\leq\!\hrho\!\leq\!\rho+\frac{2(1-\rho)}{9}\bigg)\\
		&\leq\Pr\bigg(\hrho_{ij}\leq \hrho_{ik}\frac{11+9\hrho}{20}, \rho-\frac{1-\rho}{6}\leq\hrho\leq\rho+\frac{2(1-\rho)}{9}\bigg)\\
		&\leq\Pr\bigg(\hrho_{ij}\leq \hrho_{ik}\frac{13+7\rho}{20}, \rho-\frac{1-\rho}{6}\leq\hrho\leq\rho+\frac{2(1-\rho)}{9}\bigg)\\
		&\leq\Pr\bigg(\hrho_{ij}\leq \hrho_{ik}\frac{13+7\rho}{20}\bigg)\\
		&=\Pr\Big(\hrho_{ij}(\halpha n)\leq \hrho_{ik}(\halpha n)\frac{13+7\rho}{20}, \halpha\in \calM\Big)+\Pr\Big(\hrho_{ij}(\halpha n)\leq \hrho_{ik}(\halpha n)\frac{13+7\rho}{20}, \halpha\in \calM^{\complement}\Big) \\
		&\leq \max_{\beta\in \calM}\Pr\Big(\hrho_{ij}(\beta n)\leq \hrho_{ik}(\beta n)\frac{13+7\rho}{20}\Big)+\Pr( \halpha\in \calM^{\complement}), \label{eqn:maxBeta}
	\end{align}
	where $\calM\subseteq \Psi$ is any  subset of $\halpha$.
	
	Let $\rho^{(Q)}_{ij}\triangleq\bbE_Q[X_iX_j]$ be the correlation computed based on the distribution $Q$. Continuing from inequality~\eqref{eqn:maxBeta} and  by using the theory of large deviations (Sanov's theorem \cite{zeitouni1998large}), 
	\begin{align}
		&\varliminf_{n\rightarrow\infty}-\frac{1}{n}\log\Pr(\calF(n)\cap \calG(n))\nn\\*
		&\geq \min\bigg\{\min_{\beta\in \calM} \beta\cdot \min_{Q\in\calP(\calX^{3})} \Big\{D(Q\|P):\rho^{(Q)}_{ij}\leq\rho^{(Q)}_{ik}\frac{13+7\rho}{20}\Big\},K (\halpha\in \calM^{\complement})\bigg\}. \label{eqn:mink}
	\end{align}

	We then prove that $\varliminf_{n\rightarrow\infty} -\frac{1}{n}\log\Pr(\calF\cap \calG)\geq c_{\rho}K _{\Passive}(\rho)$ in a computer-aided approach. We first note that the optimization problem 
	\begin{align}
	K_{2,\mathrm{Conf}}(\rho)\triangleq\min_{Q\in\calP(\calX^{3})}    \Big\{D(Q\|P):\rho^{(Q)}_{ij}\leq\rho^{(Q)}_{ik}\frac{13+7\rho}{20}\Big\}
	\end{align}
	is a convex optimization problem (as $Q\mapsto D(Q\|P)$ is convex and the inequality constraint is linear in $Q$), where the subscript indicates that it is the exponent of the 2-hop error events on confident edges. The plot of $K_{2,\mathrm{Conf}}(\rho)$ is shown in Fig.~\ref{fig:thoperr_1}.
	
	\begin{figure}[t]
		\centering 
		\includegraphics[width=0.45\textwidth]{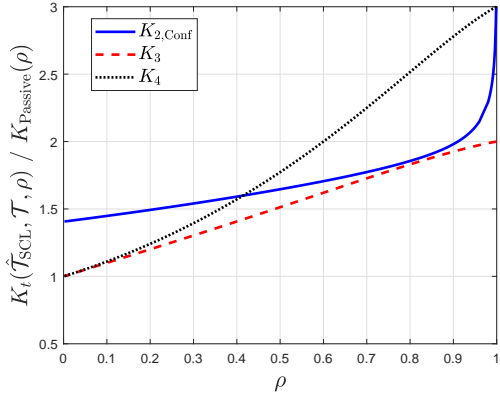} 
		\caption{The ratio of error exponents between $t$-hop error and passive structure learning error.} 
		\label{fig:thoperr_1} 
	\end{figure}

	Fig.~\ref{fig:thoperr_1} shows that $K_{2,\mathrm{Conf}}(\rho)\geq K_{3}(\hatcalT_{{\rm SCL}},\calT,\rho)$. Recalling Lemma~\ref{lem:thoperror}, we have shown that for each $0<\rho<1$
	\begin{align}
	    \min\Big\{\min_{\beta\in \tilde{\calM} }\,\beta \cdot K _{3}(\hatcalT_{{\rm SCL}},\calT,\rho),K (\halpha\in \tilde{\calM}^{\complement})\Big\}\geq c_{\rho}K_{\Passive}(\rho),
	\end{align}
	where $\tilde{\calM}\subseteq \Psi$ is the appropriately chosen subset of values of $\halpha$ that has been designed in Lemma~\ref{lem:thoperror}. 
	Set $\calM=\tilde{\calM}$ for each $\rho$. Lemma~\ref{lem:thoperror} and the fact that $K_{2,\mathrm{Conf}}(\rho)\geq K _{3}(\hatcalT_{{\rm SCL}},\calT,\rho)$ for all $\rho\in(0,1)$ show that
	\begin{align}
	    \min\bigg\{\min_{\beta\in \calM} \beta\cdot \min_{Q\in\calP(\calX^{3})} \Big\{D(Q\|P):\rho^{(Q)}_{ij}\leq\rho^{(Q)}_{ik}\frac{13+7\rho}{20}\Big\},K (\halpha\in \calM^{\complement})\bigg\}\geq c_{\rho}K_{\Passive}(\rho),
	\end{align}
	which proves that desired inequality.
\end{proof}
\subsubsection{Proof of Lemma \ref{lem:packnum}}
\begin{proof}[Proof of Lemma \ref{lem:packnum}]
    Lemma \ref{lem:thoperror} shows that 
    \begin{align}
    	K _{t}(\hatcalT_{{\rm SCL}},\calT,\rho,P_{\halpha})\geq c_{\rho}K _{\Passive}(\rho) \text{ for all }t\geq 3.
    \end{align}
    Thus, we only need to consider  $2$-hop errors for unconfident edges.

	Define $\calH(n)$ to be the event $\calH(n)\triangleq \{N(\scC_{\mathrm{E}}(\halpha n)^{\complement},\calT\big(\scC_{\mathrm{V}}(\halpha n)^{\complement}\big),2)\geq 13\}$. In addition, define $\calI(n)$ to be the event $\calI(n)\triangleq\cap_{t\geq 3}\big(E_{t}(\hatcalT_{{\rm SCL}},\calT,\halpha n)\big)^{\complement}$, namely that there  exists only $2$-hop errors (and not $t$-hop errors for $t\ge 3$). 
	Lemma~\ref{lem:thoperror} shows that $K (\calI^{\complement})\geq c_{\rho}K _{\Passive}(\rho)$, so we can focus the analysis  of the probability of the event $\calH\cap \calI$. 
	
	If $\{i,j\},\{j,k\}\in \hatcalT_{{\rm SCL}}(\halpha n)$ and $\calC_{ijk}(\halpha n)$ does not hold, we call $(i,j,k)$ an unconfident triple. Two unconfident triples are independent if they do not  share an edge in the tree that is induced by them. 
	
	The set of independent unconfident triples $T_{{\rm uc}}(m)$ is defined as follows. $(i,j,k)$ is in $T_{{\rm uc}}(m)$ if $\{i,j\},\{j,k\}\in \hatcalT_{{\rm SCL}}(m)$, $\calC_{ijk}(\halpha n)$ does not hold and $\calT(\{i,j,k\})$ is connected, and for any other
	$(i^{\prime},j^{\prime},k^{\prime}) \in T_{{\rm uc}}(m)$, $|\{i,j,k\}\cap\{i^{\prime},j^{\prime},k^{\prime}\}|\leq 1$. The condition that $\calT(\{i,j,k\})$ is connected is to ensure that there are two edges among these three nodes, i.e., there 
	are only 2-hop errors.
	
	Then we have
	\begin{align}
		\varliminf_{n\rightarrow\infty}-\frac{1}{n}\log\Pr\big(\calH(n)\cap \calI(n)\big)&\geq \varliminf_{n\rightarrow\infty}-\frac{1}{n}\log\Pr\big(|T_{{\rm uc}}(\halpha n)|\geq 13\big)\\
		&=\varliminf_{n\rightarrow\infty}-\frac{1}{n}\log\bigg(\sum_{\beta\in\Psi}\Pr(|T_{{\rm uc}}(\halpha n)|\geq 13,\halpha=\beta)\bigg)\\
		&=\varliminf_{n\rightarrow\infty}-\frac{1}{n}\log\Pr\big(|T_{{\rm uc}}(\halpha n)|\geq 13,\halpha=0.8\big)\\
		&\geq \varliminf_{n\rightarrow\infty}-\frac{1}{n}\log\Pr \big(|T_{{\rm uc}}(0.8 n)|\geq 13\big)\\
		&=0.8 \varliminf_{n\rightarrow\infty}-\frac{1}{n}\log\Pr\big(|T_{{\rm uc}}(n)|\geq 13\big),
	\end{align}
	There are two different kinds of unconfident triples. The first type consists of  triples for which the corresponding learned structures based on $\hatcalT_{{\rm SCL}}$ are correct; the other type consists of the triples for which the corresponding learned structures are incorrect. For the triple $(i,j,k)$ that yields the correct    structure $i-j-k$, we have
	\begin{align}
		&\varliminf_{n\rightarrow\infty}-\frac{1}{n}\log\Pr\big(\hatcalT_{{\rm SCL}}(n)(\{i,j,k\})=\calT(\{i,j,k\}),\calC_{ijk}(n)^{\complement}\big)\nonumber\\
		&=\varliminf_{n\rightarrow\infty}-\frac{1}{n}\log\Pr\bigg(\{\hrho_{ij}( n)>\hrho_{ik}( n),\hrho_{jk}( n)>\hrho_{ik}( n)\}\cap \nonumber\\*
		& \qquad\Big(\Big\{\hrho_{ik}( n)>\hrho_{ij}( n)\frac{11+9\hrho( n)}{20}\Big\}\cup\Big\{\hrho_{ik}( n)>\hrho_{jk}( n)\frac{11+9\hrho( n)}{20}\Big\}\Big)\bigg)\\
		&\geq \varliminf_{n\rightarrow\infty}-\frac{1}{n}\log\Pr\Big(\hrho_{ik}( n)>\hrho_{ij}( n)\frac{11+9\hrho( n)}{20}\Big).
	\end{align}
	For the triple $(i,j,k)$ with the  structure $i-j-k$ that learns the wrong graph, we have
	\begin{align}
		&\varliminf_{n\rightarrow\infty}-\frac{1}{n}\log\Pr\big(\hatcalT_{{\rm SCL}}(n)(\{i,j,k\})\neq \calT(\{i,j,k\}),\calC_{ijk}(n)^{\complement}\big)\nonumber\\
		&\geq\varliminf_{n\rightarrow\infty}-\frac{1}{n}\log\Pr\Big(\hrho_{ik}( n)>\hrho_{ij}( n),\hrho_{ij}( n)>\hrho_{ik}( n)\frac{11+9\hrho( n)}{20}\Big)\\
		&\geq \varliminf_{n\rightarrow\infty}-\frac{1}{n}\log\Pr\big(\hrho_{ik}( n)>\hrho_{ij}( n)\big)\\
		&\geq \varliminf_{n\rightarrow\infty}-\frac{1}{n}\log\Pr\Big(\hrho_{ik}( n)>\hrho_{ij}( n)\frac{11+9\hrho( n)}{20}\Big)\\
		&=\min \bigg\{\varliminf_{n\rightarrow\infty}-\frac{1}{n}\log\Pr\Big(\hrho_{ik}( n)>\hrho_{ij}( n)\frac{11+9\hrho( n)}{20},\hrho(n)\geq \rho-\frac{1-\rho}{6}\Big),\nonumber\\
		&\qquad \varliminf_{n\rightarrow\infty}-\frac{1}{n}\log\Pr\big(\hrho(n)< \rho-\frac{1-\rho}{6}\big)\bigg\}.
	\end{align}
	Lemma \ref{lem:rhobound} implies we only need to consider the case where $\hrho(n)\geq \rho-\frac{1-\rho}{6}$ in which 
	\begin{align}
		&\varliminf_{n\rightarrow\infty}-\frac{1}{n}\log\Pr\Big(\hrho_{ik}(n)>\hrho_{ij}(n)\frac{11+9\hrho(n)}{20},\hrho(n)\geq \rho-\frac{1-\rho}{6}\Big)\nonumber\\
		&\geq \varliminf_{n\rightarrow\infty}-\frac{1}{n}\log\Pr\Big(\hrho_{ik}(n)>\hrho_{ij}(n)\frac{19+21\rho}{40},\hrho(n)\geq \rho-\frac{1-\rho}{6}\Big)\\
		&\geq \varliminf_{n\rightarrow\infty}-\frac{1}{n}\log\Pr\Big(\hrho_{ik}(n)>\hrho_{ij}(n)\frac{19+21\rho}{40}\Big)\\
		&= \min_{Q\in\calP(\calX^{3})}\bigg\{D(Q\|P):\rho_{ik}^{(Q)}\geq\rho_{ij}^{(Q)}\frac{19+21\rho}{40}\bigg\}.
	\end{align}
	Then we have that
	\begin{align}
		&\varliminf_{n\rightarrow\infty}-\frac{1}{n}\log\Pr(|T_{{\rm uc}}(\halpha n)|\geq 13,\halpha=0.8)\nonumber\\
		&\geq 0.8\times 13\times\min_{Q\in\calP(\calX^{3})}\bigg\{D(Q\|P):\rho_{ik}^{(Q)}\geq\rho_{ij}^{(Q)}\frac{19+21\rho}{40}\bigg\}.
	\end{align}
	\begin{figure}[t] 
		\centering 
		\includegraphics[width=0.45\textwidth]{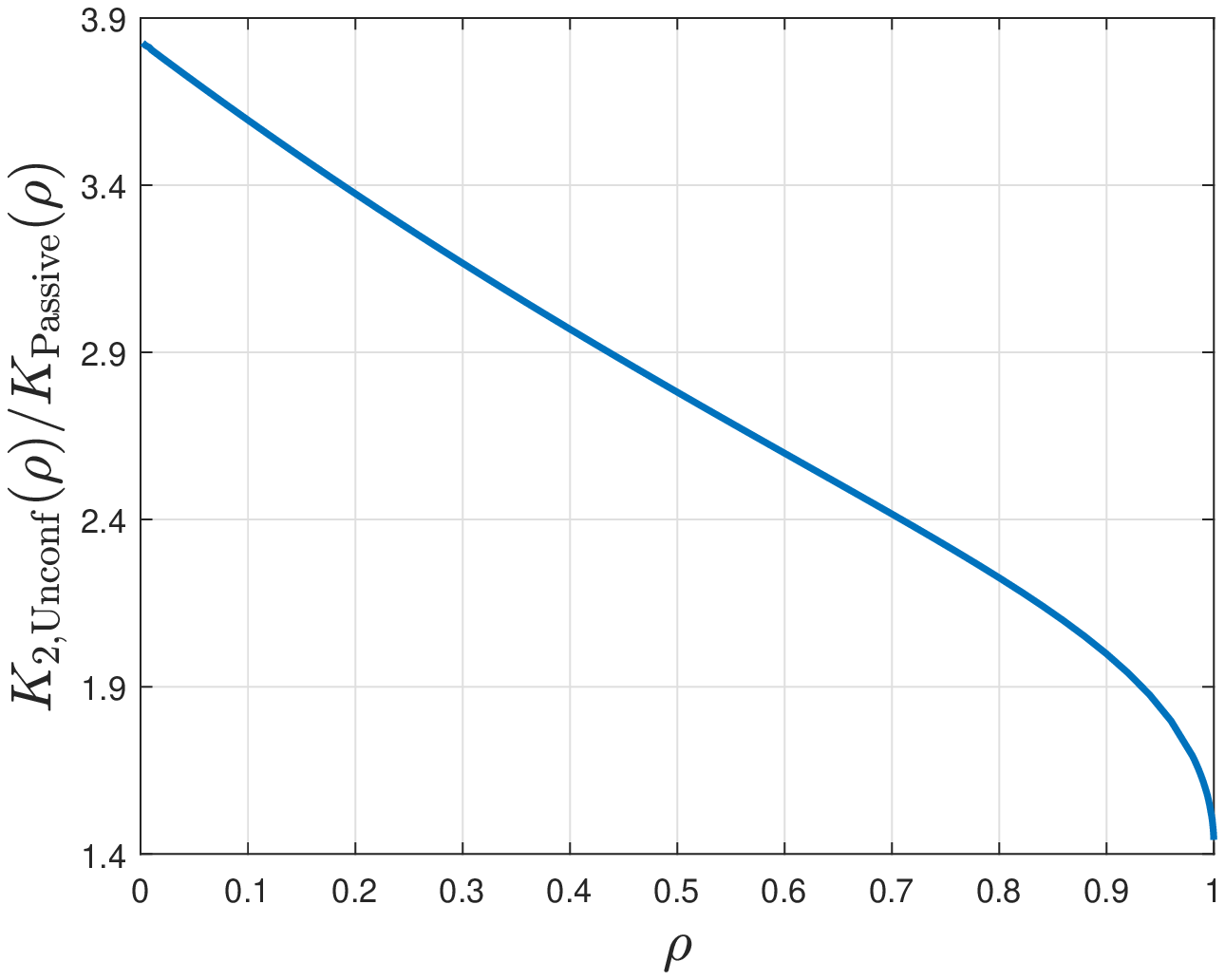} 
		\caption{The ratio between the exponent of the sequence of events $\{N(\scC_{\mathrm{E}}(\halpha n)^{\complement},\calT(\scC_{\mathrm{V}}(\halpha n)^{\complement}),2)\geq 13\}_{n\in\bbN}$ and the error exponent of passive structure learning algorithm.} 
		\label{fig:unconfide_tri} 
	\end{figure}
	
	We then prove that
	\begin{align}
    	K_{2,\mathrm{Unconf}}(\rho)\triangleq 0.8\times13\times\min_{Q\in\calP(\calX^{3})}\bigg\{D(Q\|P):\rho_{ik}^{(Q)}\geq\rho_{ij}^{(Q)}\frac{19+21\rho}{40}\bigg\}\geq c_{\rho}K _{\Passive}(\rho), \label{eqn:minDQP}
	\end{align}
	using a computer-aided approach, where the subscript of $K_{2,\mathrm{Unconf}}$ stands for 2-hop error events on unconfident edges.
	
	We first note that the optimization problem in~\eqref{eqn:minDQP} is  convex in the entries of $Q$ and so it can be evaluated efficiently and accurately using off-the-shelf convex optimization software. 
	
	Fig.~\ref{fig:unconfide_tri} shows that the ratio of error exponents for the event $\{N(\scC_{\mathrm{E}}(\halpha n)^{\complement},\calT(\scC_{\mathrm{V}}(\halpha n)^{\complement}),2)\geq 13\}_{n\in\bbN}$ and passive structure learning error is larger than $c_{\rho}K _{\Passive}(\rho)$.
\end{proof}

\subsubsection{Proof of Lemma \ref{lem:unconfide}}
To prove this important lemma, we start with a few crucial supporting lemmas. 
\begin{lemma}\label{lem:tilpbound}
	In the $\hatcalT_{\rm{SCL}}(\halpha n)$ learned in Algorithm~\ref{algo:alathe} that only has 2-hop errors, if 
	\begin{align}
	    N(\scC_{\mathrm{E}}(\halpha n)^{\complement},\calT(\scC_{\mathrm{V}}(\halpha n)^{\complement}),2)\leq 12,
	\end{align}
	then $\tilp(\halpha n)<26d$.
\end{lemma}
\begin{proof}
	The condition that $\hatcalT_{\rm{SCL}}(\halpha n)$ only has $2$-hop errors guarantees that $\calT(\scC_{\mathrm{V}}(\halpha n)^{\complement})$ does not have isolated nodes. We will use a proof by contraposition, i.e., we show that if $\tilp(\halpha n)\geq 26d$, then $N(\scC_{\mathrm{E}}(\halpha n)^{\complement},\calT(\scC_{\mathrm{V}}(\halpha n)^{\complement}),2)\geq 13$.

	The induced graph $\calT(\scC_{\mathrm{V}}(\halpha n)^{\complement})$ is not necessarily a connected graph, and we assume that it has $k$ components: $\calT_{1}$, \ldots, $\calT_{k}$. For each component, we 
	show that $N(\calE(\calT_{i}),\calT_{i},2)\geq \frac{|\calV(\calT_{i})|}{2d}$. See Fig.~\ref{fig:packing} for an illustration of this proof.
	
	\begin{figure}[t] 
		\centering 
		\includegraphics[width=0.36\textwidth]{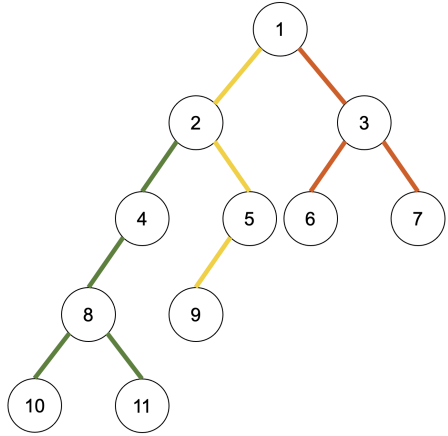} 
		\caption{Edge $\{8,10\}$ is the deepest edge and is  chosen first. Then, all then green edges ($\{2,4\}$, $\{4,8\}$, $\{8,10\}$, and $\{8,11\}$) will be deleted. Edge $\{5,9\}$ has the largest depth among the remaining edges, and all the yellow edges ($\{1,2\}$, $\{2,5\}$, and $\{5,9\}$) will then be deleted. Finally, edge $\{3,6\}$ is chosen, and all the red edges ($\{1,3\}$, $\{3,6\}$, $\{3,7\}$) will then be deleted.} 
		\label{fig:packing} 
	\end{figure}
	
	In the component $\calT_{i}$, we first choose an arbitrary node as the root node. The depth of a node is the length of the path between it and the root node. Choose an edge $\{j,l\}$ 
	incident to a leaf node with the largest depth and then delete all the edges $\{m,s\}$ in $\calT_{i}$ that satisfy $d_{\calT_{i}}(\{j,l\},\{m,s\})\leq 1$. Note that the number of the deleted edges 
	in this step is upper bounded by $2d-1$. We continue to choose an edge in the remaining edge set with the deepest terminal node in $\calT_{i}$. Repeating this process, we obtain the following bound
	\begin{align}
		N(\calE(\calT_{i}),\calT_{i},2)\geq \frac{|\calV(\calT_{i})|-1}{2d-1}.
	\end{align}
	If $|\calV(\calT_{i})|\geq 2d$, we have $N(\calE(\calT_{i}),\calT_{i},2)\geq\frac{|\calV(\calT_{i})|-1}{2d-1}\geq\frac{|\calV(\calT_{i})|}{2d}$. If $|\calV(\calT_{i})|< 2d$, since $N(\calE(\calT_{i}),\calT_{i},2)\geq 1$, so $N(\calE(\calT_{i}),\calT_{i},2)\geq \frac{|\calV(\calT_{i})|}{2d}$. Hence, we have  $N(\calE(\calT_{i}),\calT_{i},2)\geq \frac{|\calV(\calT_{i})|}{2d}$.

	Finally, we have
	\begin{align}\label{ieq:tilp}
		N(\scC_{\mathrm{E}}(\halpha n)^{\complement},\calT(\scC_{\mathrm{V}}(\halpha n)^{\complement}),2)&=\sum_{i} N(\calE(\calT_{i}),\calT_{i},2)\geq \frac{\sum_{i}|\calV(\calT_{i})|}{2d}=\frac{\tilp(\halpha n)}{2d}.
	\end{align}
	If $\tilp(\halpha n)\geq 26d$, the inequality in \eqref{ieq:tilp} implies that $N(\scC_{\mathrm{E}}(\halpha n)^{\complement},\calT(\scC_{\mathrm{V}}(\halpha n)^{\complement}),2)\geq 13$.
\end{proof}

\begin{lemma}\label{lem:numsample}
	If $p\geq 82d$ and $\tilp(\halpha n)<26d$, then $p/ \tilp(\halpha n)>3$.
\end{lemma}
\begin{proof}
    We have
	$
		3\times \tilp(\halpha n) <3\times 26d=78d<82d\leq p.
	$
\end{proof}

\begin{remark}
	 Algorithm~\ref{algo:alathe} implies that there are $(\halpha+(1-\halpha)p/\tilp)n$ scalar samples for each node in $\scC_{\mathrm{E}}(\halpha n)^{\complement}$, and Lemmas~\ref{lem:tilpbound} and \ref{lem:numsample} show that $N(\scC_{\mathrm{E}}(\halpha n)^{\complement},\calT(\scC_{\mathrm{V}}(\halpha n)^{\complement}),2)\leq 12$ is sufficient to ensure that there are at least $(\halpha+3(1-\halpha))n=(3-2\halpha)n$ scalar samples for each unconfident node.
\end{remark}
\begin{proof}[Proof of Lemma \ref{lem:unconfide}]
	Note that the edges in the induced subgraph $\hatcalT_{\Active}\big(\scC_{\mathrm{V}}(\halpha n)^{\complement}\big)$ not only contain the unconfident edges, but also contain some confident edges. However, Lemma \ref{lem:confide} shows that the exponent of the event that the confident edges are wrongly estimated is larger than $c_{\rho}K _{\Passive}(\rho)$. Thus, in the following proof, we focus on the unconfident edges in $\hatcalT_{\Active}\big(\scC_{\mathrm{V}}(\halpha n)^{\complement}\big)$.
	
	We set $\calJ(n)\triangleq\{\hatcalT_{\Active}\big(\scC_{\mathrm{V}}(\halpha n)^{\complement}\big)\neq \calT\big(\scC_{\mathrm{V}}(\halpha n)^{\complement}\big)\}$ be the event that the learned structure of unconfident nodes is wrong, and $\calK(n)\triangleq\{N(\scC_{\mathrm{E}}(\halpha n)^{\complement},\calT(\scC_{\mathrm{V}}(\halpha n)^{\complement}),2)\leq 12\}$ be 
	the event that the 2-packing number of the unconfident edges set is not more than  $12$. Lemma \ref{lem:packnum} then shows that 
	\begin{align}
		\varliminf_{n\rightarrow\infty}-\frac{1}{n}\log\Pr\big(\calK(n)^{\complement}\big)\geq c_{\rho}K _{\Passive}(\rho).
	\end{align}
	Then we only need to consider the event $\calJ\cap \calK$. Lemmas \ref{lem:tilpbound} and \ref{lem:numsample} show that there are at least $\big[\halpha+3(1-\halpha)\big]n$ scalar samples for every unconfident node, since Algorithm~\ref{algo:alathe} will distribute $(1-\halpha)pn$ scalar samples only among $\tilp$ unconfident nodes in the local refinement stage. 

	For $\rho\in [0.8,1)$, we will numerically compute the exponent of the event that $\{\halpha \neq 0.8\}$. It follows from Table~\ref{table:alpha} that $\halpha \neq 0.8$ implies $\hrho<0.76$. Furthermore, for $\rho=0.8$, we have $\rho-\frac{1-\rho}{5}=0.76$. To show that $K\big(\{\hrho(\halpha n)\leq \rho-\frac{1-\rho}{5}\}_{n\in\bbN}\big)\geq c_{\rho}K_{\Passive}(\rho)$, we only need to show that $130D(\frac{6\theta}{5}\|\theta)\geq c_{\rho}K_{\Passive}(\rho)$, following the similar procedures of Lemma \ref{lem:rhobound}.
	\begin{figure}[t] 
		\centering
		\subfigure[Lower bound of the ratio of $K\big(\{\hrho(\halpha n)\leq \rho-\frac{1-\rho}{5}\}_{n\in\bbN}\big)$ and $K_{\Passive}(\rho)$.]{
		\begin{minipage}[t]{0.45\linewidth}
		\centering
		\includegraphics[width=3in]{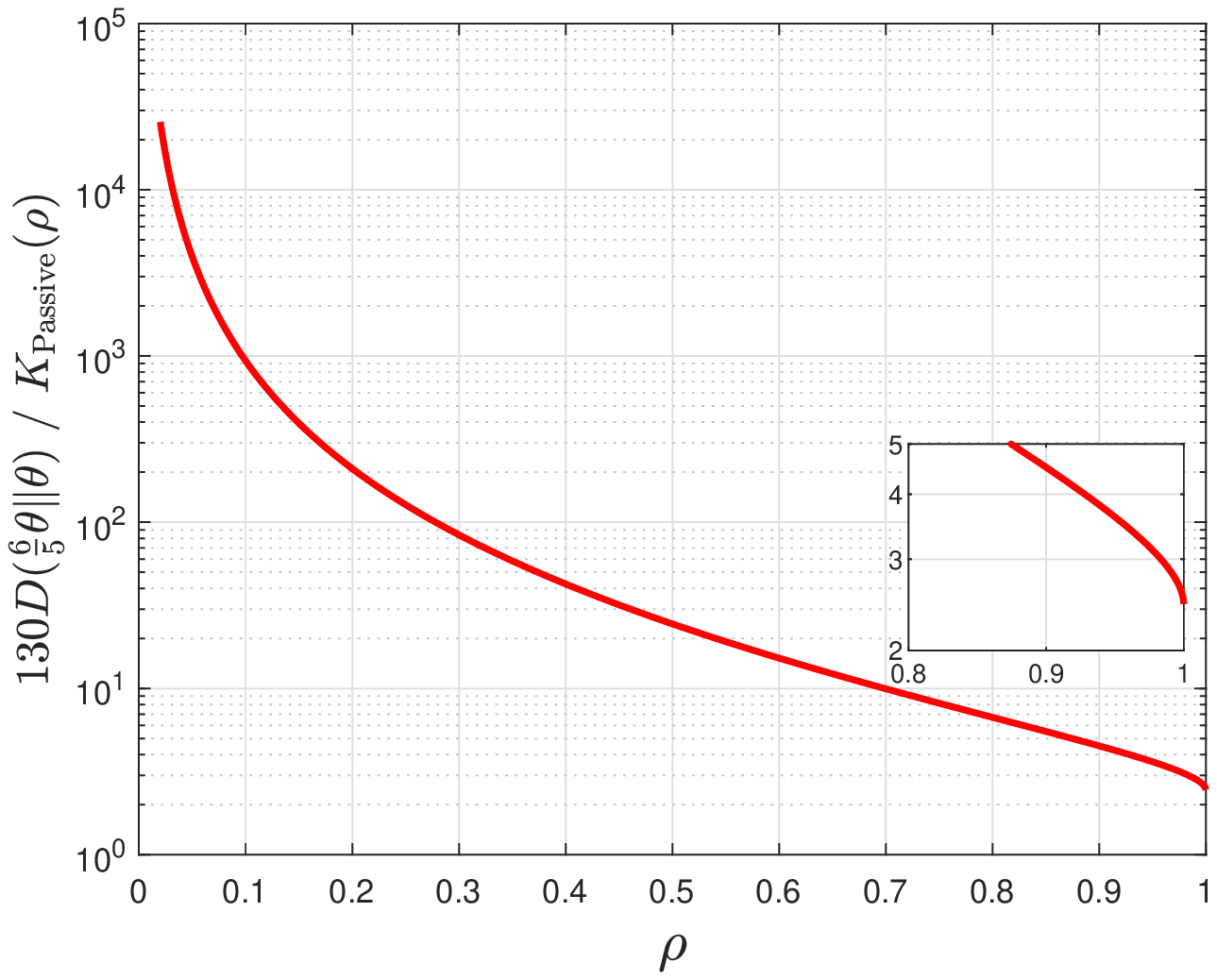}
		\end{minipage}%
		\label{fig:hrho5}
		}%
		\hspace{.15in}
		\subfigure[Lower bound of the ratio of $K\big(\{\hrho(\halpha n)\leq \rho-\frac{1-\rho}{5.72}\}_{n\in\bbN}\big)$ and $K_{\Passive}(\rho)$.]{
		\begin{minipage}[t]{0.45\linewidth}
		\centering
		\includegraphics[width=3in]{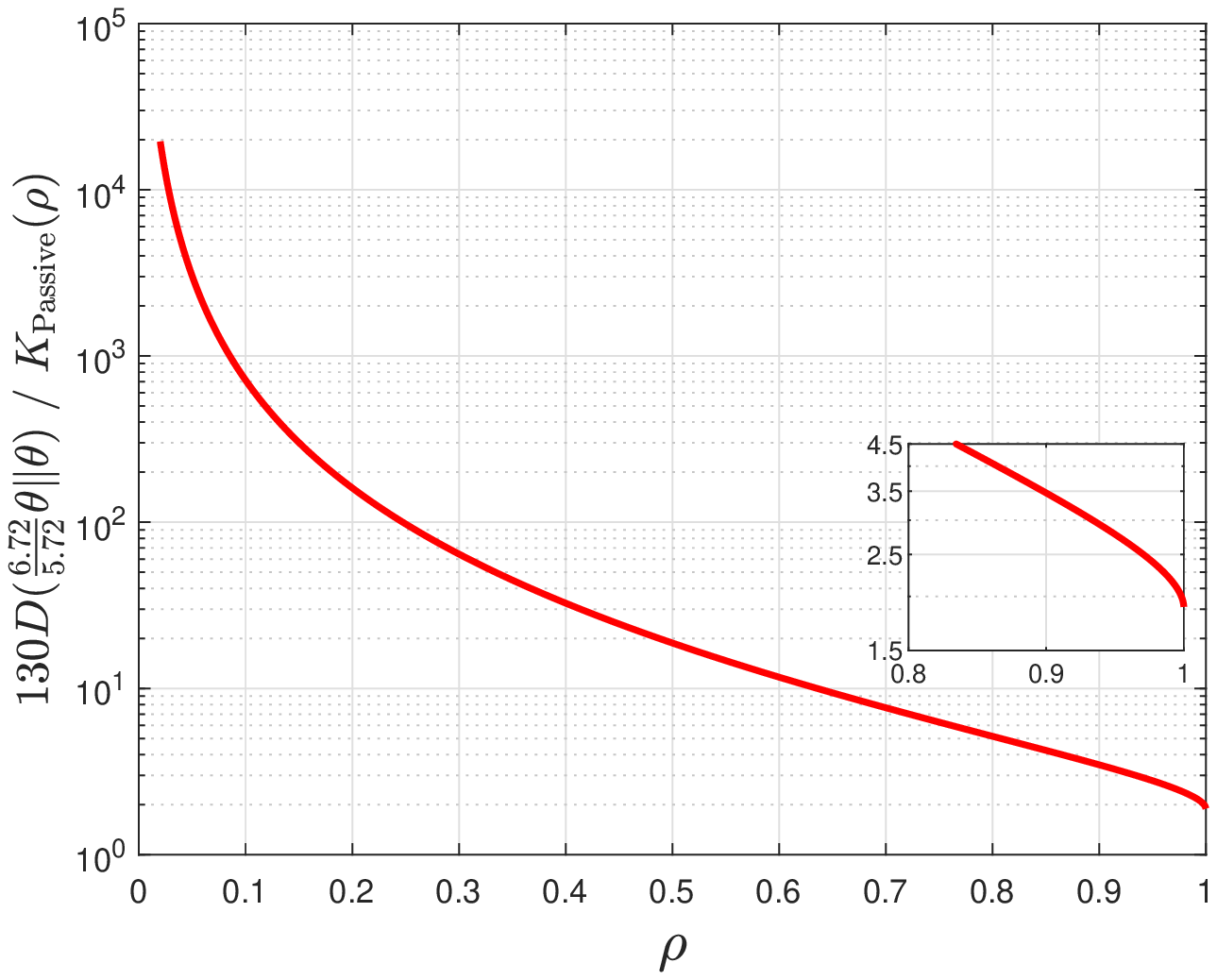}
		\end{minipage}%
		\label{fig:hrho57}
		}%
		\centering
		\caption{Lower bounds of the ratios of $K\big(\{\hrho(\halpha n)\leq \rho-\frac{1-\rho}{k}\}_{n\in\bbN}\big)$ and $K_{\Passive}(\rho)$ for $k=5$ and $k=5.72$.}
	\end{figure}
	Fig.~\ref{fig:hrho5} shows that this ratio is larger than $1.4$, so we only need to lower bound $K(\calJ\cap \calK)$ with an additional condition $\{\halpha =0.8\}$, i.e., we only need to lower bound $K(\calJ\cap \calK\cap\{ \halpha=0.8\})$. We define $\tiln (\calN)$ for any set of nodes $\calN\subseteq \calV(\calT)$ as the minimal number of scalar samples assigned to 
	each node in the set $\calN$. Lemmas \ref{lem:tilpbound} and \ref{lem:numsample} show that the event $\{N(\scC_{\mathrm{E}}(\halpha n)^{\complement},\calT(\scC_{\mathrm{V}}(\halpha n)^{\complement}),2)\leq 12\}$ implies that there are at least $\big[\halpha+(1-\halpha)\frac{p}{\tilp}\big]n\geq \big[\halpha+3(1-\halpha)\big]n$ for each 
	unconfident node, i.e., $\tiln \big(\scC_{\mathrm{E}}(\halpha n)^{\complement}\big)\geq \big[\halpha+3(1-\halpha)\big]n$. Defining the event $\calL(\calN,a) \triangleq \{\tiln \big(\calN \big)\geq a\}$, we have
	\begin{align}
		&\varliminf_{n\rightarrow\infty}-\frac{1}{n}\log\Pr\big(\calJ(n)\cap \calK(n)\cap\{ \halpha=0.8\}\big)\nonumber\\*
		&=\varliminf_{n\rightarrow\infty}-\frac{1}{n}\log\Pr\bigg(\calJ(n)\cap \calK(n)\cap\{ \halpha=0.8\}\cap \calL\Big(\scC_{\mathrm{E}}(\halpha n)^{\complement},\big[\halpha+3(1-\halpha)\big]n\Big)\bigg)\label{unconf:eq1}\\
		&=\varliminf_{n\rightarrow\infty}-\frac{1}{n}\log\Pr\bigg(\bigcup_{\calN\subseteq \calV(\calT)} \Big\{\calJ(n)\cap \calK(n)\cap \calL\Big(\scC_{\mathrm{E}}(\halpha n)^{\complement},\big[\halpha+3(1-\halpha)\big]n\Big) \cap\nonumber\\*
		&\qquad\qquad \{ \halpha=0.8\} \cap\big\{\scC_{\mathrm{E}}(\halpha n)^{\complement}=\calN\big\}\Big\}\bigg)\\
		&=\min_{\calN\subseteq \calV(\calT)}\varliminf_{n\rightarrow\infty}-\frac{1}{n}\log\Pr\bigg(\calJ(n)\cap \calK(n)\cap \calL\Big(\scC_{\mathrm{E}}(\halpha n)^{\complement},\big[\halpha+3(1-\halpha)\big]n\Big) \cap\nonumber\\*
		&\qquad\qquad\{ \halpha=0.8\} \cap\big\{\scC_{\mathrm{E}}(\halpha n)^{\complement}=\calN\big\}\bigg)\label{unconf:ieq1}\\
		&\geq \min_{\calN\subseteq \calV(\calT)}\varliminf_{n\rightarrow\infty}-\frac{1}{n}\log\Pr\bigg(\big\{\hatcalT_{\Active}(\calN)\neq\calT(\calN)\big\}\cap \calL\Big(\calN,\big[\halpha+3(1-\halpha)\big]n\Big) \cap \{ \halpha=0.8\}\bigg)\\
		&\geq \big[0.8+3\times(1-0.8)\big]K _{\Passive}(\rho)\\*
		&=1.4K _{\Passive}(\rho)= c_{\rho}K _{\Passive}(\rho), \label{unconf:ieq2}
	\end{align}
	where $c_{\rho}=1.4$ when $\rho\in[0.8,1)$, Eqn.~\eqref{unconf:eq1} follows as the event $\calK(n)$ implies the event $\calL\big(\scC_{\mathrm{E}}(\halpha n)^{\complement},\big[\halpha+3(1-\halpha)\big]n\big)$, Eqn.~\eqref{unconf:ieq1} follows  from the union bound and the fact that the smallest exponent dominates in the large deviations regime. 

	For $\rho\in [0.6,0.8)$, we compute the error exponent of the event that $\{\halpha >0.85\}$. It follows from Table~\ref{table:alpha} that $\halpha >0.85$ implies $\hrho<0.53$. Furthermore, for $\rho=0.6$, we have $\rho-\frac{1-\rho}{5.72}> 0.53$. Fig.~\ref{fig:hrho57} shows that the ratio between $K(\{\hrho(\halpha n)<0.53\}_{n\in\bbN})$ and $K_{\Passive}(\rho)$ is larger than $1.3$, and so we only need to consider the case where $\halpha \leq 0.85$. Following similar procedures as the steps leading to inequality \eqref{unconf:ieq2}, we have
	\begin{align}
		\varliminf_{n\rightarrow\infty}-\frac{1}{n}\log\Pr\big(\calJ(n)\cap \calK(n)\cap \{\halpha \leq 0.85\}\big)
		&\geq \big[0.85+3(1-0.85)\big]K _{\Passive}(\rho)\\*
		&=1.3K _{\Passive}(\rho)=c_{\rho}K _{\Passive}(\rho),
	\end{align}
	where $c_{\rho}=1.3$ when $\rho\in[0.6,0.8)$.

	For $\rho$ in other intervals,  similar analyses can be performed.
\end{proof}


\bibliographystyle{alpha}
\bibliography{IEEEabrv,StringDefinitions,TGroup}

\end{document}